\documentclass{article} % For LaTeX2e

\usepackage[preprint]{mingze}

\usepackage{amsmath,amsfonts,bm}

\def\1{\bm{1}}

\def\rd{{\textnormal{d}}}

\DeclareMathAlphabet{\mathsfit}{\encodingdefault}{\sfdefault}{m}{sl}
\SetMathAlphabet{\mathsfit}{bold}{\encodingdefault}{\sfdefault}{bx}{n}

\newcommand{\sm}{{\rm softmax}}

\newcommand{\SA}{{\textnormal{\textsf{SA}}}}
\newcommand{\TF}{{\textnormal{\textsf{TF}}}}
\newcommand{\FFN}{\textnormal{\textsf{FFN}}}
\newcommand{\FG}{\textnormal{\textsf{G}}_4}
\newcommand{\IH}{\textnormal{\textsf{IH}}_2}
\newcommand{\nIH}{\textnormal{\textsf{IH}}_n}
\newcommand{\GnIH}{\textnormal{\textsf{GIH}}_n}
\newcommand{\init}{{\rm init}}
\newcommand{\Lip}{{\rm Lip}}

\newcommand{\trinorm}[1]{{\left\vert\kern-0.25ex\left\vert\kern-0.25ex\left\vert #1 
   \right\vert\kern-0.25ex\right\vert\kern-0.25ex\right\vert}}

\newcommand{\bb}{b}

\newcommand{\be}{\boldsymbol{e}}

\newcommand{\bu}{u}
\newcommand{\bv}{v}
\newcommand{\bx}{x}
\newcommand{\by}{y}
\newcommand{\bz}{z}

\newcommand{\bI}{I}

\newcommand{\bS}{S}

\newcommand{\bW}{W}
\newcommand{\bX}{X}
\newcommand{\bZ}{Z}

\newcommand{\btheta}{\theta}

\newcommand{\cB}{\mathcal{B}}
\newcommand{\cC}{\mathcal{C}}

\newcommand{\cI}{\mathcal{I}}

\newcommand{\cL}{\mathcal{L}}

\newcommand{\cN}{\mathcal{N}}
\newcommand{\cO}{\mathcal{O}}

\newcommand{\bbE}{\mathbb{E}}

\newcommand{\bbI}{\mathbb{I}}

\newcommand{\bbN}{\mathbb{N}}

\newcommand{\bbR}{\mathbb{R}}

\newcommand{\bzero}{0}

\newcommand{\pll}{\kern 0.56em/\kern -0.8em /\kern 0.56em}

\newcommand{\norm}[1]{\ensuremath{\left\| #1 \right\|}}
\newcommand{\bracket}[1]{\ensuremath{\left( #1 \right)}}

\newcommand{\<}{\left\langle}
\renewcommand{\>}{\right\rangle}

\usepackage[utf8]{inputenc} % allow utf-8 input
\usepackage[T1]{fontenc}    % use 8-bit T1 fonts
\usepackage{titletoc}
\usepackage[pagebackref]{hyperref}   % hyperlinks
\usepackage{url}            % simple URL typesetting
\usepackage{booktabs}       % professional-quality tables
\usepackage{amsfonts}       % blackboard math symbols
\usepackage{nicefrac}       % compact symbols for 1/2, etc.
\usepackage{microtype}      % microtypography
\usepackage{xcolor}         % colors

\usepackage[ruled,vlined]{algorithm2e}

\newcommand{\triplebar}[1]{\left\vert\!\left\vert\!\left\vert #1 \right\vert\!\right\vert\!\right\vert}
\hypersetup{colorlinks=true,linkcolor=red!70!black,linktocpage=false,citebordercolor=blue!70!black,citecolor=blue!70!black,anchorcolor=blue!70!black}

\usepackage{graphicx,subfig}
\usepackage{etoc}
\usepackage{amsmath}
\usepackage{amssymb}
\usepackage{mathtools}
\usepackage{amsthm}
\theoremstyle{plain}
\newtheorem{theorem}{Theorem}[section]
\newtheorem{proposition}[theorem]{Proposition}
\newtheorem{lemma}[theorem]{Lemma}

\theoremstyle{definition}
\newtheorem{definition}[theorem]{Definition}

\newtheorem{remark}[theorem]{Remark}
\newtheorem*{main result}{Main Theorem}
\allowdisplaybreaks[4]

\usepackage{todonotes}
\usepackage{enumitem}

\usepackage{wrapfig}
\title{How Transformers Get Rich: Approximation and Dynamics Analysis}

\author{Mingze Wang \\
School of Mathematical Sciences\\
Peking University\\
\texttt{mingzewang@stu.pku.edu.cn} 
\And
Ruoxi Yu\\
Center for Data Science\\
Peking University\\
\texttt{yuruoxi@stu.pku.edu.cn} 
\\
\And
Weinan E \\
School of Mathematical Sciences\\ 
Center for Machine Learning Research\\
AI for Science Institute\\
Peking University\\
\texttt{weinan@math.pku.edu.cn} 
\And
Lei Wu \\
School of Mathematical Sciences\\ 
Center for Machine Learning Research\\
Peking University\\
\texttt{leiwu@math.pku.edu.cn} 
}

\begin{document}

\maketitle

\begin{abstract}
    Transformers have demonstrated exceptional in-context learning  capabilities, yet the theoretical understanding of the underlying mechanisms remains limited.
A recent work~\citep{elhage2021mathematical} identified a ``rich'' in-context mechanism known as induction head, contrasting with ``lazy'' $n$-gram models that overlook long-range dependencies.
In this work, we provide both approximation and dynamics analyses of how transformers implement induction heads.
In the {\em approximation} analysis, we formalize both standard and generalized induction head mechanisms, and examine how transformers can efficiently implement them, with an emphasis on the distinct role of each transformer submodule.
For the {\em dynamics} analysis, we study the training dynamics on a synthetic mixed target, composed of a 4-gram and an in-context 2-gram component. This controlled setting allows us to precisely characterize the entire training process and uncover an {\em abrupt transition} from lazy (4-gram) to rich (induction head) mechanisms as training progresses.

% In-context learning (ICL) has emerged as a promising capability of large language models, induction heads (~\cite{elhage2021mathematical},~\cite{olsson2022context}) contribute a lot to ICL ability of Transformers.
% In this paper, we explore the mechanism of induction heads in Transformer networks from both approximation and optimization perspectives. We first provide a theoretical understanding by proving that a 2-layer Transformer with absolute positional encoding (APE) can replicate the retrieval and pasting mechanism, and extend this to retrieve according to local semantic information. Additionally, we analyze the phase transition from n-gram mechanism to induction heads, showing that under gradient flow, Transformers initially learn n-gram patterns before transitioning to induction heads as input sequences are long, driven by small initialization.
\end{abstract}

\section{Introduction}

% \vspace{-.1cm}

Transformer, introduced by~\citet{vaswani2017attention}, have achieved remarkable success across various domains, including natural language processing, computer vision, and scientific computing. An emergent observation is that transformers, trained on trillions of tokens, can perform (few-shot) in-context learning (ICL), which makes prediction based on the contextual information without needing model retraining~\citep{brown2020language}.
This ICL ability is widely regarded as crucial for enabling large language models (LLMs) to solve reasoning tasks, representing a key step toward more advanced artificial intelligence. 
 % Despite these practical accomplishments, the theoretical understanding of transformer  remains largely unexplored. 

To understand how transformers implement ICL, \citet{elhage2021mathematical} and \citet{olsson2022context} identified a simple yet powerful mechanism known as  {\bf induction head}. Specifically, given an input sequence \texttt{[$\cdots$ab$\cdots$a]}, an induction head predicts \texttt{b} as the next token by leveraging the prior occurrence of the pattern \texttt{ab} in the context, effectively modeling an in-context bi-gram.
In contrast,  traditional $n$-gram model~\citep{shannon1948mathematical} (with a small $n$) utilizes only a limited number of recent tokens to predict the next token, which is context-independent and inevitably overlooks long-range dependence.
Based on the extent of context utilization, we categorize $n$-gram model as a {\em ``lazy'' mechanism}, whereas the induction head represents a more {\bf ``rich'' mechanism}.

Practically, induction heads have been demonstrated to play a  critical role in  enabling LLMs' ICL capabilities~\citep{song2024out,crosbie2024induction}, and even  used to test new LLM architectures~\citep{gu2023mamba}.
Theoretically, induction heads also serve as a controllable tool for understanding various aspects of LLMs, such as multi-step reasoning~\citep{sanford2024transformers} and inductive biases of different architectures~\citep{jelassi2024repeat}.

In this paper, we aim to provide a theoretical analysis  of how transformers can efficiently implement induction heads. 
The first key problem is to rigorously formalize induction heads and evaluate the efficiency of transformers in representing them. According to \citet{elhage2021mathematical}, the original induction head can be implemented using a two-layer, twelve-head transformer without feed-forward networks (FFNs). However, practical scenarios demand more powerful induction heads. Thus, it is crucial to generalize the mechanism behind and explore how different transformer submodules, such as varying the number of attention heads or incorporating FFNs, impact the transformer’s ability to implement them. This forms our first research objective:

% \vspace{-.1cm}

\begin{center}
    \em (Approximation). Investigate how two-layer transformers express the induction head mechanism and its potential variants.
\end{center}

% \vspace{-.1cm}

The next problem is to investigate the dynamics of  transformers in learning induction heads. The pioneering works by \citet{elhage2021mathematical} and \citet{olsson2022context} demonstrated that transformers 
% initially learn $n$-gram (with a small $n$) and then 
undergo an abrupt phase transition to learning induction heads. A recent empirical study on synthetic datasets replicate this behavior, 
further showing that 2-gram is always learned prior to induction heads~\citep{bietti2024birth}. However, a rigorous theoretical analysis of this learning progression is still lacking. Closing this gap forms our second research objective:
\begin{center}
    \em (Dynamics). Understand how transformers transition from relying on $n$-gram patterns to employing the induction head mechanism as training progresses.
\end{center}

Focusing on these two key problems, in this paper, we make the following contributions:

% \vspace{-.1cm}

\begin{itemize}[leftmargin=2em]
    \item
    {\bf Approximation analysis: how transformers express induction heads.} We consider three types of induction heads with varying complexities. 
    First, we show that two-layer, single-head transformers without FFNs can efficiently approximate the vanilla induction head~\citep{elhage2021mathematical}. We then introduce two generalized induction heads, which  leverage richer in-context $n$-gram information and incorporate a general similarity function. Our  analysis clarifies the distinct roles of multihead attention, positional encoding, dot-product structure, and FFNs in implementing these generalized induction heads.
 
    \item{\bf Dynamics analysis: how learning undergoes a sharp transition from $n$-gram to induction head.} We study the learning dynamics of a two-layer transformer without FFNs for a mixed target, composed of a 4-gram and an in-context 2-gram component. This toy setting allows us to capture the entire training process precisely. Specifically, we show that learning progresses through four phases: partial learning of the 4-gram, plateau of induction head learning, emergence of the induction head, and final convergence, showcasing a sharp transition from 4-gram to induction head. 
    Our analysis identifies two key drivers of the transition: 1) time-scale separation due to   low- and high-order parameter dependencies in self-attention, and 2) speed differences caused by the relative proportions of the two components in the mixed target. 
    % Additionally, in our analysis, we introduce a novel Lyapunov function that exploits the unique structure of self-attention, which may be of independent interest.
\end{itemize}

Finally, we conduct a series of experiments, ranging from simple toy models to real-world natural language training tasks, to validate our theoretical insights.

% Due to space limitations, t
The detailed discussion of related works is deferred to Appendix~\ref{section: related works}.

% \vspace{-.1cm}

\section{Preliminaries}

% \vspace{-.1cm}

{\bf Notations.} For $k\in\bbN^+$, let $[k]=\{1,2,\dots,k\}$. 
For a vector $v$ and $1\leq p\leq\infty$, we denote by $\norm{v}_p$ the $\ell_p$ norm of $v$.
For a matrix $A=(a_{i,j})$, we denote by $\|A\|$, $\|A\|_F$ the spectral and Frobenius norms, respectively; let $\|A\|_{1,1}=\sum_{i,j}|a_{i,j}|$. For an event $S$, we define $\bbI\{S\}=1$ if $S$ is true, and $0$ otherwise.
We use ${\rm sm}(\cdot)$ to denote the softmax function. We use standard 
% asymptotic notations $\lesssim,\gtrsim,\sim$ or 
big-O notations $\cO,\Omega,\Theta$ to hide absolute positive constants, and use  $\tilde{\cO},\tilde{\Omega},\tilde{\Theta}$ to further hide logarithmic constants. 

{\bf Sequence modeling.}
Given a sequence of tokens $(x_1,x_2,x_3, \dots)$ with each token lying in $\bbR^d$, let $X_L=(x_1,x_2,\dots,x_L)\in\bbR^{d\times L}$ and $X_{m:n}=(x_m^\top,x_{m+1}^\top,\dots,x_{n}^\top)^\top\in\bbR^{(n-m+1)d}$. Given $A=(a_1,\cdots a_n)\in\bbR^{m\times n}$, we denote $(a_s)_{s=i}^j=(a_i,\cdots,a_j)\in\bbR^{m\times(j-i+1)}$.
Then, we consider the next-token prediction task: predict $x_{L+1}$ using $X_L=(x_1,x_2,\dots,x_L)$. 
% \subsection{Transformer Architecture}

In a {\bf $n$-gram model}~\citep{shannon1948mathematical}, the conditional probability of predicting the next token is given by $p(x_{L+1}|X_L)=p(x_{L+1}|X_{L-n+2:L})$, meaning that the  prediction  depends only on the most recent $n-1$ tokens.
In practice, the value of $n$ is typically small (e.g., 2, 3, or 4), as the computational cost of $n$-gram models grows exponentially with $n$.
However, $n$-gram models with small $n$ cannot capture long-range interactions,  leading to inferior performance in  sequence modeling.

{\bf Transformer}  is designed to more efficiently  capture long-range dependencies in sequence modeling~\citep{vaswani2017attention}. Specifically, given an  $L$-token input sequence $\bX=(\bx_1,\cdots,\bx_L)\in\bbR^{d\times L}$, an $U$-layer transformer $\TF$ processes it as follows. First, each input token is embedded into a higher-dimensional space through an {\em embedding layer} with  $\bW_E\in\bbR^{D\times d},\bb_E\in\bbR^D$:
$$
\bx_s^{(0)}=\bW_E \bx_{s}+\bb_E,\ s\in[L].
$$ 
Next, the $U$-layer \SA-\FFN blocks process the embedded sequence $\bX^{(0)}=(\bx_1^{(0)},\cdots,\bx_L^{(0)})$ as follows, and the output of the final layer is taken as the output sequence $\TF(X)=X^{(U)}\in\bbR^{D\times L}$:
\begin{equation}\label{model: Transformer}
\begin{aligned}
    \bX^{(u-\frac{1}{2})}&=\bX^{(u-1)}+\SA^{(u)}(\bX^{(u-1)}),\quad u\in[U];
    \\
    \bX^{(u)}&=\bX^{(u-\frac{1}{2})}+\FFN^{(u)}(\bX^{(u-\frac{1}{2})}),\quad u\in[U].
\end{aligned}
\end{equation}
% \todo[inline]{Change the FFN width from $W$ to $M$ to auoid confusion.}
Here, $\FFN^{(u)}$ denotes a (token-wise) two-layer FFN of width $M$, and $\SA^{(u)}$ represents the $H$-head self-attention operation. Specifically,  when applied to a sequence $\bZ=(\bz_1,\cdots,\bz_L)\in\bbR^{D\times L}$, $\SA^{(u)}$ operates it as follows:
% \vspace{-.2cm}
\begin{equation}\label{model: multi-head self-attention}
    \begin{gathered}
        \!\!\SA^{(u)}(Z)=\bW_O^{(u)}\sum_{h=1}^{H}\SA^{(u,h)}(\bZ),\quad \SA^{(u,h)}(\bZ)=
        \\
        \!\left(\bW_V^{(u,h)}Z\right)\!\sm\!\left(\<\!\bW_Q^{(u,h)}\!\bZ,\bW_K^{(u,h)}\!\bZ\!\>\!+\! R^{(u,h)}\!\right)\!,
    \end{gathered}
\end{equation}
where $\bW_{Q}^{(u,h)},\bW_{K}^{(u,h)},\bW_{V}^{(u,h)},\bW_O^{(u)}\in\bbR^{D\times D}$ correspond to the query, key, value and output matrices of the $(u,h)$-th head, respectively. $\sm$ represents taking softmax normalization across columns. 
$\<\bW_Q^{(u,h)}\bX,\bW_K^{(u,h)}\bX\>$ is called the dot-product  (DP) structure. 
Furthermore, $R^{(u,h)}=(R_{i,j}^{(u,h)})\in\bbR^{L\times L}$ denotes the additive relative positional encoding matrix, which satisfies $R_{i,j}^{(u,h)}=-\infty$ if $i\leq j$ for the next-token prediction task.

% Here, $\bM^{(l,h)}\in\{-\infty,0\}^{n\times n}$ is a causal mask defined as $M_{ij}=-\infty$ iff $i\leq j$, ensuring that each position $i$ can only attend to preceding positions $j<i$, which is essential for the next token prediction paradigm.

% More concretely, the output of the $l$-th Attention block for the $t$-th input token can be expressed as follows 
% \begin{align*}
%     \SA^{(l)}(\bx_t)
%     =\bW_O^{(l)}\sum_{h=1}^{H_l} \sum_{s=0}^{t-1}\frac{\bW_V^{(l,h)}\bx_{t-s}\exp\left(\<\bW_Q^{(l,h)}\bx_{t},\bW_K^{(l,h)}\bx_{t-s-1}\>\right)}{\sum_{j=0}^{t-1}\exp\left(\<\bW_Q^{(l,h)}\bx_{t},\bW_K^{(l,h)}\bx_{t-j-1}\>\right)},
% \end{align*}

% {\bf Relative Positional Encoding} (RPE) is crucial in the Transformer architecture to identify the sequence order. In general, there are two types of PE, which we study separately according to our requirements.

{\bf Relative positional encoding} (RPE).
Throughout this paper, we focus on the Alibi RPE~\citep{press2021train}, where $R_{ij}^{(u,h)}$ exhibit a Toeplitz structure, i.e., $R_{ij}^{(u,h)}=\phi(i-j;p^{(u,h)})$ for $i,j\in [L]$. Here, $p^{(u,h)}$'s are learnable parameters and  $\phi(\cdot;p)$ has the following form:

% \vspace{-.5cm}

\begin{equation}\label{equ: RPE}
    \phi(z;p)=\begin{cases}
    -p\cdot (z-1) & \text{ if } z\geq1
    \\
    -\infty & \text{\ otherwise}
    \end{cases}.
\end{equation}

% \vspace{-.3cm}

Note that we adopt the Alibi RPE only for simplicity and our results can be  extended to other additive RPEs, such as T5~\citep{raffel2020exploring}, and KERPLE~\citep{chi2022kerple}. However,  extending our analysis to the popular rotary RPE~\citep{su2024roformer} may be nontrivial, and we leave this for future work.

\section{Formulation and Approximation of Induction Head}\label{section: approximation}

% \vspace{-.1cm}

In this section, we formalize three types of induction head mechanisms with varying levels of  complexity.
We then theoretically investigate how two-layer single- or multi-head transformers, with or without FFNs, can efficiently implement these mechanisms, highlighting the distinct roles of different transformer submodules

% In the subsections below, we will identify {\bf two types} of the induction head with varying complexity, and explore how two-layer Transformers can implement them.

% a formal formulation of the standard induction head and demonstrate how efficiently two-layer Transformers without FFNs can implement it. 
% Building on this understanding, along with insights into various Transformer components, we proposed a more generalized induction head and prove that two-layer multi-head Transformers can efficiently approximate it.

% \vspace{-.1cm}

\subsection{Vanilla Induction Heads}\label{subsection: standard}

% \vspace{-.1cm}

The original induction head, proposed in~\citet{elhage2021mathematical} and \citet{olsson2022context}, is regarded as one of the key mechanisms to implement ICL and reasoning. 
This induction head suggests that two-layer multi-head transformers without FFNs can execute a simple in-context algorithm to predict the next token \texttt{b} from a context \texttt{[$\cdots$ab$\cdots$a]} through retrieval, copying, and pasting, based on in-context bi-gram pairs, as illustrated in Figure~\ref{fig: example, induction head}.

\begin{figure}[!ht]
    \centering
    \includegraphics[width=12cm]{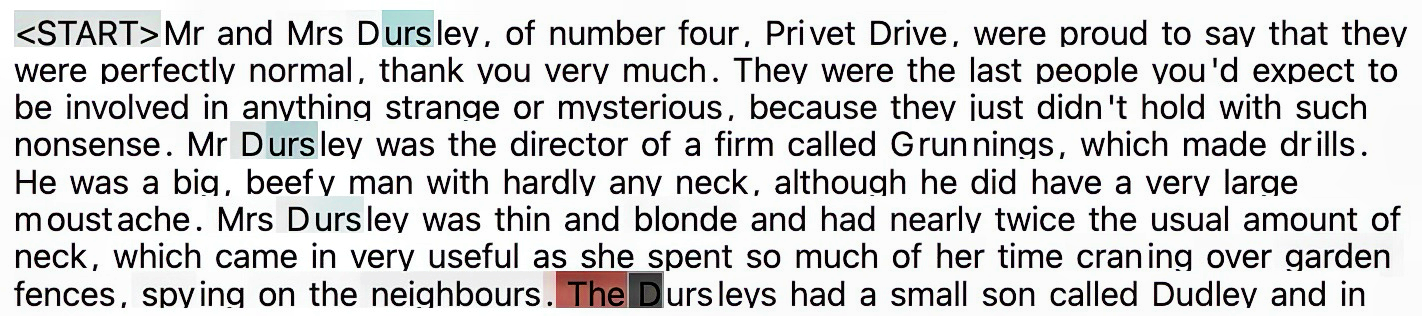}
    \caption{\small An illustration of the original induction head (taken from~\citet{elhage2021mathematical}). The induction head proceeds the context [$\cdots$The D] by retrieving the preceding information most relevant to the current token (\texttt{D}), then copying and pasting the subsequent token (the green \texttt{urs}) as the current prediction. 
    Notably, the first and second self-attention layers focus on the highlighted red and green tokens, respectively. For further details, refer to the description below Theorem~\ref{theorem: standard}. 
    % {\red Particularly, in this figure, the highlighted tokens are $z_{L}=(x_{L-1},x_L)$ and the ones that match $z_L$. We refer to the description below Theorem~\ref{theorem: standard} for details.}
    }
    \label{fig: example, induction head}
    % \vspace{-.2cm}
\end{figure}

{\bf Formulation of $\IH$.}
% Given an input sequence $\bX_t=(\bx_1,\cdots,\bx_t)\in\bbR^{d\times t}$, and let the model be tasked with predicting the next token $\bx_{t+1}$.
Based on the phenomenon illustrated in Figure~\ref{fig: example, induction head}, we define the vanilla induction head $\IH:\cup_{L\in\bbN^+}\bbR^{d\times L}\mapsto\bbR^d$ as follows:
\begin{equation}\label{equ: induction head, type I}
    \IH(\bX_L)= \sum_{s=2}^{L-1} x_s \ {\rm sm}\Big(\big(x_{L}^\top W^\star x_{\nu-1} \big)_{\nu=2}^{L-1}\Big)_{\nu=s}.
\end{equation}
Specifically, $\IH$ retrieves in-context information based on the similarities of in-context bi-gram pairs $\{(x_{s}, x_L)\}_{s=1}^{L-2}$. Note that the magnitude of  matrix $W^\star$ controls the sparsity of retrieval, since increasing $\|W^\star\|$ causes the softmax output to concentrate as a delta measure over the preceding tokens. 
% {\red Particularly, when $W^\star=\gamma I_{d\times d}$ with $\gamma\to+\infty$, it recovers the original induction head introduced in \citet{elhage2021mathematical}.}
Additionally, $\IH$ can handle input sequences of arbitrary length.

% \paragraph{Explanations of Equation~\eqref{equ: induction head, type I}.}
This model retrieves previous tokens $x_{s-1}$'s that are similar to the current token $\bx_L$ based on a dot-product similarity, and then copies and pastes $\bx_{s-1}$'s subsequent token $\bx_{s}$ as the current prediction $\bx_{L+1}$.
% The softmax normalization is applied to the similarity scores to ensure the sparsity of relevant information.
For example, in Figure~\ref{fig: example, induction head}, the current token $\bx_L$ is \texttt{D}, and the model retrieves previous tokens similar to \texttt{D}, copying and pasting its subsequent token \texttt{urs} as the prediction.

{\bf Comparison with previous formulations.}
As shown in Figure~\ref{fig: example, induction head}, the current token \texttt{D} appears multiple times in the preceding context, and the induction head detects all occurrences of \texttt{D}.
Our formulation~\eqref{equ: induction head, type I} captures this behavior, as the softmax scores for all preceding \texttt{D} are identical.
In contrast, previous formulations, such as \citet{sanford2024one} and \citet{sanford2024transformers}, focus solely on the most recent occurrence of \texttt{D}, neglecting this multi-occurrence aspect.

{\bf Measure of approximation.}  Consider a target function $\textsf{H}:\cup_{L\in\bbN^+}\bbR^{d\times L}\mapsto\bbR^d$, where $d$ is the token dimension and $L$ denotes the sequence length.
Given an input sequence $X\in\bbR^{d\times L}$, transformer $\TF$ approximates $\textsf{H}(X)$ using its last output token, i.e., $\TF_{-1}(X)\in\bbR^d$. To quantify the approximation error, we define the following metric: for $1\leq p\leq +\infty$,
\begin{equation}
\label{equ: approximation norm}
\!\!\!\triplebar{\textsf{H}-\TF}_{L,p}\!\!:= \big(\bbE_{X_L}[\|\textsf{H}(X_L)-\TF_{-1}(X_L)\|_{\infty}^p]\big)^{1/p}.
\end{equation}

The next theorem shows that a two-layer {\em  single-head} transformer {\em without FFNs}  suffices to implement vanilla induction heads.

\begin{theorem}[two-layer single-head $\TF$ w/o FFNs]\label{theorem: standard}
Let $\IH$ satisfy Eq.~\eqref{equ: induction head, type I}. 
Then exists an absolute constant $C>0$ and
% for any target error $\epsilon>0$,
a two-layer single-head transformer $\TF$ (without FFNs), with $D=2d$, $W_K^{(1,1)}=W_Q^{(1,1)}=0$,
% $|p^{(1,1)}|\leq\cO(\log(\norm{W^\star}_{1,1}/\epsilon))$, 
$p^{(2,1)}=0$, and $\|W_K^{(2,1)}\|,\|W_Q^{(2,1)}\|\leq\cO(1,\|W^\star\|_{F})$, such that
\begin{align*}
\sup_{L\in\bbN^{+}}\triplebar{\IH-\TF}_{L,\infty} \leq \frac{ C}{e^{p^{(1,1)}}}.
\end{align*}
\end{theorem}

% \vspace{-.2cm}

This theorem shows that single head suffices to approximate the vanilla induction head and moreover,  the approximation efficiency is independent of the sequence length.
The proof  is provided in Appendix \ref{subsecton: thm standard proof}, offering the following  insights into how two-layer single-head transformers without FFNs implement vanilla induction heads:

% \vspace{-.1cm}

\begin{itemize}[leftmargin=2em]
    \item {\bf The first layer} aggregates local tokens and  outputs $\bz_s=[x_{s-1},x_{s}]$ for the $s$-th token ($2\leq s\leq L$). This is achieved by using $\SA$ with only RPE (no DP). Specifically, RPE allows $\SA$ to  capture the {\em preceding token} via $x_{s-1}=\sum_{j\geq1} x_{s-j}\rho(j)$ for each token $x_s$, where  $\rho(\cdot)=\bbI\{\cdot=1\}$. Hence, DP in this layer is not essential and can be omitted.
    \item
    {\bf The second layer} extracts the relevant tokens using DP similarity. First, DP computes the similarity  $\<W_Q\bz_L,W_K \bz_{s}\>=\bx_L^\top W^\star \bx_{s-1}$, where $\bz_L=[\bx_{L-1},\bx_L]$ and $\bz_s=[\bx_{s-1},\bx_s]$ represent the hidden tokens output by the first layer.
    This similarity measure enables
    $\SA$ to identify tokens that match $\bx_L$. 
    Subsequently, the value $W_V \bz_s$ extracts $\bx_s$,  effectively copying the subsequent token of $\bx_{s-1}$ and using it as the current prediction. In this layer, RPE is not necessary and can be omitted.
\end{itemize}

\begin{remark}[Alignment with experimental findings]
Our theoretical analysis is consistent with the experimental observations reported in~\citet{elhage2021mathematical}. Specifically, the experiments there demonstrate that  $\SA$ in the first layer attends to adjacent tokens, while  $\SA$ in the second layer  retrieves information related to the current token. Our analysis identifies components responsible for these two operations, and reveals that {\em single-head} transformers suffice to perform them efficiently. Furthermore, we validate our theoretical construction through a {\em probing experiment}, presented in Figure~\ref{fig: approximation: probing} in Appendix~\ref{appendix: experiment: approximation}.
\end{remark}

% \begin{remark}[Other types of positional encodings]
% In Theorem~\ref{theorem: standard}, we employ RPE to identify the sequence order. 
% However, our result also holds for other absolute positional encoding (APE). 
% An analogous theorem, where transformers are equipped with APE instead of RPE, is provided in Appendix~\ref{section: proof of rp with rpe}.
% \end{remark}

% \vspace{-.1cm}

\subsection{Generalized Induction Heads: In-context \texorpdfstring{$n$}{}-gram and Generic Similarity}\label{subsection: general}

% \vspace{-.1cm}

Although the standard induction head defined in Eq.~\eqref{equ: induction head, type I} is intuitive, it exhibits  notable limitations: 
{\bf 1)} it retrieves only a {\em single token}, potentially missing {\em complete local information} and leading to false retrievals; 
{\bf 2)} it relies solely on the {\em dot-product} to measure the similarity between two tokens, which is not sufficiently general.
% 3) it does not consider {\em the positions} of relevant information, which neglects the fact that more recently appearing relevant information should receive more consideration. 

{\bf Formulation of $\nIH$.}
Motivated by the limitation {\bf 1)} above, we define a generalized induction head:
\begin{equation}\label{equ: induction head, type II}
\begin{gathered}
    \nIH(\bX_L)=\sum_{s=n}^{L-1}\bx_{s}\pi_s,
    \\
    \pi_s={\rm sm}\Big(\big(X_{L-n+2:L}^\top W^\star
    X_{\nu-n+1:\nu-1}\big)_{\nu=n}^{L-1}\Big)_{\nu=s},
\end{gathered}
\end{equation}
where the patch $X_{\nu-n+1:\nu-1}$ and $X_{L-n+2:L}$ incorporate {\em local information} near the previous token $\bx_{\nu-1}$ and the current $\bx_L$, respectively. 
This induction head operates based on the similarity between pairs: $(X_{s-n+1:s-1};X_{L-n+2:L})$ for $s=n,\dots,L-1$.
Notably, $n$ is typically small when extracting {\em local semantics}, so we assume $n\leq 100$ in Eq.~\eqref{equ: induction head, type II}.\footnote{For a given context $X_{1:L}$, Eq.~\eqref{equ: induction head, type II} predicts by considering all past $n$-gram pairs $x_{\nu}|X_{\nu-n+1:\nu-1}$ (for $\nu\leq L-1$), making it a context-dependent statistic, and is therefore called {\em in-context $n$-gram}. Additionally, note that the $n$ used in classic $n$-grams is typical small (e.g., $2,3,4$), so we make a reasonable assumption that $n\leq 100$.}

% This formulation is more general than Eq.~\eqref{equ: induction head, type I}, which only focuses on $x_{s-1}$.
% This induction head operates based on the similarity between the 
% pairs: $(X_{s-n+1:s-1};X_{L-n+2:L})$ for $s=n,\dots,L-1$.

% \paragraph{Explanations of Equation~\eqref{equ: induction head, type II}.}
Integrating richer local information facilitates more accurate information retrieval.
The model~\eqref{equ: induction head, type II} retrieves previous $(n-1)$-token patch that are similar to the current $(n-1)$-token patch, thereby generalizing the vanilla induction head~\eqref{equ: induction head, type I}, which considers only single-token retrieval ($n=2$ in Eq.~\eqref{equ: induction head, type II}).
For example, as depicted in Figure~\ref{fig: example, induction head}, if the current local information is \texttt{The D} (comprising two tokens), and prior local information such as \texttt{Mr D} and \texttt{Mrs D} is identified as similar to \texttt{The D}, transformer would copy and paste their subsequent token, \texttt{urs}, as the prediction.

% The following theorem demonstrates that two-layer multi-head Transformers without FFNs can efficiently exhibit this generalized mechanism~\eqref{equ: induction head, type II}, providing an explicit approximation rate:

\begin{theorem}[two-layer multi-head $\TF$ w/o FFNs]\label{theorem: type II}
Let $\nIH$ satisfy Eq.~\eqref{equ: induction head, type II}. 
Then, for any $H\in\bbN^+$ and rate $q\in\bbN^+$, there exists a constant $C_{n,q}>0$ and a two-layer $H$-head transformer $\TF(\cdot)$ (without FFNs), with $D=nd$, such that:
\begin{align*}
    \sup_{L\in\bbN^+}\triplebar{\nIH-\TF}_{L,\infty}\leq \left(\frac{C_{n,q}}{H}\right)^q,
\end{align*}
where $C_{n,q}=\cO(n q^2)$.
\end{theorem}

% \vspace{-.2cm}

This theorem demonstrates that two-layer multi-head transformers, even without FFNs,  can {\em efficiently} implement the generalized induction head~\eqref{equ: induction head, type II}. Notably, the approximation error scales as $\cO(H^{-q})$, where $q$ can be arbitrarily large.
Moreover, for a fixed $q$, $H\geq\Omega(n)$ is sufficient to ensure a good approximation. The proof of this theorem is provided in Appendix \ref{subsecton: thm type II proof}.
Additionally, the necessity of the required number of heads $H$ and embedding dimension $D$ is experimentally verified in Figure~\ref{fig: approximation: hardness} in Appendix~\ref{appendix: experiment}.
% Furthermore, $n$ is typically small when extracting local semantics. For example, in the vanilla induction head, $n=2$.

{\bf The role of multiple heads.}
In Theorem~\ref{theorem: type II}, multiple heads are employed in the first layer to approximate the $n$-gram interaction, represented by the $n-1$ memory kernels $\{\rho_j:=\bbI\{\cdot=j\}\}_{j=1}^{n-1}$.
Thus, $\TF$ can capture $n-1$ {\em preceding tokens} via $x_{s-j}=\sum_{k\geq1} x_{s-k}\rho_j(k)$ for $j\in [n-1]$. 
Intuitively, as $n$ increases, more memory kernels are required for accurate approximation, necessitating more attention heads.
In contrast, Theorem~\ref{theorem: standard} only requires approximating a single memory kernel $\bbI\{\cdot=1\}$, which can be efficiently achieved using a single attention head.

Recently,~\citet{rajaraman2024transformers} explored a generalized induction head similar to Eq.~\eqref{equ: induction head, type II} and showed that multi-layer single-head transformers can implement it. In contrast, our Theorem~\ref{theorem: type II} demonstrates that two layers suffice if multi-head self-attention is adopted.

% {\bf Comparison with previous work.}

{\bf Formulation of $\GnIH$.}
Building on the formulation~\eqref{equ: induction head, type II}, and motivated by the limitation {\bf 2)} above, we further consider the following generalized induction head:

% \vspace{-.6cm}

\begin{equation}\label{equ: induction head, type III}
\begin{gathered}
    \GnIH(\bX_L)=\sum_{s=n}^{L-1}\bx_s\pi_s,
    \\
    \ 
    \pi_s={\rm sm}\Big(\big(g\big(X_{L-n+2:L};X_{\nu-n+1:\nu-1}\big)\big)_{\nu=n}^{L-1}\Big)_{\nu=s},
\end{gathered}
\end{equation}

% \vspace{-.2cm}

where $g:\bbR^{D\times(n-1)}\times\bbR^{D\times(n-1)} \to\bbR$ denotes a generic function measuring the similarity between two $(n-1)$-length patches.

% This is more general than Equation~\eqref{equ: induction head, type II}, which only apply to the dot-product similarity measure.
% (not only the dot-product measure $\<\cdot,W^*\cdot\>$ in~\eqref{equ: induction head, type I}).
% \item $r(\cdot):\bbN\to\bbR$ is decreasing, which ensures that more recently appearing relevant information should receive more consideration.

% \paragraph{Explanations of Equation~\eqref{equ: induction head, type III}.}
This model retrieves previous relevant multi-token patch $X_{s-n+1:s-1}$ that is similar to the current multi-token patch $X_{L-n+2:L}$ , utilizing the generalized similarity function $g(\cdot,\cdot)$.
This mechanism is more general than Eq.~\eqref{equ: induction head, type II}, which is limited to  dot-product similarities.
For instance, the use of general similarity $g$ enables the model to recognize not only synonymous but also antonymic semantics, thereby improving both the accuracy and diversity of in-context retrievals.

% Additionally, this formulation include the position weighting $r$ to give precedence to tokens that are closer to current token.
% The subsequent token $\bx_{s}$, following the retrieved semantic $\tilde{\bx}_{s-1}$, is then copied and pasted as the prediction for the next token.
% Softmax normalization is still applied to the similarity scores to ensure the sparsity of relevant information.

% The integration of richer local information and general similarity measure in the general induction head mechanism facilitates more precise information retrieval.
% The introduction of richer local information enable it exhibit a more complete in-context $n$-gram
% \texttt{BC D}
% The generalized induction head can significantly reduce incorrect retrieval caused by the single-token retrieval, such as erroneously copying the token following ``\texttt{BC D}'' as the current prediction.

% To describe the rate of universal approximation for the induction head, we introduce a norm based on a double integral

% For functional $f:[0,1]^{d\times t}\mapsto\mathbb{R}^d$, we define integral norm as 
% $$\triplebar{f(\cdot)}:=\left\|\Vert f(\cdot) \Vert_{2}\right\|_{(L^2,[0,1]^{d\times t})}.$$

\begin{theorem}[two-layer multi-head $\TF$ with FFNs]\label{theorem: type III}
Let $\GnIH$ satisfy Eq.~\eqref{equ: induction head, type III}. 
Suppose the similarity function $g$ is $\alpha$-well-behaved (see Definition~\ref{definition: POD}).
Then there exist two absolute constants $A,B>0$ (only depending on the properties of $g$) such that: for any $H,M\in\bbN^{+}$ and rate $q\in\bbN^+$, there exists a constant $C_{n,q}>0$ and a two-layer $H$-head transformer $\TF(\cdot)$ with FFNs of width $M$, such that 
% \begin{align*}
%     \sup_{t}\norm{\GnIH-\TF}_{t,2}\leq\cO\left(\frac{C_1(g;q)n^2}{H^q}\right)+\tilde{\mathcal{O}}\left(\frac{C_2(g)K}{\sqrt{M}}\right)+\cO\left(\frac{1}{K^\alpha}\right).
% \end{align*}
\begin{align*}
    \triplebar{\GnIH-\TF}_{L,2}\leq A\left(\frac{C_{n,q}}{H}\right)^q+ B\frac{L^{1/(1+2\alpha)}}{M^{\alpha/(1+3\alpha)}},
\end{align*}
where $C_{n,q}=\cO(n q^2)$.
\end{theorem}

% The {\bf proof sketch} and {\bf theoretical insights} are summarized in~\eqref{equ: proof sketch, type II}. 
This theorem establishes that if the similarity function $g$ is well-behaved, two-layer multi-head transformers with FFNs can efficiently implement the generalized induction head~\eqref{equ: induction head, type III}.

{\bf The role of FFNs.} 
In contrast to Theorem~\ref{theorem: type II}, transformer models in Theorem~\ref{theorem: type III} include FFNs. 
These FFN layers are used to approximate the similarity function $g$.
Specifically, we consider the proper orthogonal decomposition (POD) of $g$, which can be viewed as an extension of the matrix singular value decomposition (SVD) applied to functions of two variables. 
For $g:\cI\times\cI\to\bbR$, its POD is $g(\bu,\bv)=\sum_{k=1}^{\infty}\sigma_{k}\phi_{k}(\bu)\psi_{k}(\bv)$, where $\phi_{k},\psi_k$ are orthonormal bases for $L^2(\cI)$ (see Appendix \ref{appendix: lemma} for details).
Intuitively, the FFN in the first layer is used to efficiently approximate $K$ bases ($\phi_i$'s and $\psi_i$'s). 
Then, in the second layer, DP in $\SA$ can approximately reconstruct $g$ by using the truncated sum $g(\bu,\bv)\approx\sum_{k=1}^{K}\sigma_{k}\phi_{k}(\bu)\psi_{k}(\bv)$.
The complete proof is deferred to Appendix \ref{subsection: thm III proof}.

% \begin{itemize}[leftmargin=2em]
%     \item {\bf 1-st Attn Layer:} The first attention layer captures local semantic information $\tilde{\bx}_s$ and $\tilde{\bx}_{s-1}$ of the tokens $\bx_s$ by approximating the extraction kernel $\rho$, which is critical for subsequent comparison. 
%     \item {\bf FFN layer:} Subsequently, the FFN layer approximates the first $n$ pair bases ($\phi_1,\cdots,\phi_n,\psi_1,\cdots\psi_n$) of the function $g$'s POD. It then applies these functions to each token.
%     \item {\bf 2-nd Attn Layer:} The second attention layer utilizes the dot-product structure to reconstruct the $n$-rank truncation of $g$, coupled with relative positional weights $r$. 
%     This process enables the model to compare the current semantic context $\tilde{\bx}_t$ with all previous semantics, identifying similarities under the generalized similarity measure $g$ and postion weighting $r$.
%     Finally, if some previous semantic $\tilde{\bx}_{s-1}$ closely aligns with $\tilde{\bx}_t$, the model will copy its subsequent token $\bx_s$ and pastes it as the current prediction. 
% \end{itemize}

% \vspace{-.1cm}

\section{The Transition from Lazy to Rich Mechanisms in Learning Induction Heads}\label{section: optimization}

% \vspace{-.1cm}

% In this section, we examine the phase transition initially discussed in ~\citet{olsson2022context}, where the model transitions from relying on the $n$-gram mechanism to employing the induction head mechanism. This transition marks the emergence of most in-context learning capabilities.

In this section, we investigate the dynamics of  learning  induction heads using a  transformer, particularly focusing on how this differs from  $n$-gram learning. To facilitate the analysis, we consider  a mixed target function that comprises a $4$-gram component and a vanilla induction head component as defined in Eq.~\eqref{equ: induction head, type I}. Specifically, we study  the gradient flow dynamics of a two-layer multi-head transformer without FFNs on this task.

% Our meticulous theoretical analysis uncovers {\bf a distinct learning transition}: the 4-gram mechanism, which fits 4-gram information, is learned rapidly at first, followed by a delayed, sudden learning of the induction head mechanism, which fits in-context 2-gram information. 
% Specifically, the dynamics exhibit {\bf four phases}, comprising partial learning, plateau, emergence, and convergence.
% Furthermore, our analysis identifies two primary factors driving this transition.

% We construct a target function that incorporates both a bi-gram mechanism and an induction head mechanism, and train a two-layer attention-only Transformer model to approximate the target. 
% During training, we first increase the learning rate for the position parameter in the first layer, enabling the output of this layer to capture the preceding token. Then, we increase the learning rate for the second layer, which includes both the bi-gram head and induction head, to demonstrate the phase transition phenomenon.

% Through gradient flow analysis, we show that the bi-gram head is learned first. Our analysis identifies two key reasons: first, the induction head operates at a higher frequency compared to the bi-gram head; second, in practice, most data is generated by the $n$-gram mechanism, which we represent by the fraction $\alpha$ in our target function. 

% \vspace{-.2cm}

\subsection{Setups}

% \vspace{-.1cm}

\subsubsection{Mixed Target Function}

% \vspace{-.1cm}

% Since our data is sampled from a Gaussian distribution, we consider one-dimensional data without loss of generality. The notations in this section differ slightly from those used in the preliminaries and the previous section, as we focus exclusively on input sequences that are sufficiently long.

{\bf Mixed target function.} 
Let the input sequence be $X=(x_1,\cdots,x_L)\in\bbR^{1\times L}$. 
Our mixed target function $f^\star$ contains both a $4$-gram component $f_{\FG}^\star$ and an in-context $2$-gram component $f_{\IH}^\star$:

% \vspace{-.6cm}

\begin{equation}\label{target: optimization}
    \!\!\! f^\star(\bX):=\left(\frac{\alpha^\star}{1+\alpha^\star}f_{\FG}^\star(\bX),\frac{1}{1+\alpha^\star}f_{\IH}^\star(\bX)\right)^\top\!\!\!\in\bbR^2,
\end{equation}

% \vspace{-.2cm}

where $\alpha^\star>0$ represents the relative weight between the two components: $f_{\FG}^\star(\bX)$ and $f_{\IH}^\star(\bX)$.
Here, $f_{\FG}^\star$ represents a $4$-gram component and $f_{\IH}^\star$ is given by the vanilla induction head~\eqref{equ: induction head, type I} to represent a type of in-context 2-gram information:

% \vspace{-.6cm}

\begin{align*}
f_{\FG}^\star(\bX)&:=\bx_{L-2},\\
f_{\IH}^\star(\bX)&:=\sum_{s=2}^{L-1}x_s\ {\rm sm}\Big(\big(x_L {w^\star}^2 x_{\nu-1}\big)_{\nu=2}^{L-1}\Big)_{\nu=s}.    
\end{align*}

% \vspace{-.3cm}

Note that $f_{\FG}^\star$ denotes a ``simplest'' 4-gram target, where the next token is predicted according to the conditional probability $p(z|X)=p(z|x_{L},x_{L-1},x_{L-2})=\bbI\{z=x_{L-2}\}$.

\begin{remark}[The reason for considering $4$-gram]
Note that our target includes a  4-gram component rather than simpler 2- or 3-gram components. 
As suggested by the experimental results in~\citet{elhage2021mathematical}, for a learned two-layer transformer that implements vanilla induction head $\IH$, the first layer has extracted both $x_{L}$ and $x_{L-1}$, which can be outputted using the residual block. Thus, the 2- and 3-gram targets: $p(z|X)=\bbI\{z=x_{L}\}$ and $p(z|X)=\bbI\{z=x_{L-1}\}$ must be learned prior to the induction head. Hence we focus on the more challenging 4-gram target to avoid trivializing the learning process, though our analysis extends straightforwardly to the 2- or 3-gram scenarios.
\end{remark}

\begin{remark}[Extension]
    Since the transformer studied in this section does not have FFNs, its expressive power is limited. Consequently, we only consider the  simple but representative mixed target~\eqref{target: optimization}. 
    However, \eqref{target: optimization} can be generalized to $f^\star(X)=F(f_{\FG}^\star(\bX);f_{\IH}^\star(X))$, where $F$ is general nonlinear function.
    Such a form can be efficiently approximated by transformers with FFNs. We leave the dynamics analysis under this general setting for future work.
\end{remark}

% both a $4$-gram part $f_{\FG}^\star(\bX)$  and in-context bi-gram information $f_{\IH}^\star(\bX)$, with the proportion $\alpha^\star$.

% Given $\bX=(x_1,\cdots,x_{L})\in\bbR^{1\times L}$, the target is defined as
% \begin{align*}
%     f^*(\bX) = \begin{pmatrix}
%         \frac{\alpha}{1+\alpha}x_{L-2}
%         \\\frac{1}{1+\alpha}\frac{1}{L-2}\sum_{s=1}^{L-1}\exp({w^*}^2x_Lx_{s-1})x_s
%     \end{pmatrix}:=\begin{pmatrix}
%     \frac{\alpha}{1+\alpha}x_{BG}
%     \\\frac{1}{1+\alpha}x_{IH}
%     \end{pmatrix},
% \end{align*}
% where $\alpha$ represents the fraction of data generated by the bi-gram mechanism, $x_{BG}$ represents the output generated by bi-gram mechanism, $x_{IH}$ represents the output generated by induction head mechanism.

% This target function is a simplification of a compositional function
% $$f\left(g(x_{BG}),h(x_{IH})\right)$$
% where $f,g,h$ belong to Barron space, since we don't take FFN layer into consideration here.

% \vspace{-.2cm}

\subsubsection{Two-layer Multi-head Transformer with Reparameterization}

% \vspace{-.1cm}

{\bf Two-layer multi-head transformer w/o FFNs.}
We consider a simple two-layer multi-head transformer $\TF$, where the first layer contains a single head $\SA^{(1,1)}$, and the second layer contain two heads $\SA^{(2,1)},\SA^{(2,2)}$.
Given an input sequence $X=(x_1,\cdots,x_L)\in\bbR^{1\times L}$, it is first embedded as $X^{(0)}:=(X^\top,0^\top)\in\bbR^{2\times L}$. The model then processes the sequence as follows:
\begin{align*}
    X^{(1)}=&
    \ X^{(0)}+\SA^{(1,1)}(X^{(0)}),
    \\\TF(X)=&\ \SA^{(2,1)}(X^{(1)})+\SA^{(2,2)}(X^{(1)}).
\end{align*}

% \vspace{-.2cm}

% We consider two-layer attention-only Transformers, which are sufficient to implement the induction head mechanism as demonstrated in ~\citet{olsson2022context} and supported by our approximation results.
% Desipite 
% To make the optimization analysis amanable, we follow previous works~\citep{tian2023scan,huang2023context,chen2024unveiling} to consider a {\bf reparameterization} trick.

{\bf Reparameterization.} Despite the simplification, the transformer above is still too complicated for dynamics analysis. To overcome this challenge, we adopt the reparametrization trick used in previous works~\citep{tian2023scan,huang2023context,chen2024unveiling}. Specifically, by Theorem~\ref{theorem: standard} and its proof, {\em the first layer does not require DP, and the second layer does not require RPE}.
Moreover, to express the $4$-gram component $f_{\FG}^\star$, we only need an additional head without DP in the second layer. Therefore, we can reparameterize the model as follows:

% \vspace{-.1cm}

\begin{itemize}[leftmargin=2em]
    \item {\bf The first layer.} 
    This layer has only one trainable parameter $p^{(1,1)}$. 
    In the unique head $\SA^{(1,1)}$, DP is removed by setting $W_{Q}^{(1,1)}=W_{K}^{(1,1)}=0$, and we let $W_V^{(1,1)}=\begin{pmatrix}
        0 & 0 \\ 1 & 0
    \end{pmatrix}$. 
    The output sequence of this layer given by $
        X^{(1)}=X^{(0)}+\SA^{(1,1)}(X^{(0)})=\begin{pmatrix}
        x_1,\cdots,x_L
        \\
        y_1,\cdots,y_L
    \end{pmatrix}$, where

    % \vspace{-.4cm}
    
    \begin{equation}\label{equ: output of 1st layer}
    y_s\!=\!\sum_{\tau=1}^{s-1} x_\tau\ {\rm sm}\left(\!\big(-p^{(1,1)}(s-1-\nu)\big)_{\nu=1}^{s-1}\right)_{\nu=\tau}
    \end{equation}

    % \vspace{-.2cm}
    
    for $s \in [L]$, where $p^{(1,1)}$, used in RPE~\eqref{equ: RPE}, is the unique trainable parameter in this layer.

    %  A single head containing position encoding without dot-product, which also used in {\red ref [sss]}. For input sequence $\bX=(x_1,\cdots,x_L)$, the first layer (with residual block) output an sequence $X^{(1)}=\begin{pmatrix}
    %     x_1,\cdots,x_L
    %     \\
    %     y_1,\cdots,y_L
    % \end{pmatrix}$, where

    \item {\bf The second layer.} This layer has 5 trainable parameters: $w_V^{(2,1)},w_V^{(2,2)},p^{(2,1)},w_{K}^{(2,2)},w_{Q}^{(2,2)}$ for parametrizing the 
    two heads. 
    The first head $\SA^{(2,1)}$ without DP is responsible to fit $f_{\FG}^\star$, while the second head $\SA^{(2,2)}$ without RPE is responsible to fit $f_{\IH}^\star$. Specifically,

    % \vspace{-.6cm}
    
    \begin{gather*}
        W_{Q}^{(2,1)}=W_{K}^{(2,1)}=0,
        W_{V}^{(2,1)}=\begin{pmatrix} 0 & w_{V}^{(2,1)}\\ 0 & 0 \end{pmatrix},
        \\
        p^{(2,2)}=0,
        W_{V}^{(2,2)}=\begin{pmatrix} w_{V}^{(2,2)} & 0 \\ 0 & 0 \end{pmatrix}.
    \end{gather*}

    % \vspace{-.2cm}

    Then the second layer processes $X^{(1)}$ and outputs the last token:
    
    % {\red 
    % \begin{equation}\label{equ: output}
    % \begin{aligned}
    % \TF_{-1}(\bX;\theta)=\begin{pmatrix}
    %     (w_V^{(2,1)}y_{s})_{s=2}^{L-2}\ \sm\big(\big(-p^{(2,1)}(L-1-s)\big)_{s=2}^{L-2}\big)^\top
    %     \\ 
    %     (w_V^{(2,2)}x_s)_{s=2}^{L-2}\sm\left(\big(x_L w_Q^{(2,2)} w_K^{(2,2)} x_{s-1}\big)_{s=2}^{L-2}\right)^\top
    % \end{pmatrix},
    % \end{aligned}
    %  \end{equation}
    %  }

    % \vspace{-.6cm}

    \begin{gather*}
    \!\!\TF_{-1}(\bX;\theta)\!=\!\Big(
        \sum_{s=2}^{L-2} w_V^{(2,1)}y_{s}\pi_s,
        \sum_{s=2}^{L-2} w_V^{(2,2)}x_s\rho_s\!\Big)^\top\!\!,
    \end{gather*}

    % \vspace{-.6cm}
    
    \begin{equation}\label{equ: output}
    \begin{gathered}
        \pi_s={\rm sm}\left(\big(-p^{(2,1)}(L-1-\nu)\big)_{\nu=2}^{L-2}\right)_{\nu=s},
        \\
        \rho_s={\rm sm}\left(\big(x_L w_Q^{(2,2)} w_K^{(2,2)} x_{\nu-1}\big)_{\nu=2}^{L-2}\right)_{\nu=s},
    \end{gathered}
    \end{equation}

    where $y_s$ is given by~\eqref{equ: output of 1st layer}. $p^{(2,1)},w_V^{(2,1)}$ are trainable parameters in $\SA^{(2,1)}$, while $w_Q^{(2,2)},w_K^{(2,2)},w_V^{(2,2)}$ are trainable parameters in $\SA^{(2,2)}$.
    \end{itemize}

    The set of all six trainable parameters across both layers is denoted by $\theta$.
    
    % The second layer operate on it and the last output is: $\TF_L({\bX};\btheta)=\SA_L^{(2,1)}({\bX};\btheta)+\SA_L^{(2,2)}({\bX};\btheta)$, where $\SA_L^{(2,1)}({\bX};\btheta)=\sum_{s=2}^{L-2}\sm\left(-p(L-1-s)\right)\bW_{V}^{(2,1)}\tilde{x}_{s}$, 
    % $\SA_L^{(2,2)}({\bX};\btheta)=\sum_{s=2}^{L-2}\sm\left(\tilde{x}_{L}^\top{\bW_{K}^{(2)}}^\top\bW_{Q}^{(2)}\tilde{x}_{s}\right)\bW_{V}^{(2,2)}\tilde{x}_{s}$

    % we consider the following specific forms, which can achieve the global optimum of this task
    % \begin{align*}
    % \bW_{K}^{(2)}=\begin{pmatrix}
    %     w & 0
    %     \\
    %     0 & 0
    % \end{pmatrix},
    % \quad
    % \bW_{Q}^{(2)}=\begin{pmatrix}
    %     0 & w
    %     \\
    %     0 & 0
    % \end{pmatrix},
    % \end{align*}
    % \begin{align*}
    %     \bW_{V}^{(2,1)}=\begin{pmatrix}
    %     0 & g
    %     \\
    %     0 & 0
    % \end{pmatrix},
    % \quad
    % \bW_{V}^{(2,2)}=\begin{pmatrix}
    %     h & 0
    %     \\
    %     0 & 0
    % \end{pmatrix}.
    % \end{align*}

% \vspace{-.1cm}

\subsubsection{Gradient Flow on Square Loss}

% \vspace{-.1cm}
    
% {\bf Loss Function.}
We consider the Gaussian input and square loss, both of which are commonly used in analyzing transformer dynamics and ICL~\citep{akyurek2022learning,huang2023context,wang2024transformers}. 
The loss is defined as:
\begin{equation}\label{equation: loss}
    \!\!\!\!\cL(\btheta)\!=\!\frac{1}{2}\bbE_{\bX\sim\cN(\bzero,\bI_{L\times L})}\!\left[\norm{\TF_{-1}({\bX};\btheta)-f^\star(\bX)}_2^2\right],
\end{equation}
% where $\theta=(p^{(1,1)},w_V^{(2,1)},w_V^{(2,2)},p^{(2,1)},w_{K}^{(2,2)},w_{Q}^{(2,2)})^\top$ represents all trainable six parameters in the two-layer Transformer $\TF_L(\cdot;\theta)$~\eqref{equ: output}. 
To characterize the learning of $\FG$ and $\IH$, we introduce the following two partial losses:
% \begin{align*}
%     \cL_{\FG}(\btheta)&=\frac{1}{2}{\bbE}_X\left(\SA_{-1}^{(2,1)}(\bX;\btheta)-f_{\FG}^\star(X)\right)^2,\ 
%     \cL_{\IH}(\btheta)&=\frac{1}{2}{\bbE}_X\left(\SA_{-1}^{(2,2)}(\bX;\btheta)-f_{\IH}^\star(X)\right)^2,
% \end{align*}
\begin{align*}
    \cL_{\FG}(\btheta)&=\frac{1}{2}{\bbE}_X\left(\TF_{-1,1}({\bX};\btheta)-f_1^\star(\bX)\right)^2,\\ 
    \cL_{\IH}(\btheta)&=\frac{1}{2}{\bbE}_X\left(\TF_{-1,2}({\bX};\btheta)-f_2^\star(\bX)\right)^2,
\end{align*}
which correspond to the two dimensions in $\TF_{-1}({\bX};\btheta)-f^\star(\bX)\in\bbR^2$, respectively.
It follows that $\cL(\theta)=\cL_{\FG}(\theta)+\cL_{\IH}(\theta)$.
% where $\SA_{-1}^{(2,1)}(X;\theta)$ and $\SA_{-1}^{(2,2)}(X;\theta)$ denote the last token of the output sequence $\SA^{(2,1)}(X;\theta)$ and $\SA^{(2,2)}(X;\theta)$, respectively.

{\bf Gradient flow (GF).} 
We analyze the GF for minimizing the objective~\eqref{equation: loss}:
% We minimize $\cL$ by Gradient Flow:
\begin{equation}\label{equ: GF}
\begin{gathered}
    \frac{\rd \btheta(t)}{\rd t}=-\nabla\cL(\btheta(t)), 
    \\\text{\ starting with\ } \theta(0)=(\sigma_{\rm init},\cdots,\sigma_{\rm init})^\top,
\end{gathered}
\end{equation}
where $0<\sigma_{\rm init}\ll1$ is sufficiently small.
Note that $\sigma_{\rm init}\ne0$ prevents $\nabla\cL(\theta(0))=0$.

% where all parameters initialized with small initialization denoted as $p_0,g_0,w_0,h_0$ ($p_0,g_0,w_0,h_0>0$), 
% which is necessary to avoid $\nabla\cL\equiv 0$. 
% To make calculating the expectation in Equation (\ref{equation: loss}) meaningful, $w^*$ can not excceed $\frac{1}{\sqrt{2}}$, and we assume $w^*\in (0,0.7]$.

{\bf Layerwise training paradigm}.
We consider a layerwise training paradigm in which, during each stage, only one layer is trained by GF. Specifically,

% \vspace{-.1cm}

\begin{itemize}[leftmargin=2em]
    \item {\bf Training Stage I:} In this phase, only the parameter in the first layer, i.e., $p^{(1,1)}$,  is trained.
    \item {\bf Training Stage II:} In this phase,the first layer parameter $p^{(1,1)}$ keeps fixed and only parameters in the second layer are trained: $w_V^{(2,1)}, w_V^{(2,2)}, p^{(2,1)}, w_Q^{(2,2)}, w_K^{(2,2)}$. 
\end{itemize}

% \vspace{-.1cm}

This type of layerwise training has been widely used to study the training dynamics of neural networks, including FFN networks~\citep{safran2022optimization,bietti2023learning,wang2023learning} and transformers~\citep{tian2023scan,nichani2024transformers,chen2024unveiling}.

% Before stating the main results, we list a few lemma

\begin{lemma}[Training Stage I]\label{lemma: layer I} For the Training Stage I, $\lim\limits_{t\to+\infty}p^{(1,1)}(t)=+\infty$.
\end{lemma}

% \vspace{-.1cm}

According to \eqref{equ: output of 1st layer}, this lemma implies that, at the end of Training Stage I, the first layer captures the preceding token $x_{s-1}$ for each token $x_s$, i.e., $y_s=x_{s-1}$. This property is crucial for transformers to implement  induction heads and  aligns with our approximation result in Theorem~\ref{theorem: standard}. The proof of Lemma~\ref{lemma: layer I} is deferred to Appendix \ref{appendix: dynamics: stage I}.

% We apply a technical simplification by normalizing the induction head with $\frac{1}{L-2}$, while such simplification is not necessary for the bi-gram component.

% \paragraph{Training process. }
% We first increase the learning rate for $\tilde{p}$, while keeping the other parameters in the second layer barely changed. This allows $\tilde{p}$ to be effectively learned, enabling the first layer to capture the token ahead, i.e., $y_s=x_{s-1}$. Once this is established, we treat $y_s$ as $x_{s-1}$ and begin to learn the remaining four parameters in the second layer.

\subsection{Training Stage II: Transition from \texorpdfstring{$4$}{}-gram to Induction Head}\label{subsection: optimization}

% \vspace{-.1cm}

In this section, we analyze the dynamics in Training Stage II. We start from the following lemma:

\begin{lemma}[Parameter balance]\label{lemma: parameter balance}
In Training Stage II, it holds that $|{w_{Q}^{(2,2)}}(t)|^2\equiv |{w_{K}^{(2,2)}}(t)|^2$.
\end{lemma}

Lemma~\ref{lemma: parameter balance} is similar to the balance result for homogeneous networks~\citep{du2018algorithmic}, and its proof can be found at the start of Appendix \ref{appendix: dynamics: stage II}.
By this lemma, we can define $w_{KQ}^{(2,2)}:=w_Q\equiv w_K$.
Additionally, Lemma~\ref{lemma: layer I} ensures that $p^{(1,1)}=+\infty$ holds during Stage II.
% By Lemma~\ref{lemma: parameter balance}, we can denote $w_{KQ}^{(2,2)}:=w_Q=w_K$.
For simplicity, we denote $w_{V_1}:=w_V^{(2,1)},w_{V_2}:=w_V^{(2,2)}, p:=p^{(2,1)},w_{KQ}:=w_{KQ}^{(2,2)}$. Consequently, the training dynamics are reduced to four parameters

% \vspace{-.4cm}

$$
{\theta}=\left(w_{V_1},w_{V_2},p,w_{KQ}\right),
$$

% \vspace{-.2cm}

where we still denote the set of parameters as $\theta$ without introducing ambiguity.
It is important to note that the problem remains {\bf highly non-convex} due to the joint optimization of both inner parameters ($p,w_{KQ}$) and outer parameters ($w_{V_1},w_{V_2}$) in the two heads. At this training stage, GF has a {\bf unique fixed point}: 
$$w_{V_1}=\frac{\alpha^\star}{1+\alpha^\star},\, w_{V_2}=\frac{1}{1+\alpha^\star},\, {p}=+\infty,\, w_{KQ}=w^\star,$$ 
which corresponds to a global minimizer of the objective~\eqref{equation: loss}.

% This property helps simplify the dynamics of $w_{Q}^{(2,2)}$ and $w_{K}^{(2,2)}$.

% \begin{remark}[Training Stage II]\label{remark: training stage II}
% \end{remark}

\begin{figure}[!ht]
    \centering
    \includegraphics[width=0.28\linewidth]{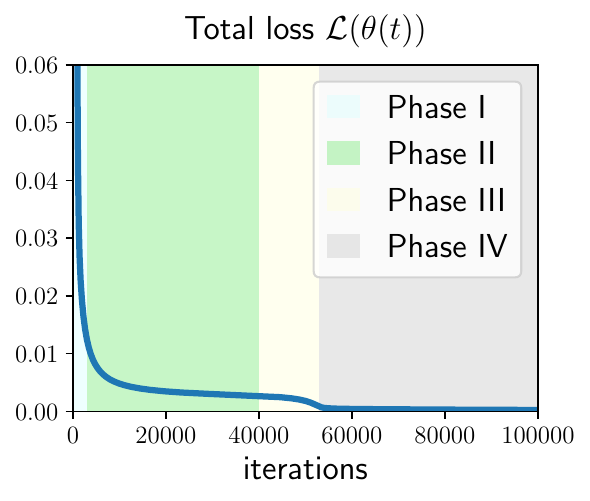}
    \includegraphics[width=0.32\linewidth]{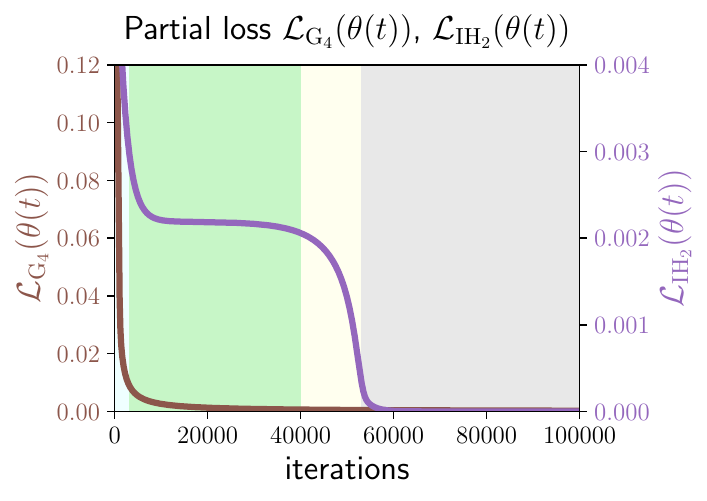}
    \includegraphics[width=0.27\linewidth]{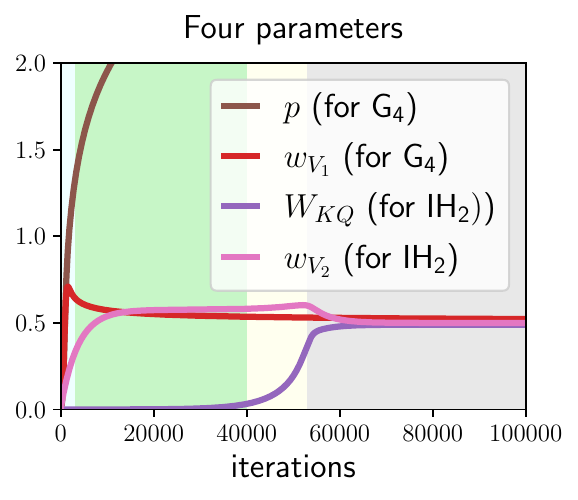}

    % \vspace{-.2cm}
    
    \caption{\small 
    Visualization of the dynamical behavior of Training Stage II with total loss, partial loss, and the parameter evolution. Here,  $\alpha^\star=1,w^\star=0.49,\sigma_{\rm init}=0.01,L=40$. 
    The is clearly shown  that transformer learns the $4$-gram component first and then, starts to learn the induction head mechanism. Notably, the entire dynamics unfold in four distinct phases, consistent with our theoretical results (Theorem~\ref{thm: optimization}).
    For more experimental details, we refer to Appendix~\ref{sec: experimental-details-fig2}.}
    \label{fig: dynamics}
    % \vspace{-.3cm}
\end{figure}

% Based on the previously defined setting, the following theorem illustrates the variation trends of the four parameters and the associated phase transition as shown in {\red graph}:

As shown in Figure~\ref{fig: dynamics}, a learning transition from the 4-gram mechanism to the induction head mechanism does occur in our setting. 
Moreover, the learning process exhibits a four-phase dynamics. The next theorem provides a precise characterization of the four phases, whose proof can be found in Appendix~\ref{appendix: dynamics: stage II}.

\begin{theorem}[Learning transition and $4$-phase dynamics]\label{thm: optimization}
% Under our theoretical setup, . 
% For simplicity, 
Let $\alpha^\star=\Omega(1)$ and $w^\star=\cO(1)$, and we consider the regime of small initialization ($0<\sigma_{\rm init}\ll1$) and long input sequences ($L\gg1$).
Then we have the following  results:

% \vspace{-.2cm}

\begin{itemize}[leftmargin=2em]
    \item {\bf Phase I (partial learning).} In this phase, most of the $4$-gram component in the mixed target is learned, while a considerable number of induction head component have not yet been learned. Specifically, let $T_{\rm I}=\cO(1)$, then we have the following estimates:
    \begin{align*}
    \cL_{\FG}(\theta(T_{\rm I}))\leq& 0.01\cdot\cL_{\FG}(\theta(0)),\\
    \cL_{\IH}(\theta(T_{\rm I}))\geq& 0.99\cdot\cL_{\IH}(\theta(0)).
    \end{align*}
    % \begin{align*}
    % \cL_{\FG}(\theta(T_{\rm I}))\leq 0.01\cdot\cL_{\FG}(\theta(0)),\quad
    % \cL_{\IH}(\theta(T_{\rm I}))\geq\left(1-\frac{1}{(1+\alpha^\star)L}\right)\cL_{\IH}(\theta(0)).
    % \end{align*}
    % \begin{align*}
    % T_{\rm I}:=\inf\big\{t>0:\cL_{\FG}(\theta(t))\leq 0.1\cdot\cL_{\FG}(\theta(0))\big\}=\cO\left(1\right);
    % \end{align*}
    % Moreover, at the end of Phase I, the induction head loss $\cL_{\IH}(T_{\rm I})=\Omega\left(\right)$

    % \vspace{-.2cm}
    
    \item {\bf Phase II (plateau) + Phase III (emergence).} In these two phases, the learning of the induction head first gets stuck in a plateau for $T_{\rm II}$ time, then is learned suddenly. 
    Specifically, denoted by an observation time $T_o=\Theta(L)$, we have the following tight estimate of the duration:
    \begin{align*}
    T_{\rm II}:=&\inf\Big\{t>T_o:\cL_{\IH}(\theta(t))\leq 0.99 \cdot\cL_{\IH}(\theta(T_o))\Big\}
    \\=&\Theta\left((\alpha^\star+1)^2L\log(1/\sigma_{\rm init})/{w^\star}^2\right);
    \\
    T_{\rm III}:=&\inf\Big\{t>T_o:\cL_{\IH}(\theta(t))\leq 0.01\cdot\cL_{\IH}(\theta(T_o))\Big\}
    \\=&\Theta\left((\alpha^\star+1)^2L\log(1/\sigma_{\rm init})/{w^\star}^2\right).
    \end{align*}
    During these phases, the parameter $w_{KQ}$ (for learning $w^\star$ in $\IH$) increases exponentially:
    $$w_{KQ}(t)=\sigma_{\init}\exp\left(\Theta\left(\frac{{w^\star}^2t}{(1+\alpha^\star)^2 L}\right)\right), \ t<T_{\rm III}.$$
    % This indicates that the learning of $w_{KQ}$  exhibits a plateau-emergence dynamics.
    % decreases as follows: for any $t\in[T_{\rm I},T_{\rm III}]$
    % \begin{align*}
    % \cL_{\IH}(\theta(t))-\cL_{\IH}(T_{\rm I})=-\cL_{\IH}(T_{\rm I})\exp\left(\Theta\left(\frac{{w^\star}^2(t-T_{\rm I})}{(1+\alpha^\star)^2 L}\right)\right).
    % \end{align*}

    % \vspace{-.2cm}
     
    \item 
    {\bf Phase IV (convergence).} In this phase, the loss converges toward zero. Specifically, the following convergence rates hold for all $t>T_{\rm III}:$
    \begin{align*}
        {\cL_{\FG}(\theta(t))}
        &=\cO\left(\frac{1}{t}\right),\\
        \cL_{\IH}(\theta(t))&=\cO\left(\exp\left(-\Omega\left(\frac{{w^\star}^2t}{(1+\alpha^\star)^2 L}\right)\right)\right),
        % {\cL_{\FG}(\theta(t))}
        % {\cL_{\FG}(\theta(T_{\rm III}))}
        % =\cO\left(\frac{1}{1+t-T_{\rm III}}\right),\quad 
        % \frac{\cL_{\IH}(\theta(t))}{\cL_{\IH}(\theta(T_{\rm III}))}=\exp\left(-\Omega\left(\frac{{w^\star}^2(t-T_{\rm III})}{(1+\alpha^\star)^2 L}\right)\right),
    \end{align*}
    and $\cL(\theta(t))=\cL_{\FG}(\theta(t))+\cL_{\IH}(\theta(t))$.
\end{itemize}
\end{theorem}

% \vspace{-.2cm}

By this theorem, the $4$-gram mechanism is first learned, taking time $T_{\rm I}$. Then, the learning of the induction head mechanism enters a plateau, taking time $T_{\rm II}$,
followed by a sudden emergence of learning, taking time $T_{\rm III}-T_{\rm II}$. Finally, the loss for both components converges to zero.

{\bf  The clear learning transition.}
When any one of $L,\alpha^\star,1/\sigma_{\rm init},1/w^\star$ is sufficiently large, Phase II lasts for $T_{\rm II}\gg 1$. During this phase, the $4$-gram component has been learned well  but the induction head component  remains underdeveloped, demonstrating a distinct learning transition. Moreover,
Theorem~\ref{thm: optimization} and its proof reveal two key factors  that drive this transition:

% \vspace{-.1cm}

\begin{itemize}[leftmargin=2em]
    \item {\bf Time-scale separation due to high- and low-order parameter dependence in self attention}. The learning of DP and RPE components differ in their parameter dependencies. DP component exhibits a quadratic dependence on the parameter $w_{KQ}$, while  RPE component shows linear dependence on the parameter $p$. With small initialization $\sigma_{\rm init}\ll1$, a clear time-scale separation emerges: $|\dot{w}_{KQ}|\sim w_{KQ}\ll1$ (DP, slow dynamics) and $|\dot{p}|\sim 1$ (RPE, fast dynamics). Consequently,  the induction head (fitted by DP) is learned much slower than   the 4-gram component (fitted by RPE).  
    This time-scale separation accounts for the term $\log(1/\epsilon_{\rm init})$ in the plateau time $T_{\rm II}$.   
    
    \item {\bf Speed difference due to component proportions in the mixed target.}
  The 4-gram target component and the induction-head component have differing proportions in the mixed target. A simple calculation shows:
$\cL_{\FG}(0)\sim{\alpha^\star}^2/({1+\alpha^\star})^2$;
If $w^\star=\cO(1)$, then $\cL_{\IH}(0)\sim{1}/[{(1+\alpha^\star)^2L}]$.
Notably, $\cL_{\IH}(0)$ is significantly smaller than $\cL_{\FG}(0)$. This proportion disparity accounts for the $(1 + \alpha^\star)^2 L$ term in the plateau time $T_{\rm II}$.
\end{itemize}

% First, the induction head represents the higher-frequency component compared to the bi-gram head, and, following the frequency principle, it is learned first. Second, the input data and targets are predominantly composed of n-gram-generated data, which also accelerates the learning of the n-gram component due to this disproportion.

% \begin{theorem}[Optimization and Phase Transition]\label{theorem: dynamics}
%     Consider training a two-layer attention-only Transformer model, as defined in Equation (\ref{model: dynamic}), using Gaussian-generated data and gradient flow under the following assumptions:
    % \begin{itemize}
    %     \item The initialization values $p_0, g_0, w_0, h_0$ are sufficiently small.
    %     \item The sequence length $L$ is sufficiently large.
    % \end{itemize}
    
    % Under these conditions, $\tilde{p}$ exhibits logarithmic growth and a phase transition occurs during training, characterized by the following stages:

% \end{theorem}

% \vspace{-.15cm}

{\bf Proof idea.} 
We highlight that our fine-grained analysis of  entire learning process is guided by two key observations: 1) the dynamics of the two heads can be decoupled; 2)  there exist a distinct transition point in the dynamics of each head, as shown in Figure~\ref{fig: dynamics} (right). These insights lead us to divide the analysis of each head into two phases: a monotonic phase and a convergence phase. 
Particularly, for the convergence phase, we introduce a novel Lyapunov function that leverages the unique dynamical structure of self-attention. This Lyapunov function may be of independent interest and offers potential for studying broader issues in self-attention dynamics.

\begin{remark}
    We conduct additional experiments to validate our theoretical insights into the training dynamics and learning transition across a wider range of scenarios. These includes using data distribution (Figure~\ref{fig: dynamics: boolean}) and optimization algorithms (Figure~\ref{fig: dynamics: Adam}) in high-dimensional settings, as well as training real-world transformers on natural language datasets (Figure~\ref{fig: dynamics: wiki}).
\end{remark}

\section{Experimental Validation}

% \vspace{-.05cm}

Our theoretical analysis provide insights into the approximation power and optimization dynamics of transformers.
To further validate these insights, we conduct a series of experiments, ranging from simple toy models to real-world  natural language training tasks. Due to space limitations, the detailed experimental results and setups are provided in Appendix~\ref{appendix: experiment}.

% \vspace{-.2cm}

\section{Conclusion}

% \vspace{-.05cm}

In this work, we present a comprehensive theoretical analysis of how transformers implement induction heads, examining both the approximation and optimization aspects. From the approximation standpoint, we identify the distinct roles of each transformer component in implementing induction heads of varying complexity. On the optimization side, we analyze a toy setting, where we clearly characterize how learning transitions from $n$-grams to induction heads. Looking forward, an important direction for future research is to investigate the dynamics of learning general induction heads, which are crucial for realizing stronger ICL capabilities.

\section*{Acknowledgments}

Mingze Wang is supported by Young Scientists (PhD) Fund of the National Natural Science Foundation of China (No.~124B2028) and the National Key Basic Research Program of China (No.~2015CB8560).
Lei Wu is supported by the National Key R\&D Program of China (No.~2022YFA1008200) and National Natural Science Foundation of China (No.~12288101). 
We thank Dr. Hongkang Yang for insightful discussions and Mingyu Xu, as well as the anonymous reviewers, for their valuable suggestions.

% \newpage

% We demonstrated that a 2-layer Transformer with Absolute Positional Encoding can effectively implement the retrieval and pasting mechanism, extending it to handle local semantic information.
% Furthermore, we analyzed the phase transition from n-gram to induction heads, showing that Transformers initially learn n-gram patterns and gradually shift to induction heads as input sequences grow longer. 
% This transition is driven by small initialization and gradient flow dynamics, highlighting the crucial role of sequence length and model initialization in enabling the emergence of in-context learning. 
% Our findings contribute to a deeper understanding of how Transformer networks acquire in-context learning capabilities, offering insights into the internal mechanisms governing this behavior.

% \bibliographystyle{plainnat}
% \bibliography{ref.bib}

%%%%%%%%%%%%%%%%%%%%%%%%%%%%%%%%%%%%%%%%%%%%%%%%%%%%%%%%%%%%%%%%%%%%%%%%%%%%%%%%%%%%%%%%%%%%%%%%%%%%%%%%%%%%%%%%%%%%%%%%%%%%%%%%%%%%%%%%%%%%%%%%%%%%%%%%%%%%%%%%%%%%%%%%%%%%%%%%%%%%

% \newpage
% \input{appendix/main}

\newpage

\appendix

\begin{center}
    \noindent\rule{\textwidth}{4pt} \vspace{-0.2cm}
    \LARGE \textbf{Appendix} % \\ ~\\[-0.5cm]
    \noindent\rule{\textwidth}{1.2pt}
\end{center}
% \wl{why new page here and also for different sections?}

\startcontents[sections]
\printcontents[sections]{l}{1}{\setcounter{tocdepth}{2}}

% \begin{center}
%     \noindent\rule{\textwidth}{1.0pt} 
%     \vspace{-0.25cm}
%     \LARGE \textbf{Appendix} % \\ ~\\[-0.5cm]
%     \noindent\rule{\textwidth}{1.0pt}
% \end{center}

% % \addtocontents{toc}{\setcounter{tocdepth}{1}}
% \startcontents[sections]
% \printcontents[sections]{l}{1}{\setcounter{tocdepth}{2}}
% % \tableofcontents

\vspace{1.cm}

\section{Related Works}
\label{section: related works}

% \vspace{-.1cm}

{\bf Empirical observations of induction head.}
The induction head mechanism was first identified by~\citet{elhage2021mathematical} in studying how two-layer transformers perform language modeling. Subsequently, 
\citet{olsson2022context}  conducted a  more systematic investigation, revealing two key findings: 1) induction head emerges abruptly during training, and 2) induction head plays a critical role in the development of in-context learning capabilities. 
To obtain a fine-grained understanding of how induction head emerges during training, 
recent studies have developed several synthetic settings~\citep{reddy2023mechanistic,edelman2024evolution,bietti2024birth}. Particularly,~\citet{bietti2024birth} successfully reproduced the fast learning of (global) bigrams and the slower development of induction head. Despite these efforts, a comprehensive theoretical understanding of how the induction head operates in two-layer transformers and how it is learned during training remains elusive.

{\bf Expressiveness of transformers.}
Theoretically,~\citet{dehghani2018universal,perez2021attention,wei2022statistically} explored the Turing-completeness of transformers;~\citet{yun2019transformers} established the universal approximation property of transformers.
Subsequent studies examined the efficiency of transformers in representing specific functions or tasks, such as sparse functions~\citep{edelman2022inductive}, targets with nonlinear temporal kernels~\citep{jiang2023approximation}, practical computer
programs~\citep{giannou2023looped}, long but sparse memories~\citep{wang2024transformers}, induction head~\citep{sanford2024one,sanford2024transformers,rajaraman2024transformers}, and memorization and reasoning~\citep{chen2024can}.
Besides, many studies suggest that transformers achieve in-context learning by approximating gradient-based iterations across various layers 
\citep{garg2022can,akyurek2022learning,von2023transformers,mahankali2023one,bai2023transformers,shen2023pretrained}.
Besides, several studies explored the limitation of transformer's expressivity, particularly in modeling formal languages or simulating circuits~\citep{hahn2020theoretical,weiss2021thinking,bhattamishra2020ability,merrill2022saturated,merrill2023expresssive}.
Among all these works, the most closely related to ours are~\citet{rajaraman2024transformers}, which examined a generalized induction head similar to our Eq.~\eqref{equ: induction head, type II}. Specifically, they showed that  multi-layer  transformers with single-head attention can implement this mechanism. In contrast, we prove that two-layer transformers are sufficient if multihead attention is used.
% Among these works, the most closely related to ours are~\citep{sanford2024one,sanford2024transformers}, which also provide a formulation of induction head that focuses exclusively on the most recent occurrence of \texttt{a}.
% However, the visualizations in~\citet{elhage2021mathematical} indicate that the induction head detects all occurrences of \texttt{a}. In contrast, our formulations capture this experimental observation, and we formalize a more generalized induction head.

% provided explicit approximation rates for Transformer in sequences modeling with inherent graph structures.
% Recently,~\citet{wang2024understanding} examined the expressive power of Transformers for sequence modeling with long but sparse memories, revealing distinct roles of various components in Transformers.

{\bf Training dynamics of transformers.}
To gain insights into the dynamics of training transformers, several studies have analyzed simplified transformers on toy tasks. These tasks include learning distinct/common tokens~\citep{tian2023scan}, leaning balance/inblanced features~\citep{huang2023context}, linear regression task~\citep{zhang2023trained,ahn2024transformers}, multi-task linear regression~\citep{chen2024training}, binary classification~\citep{li2024training}, transformer with diagonal weights~\citep{abbe2024transformers}, 
learning causal structure~\citep{nichani2024transformers}, sparse token selection task~\citep{wang2024transformers}, and learning $n$-gram Markov chain~\citep{chen2024unveiling}.
Additionally, studies such as those by~\citet{ataee2023max}, \citet{tarzanagh2023transformers} and~\citet{vasudeva2024implicit} have analyzed scenarios where transformers converge to  max-margin solutions. Furthermore, \citet{thrampoulidis2024implicit} has examined the implicit bias of next-token prediction.
Among these works, the most closely related to ours are \citet{nichani2024transformers} and \citet{chen2024unveiling}, which proved that two-layer transformers can converge to induction head solutions. 
In this work, we explore a setting where the target is a mixture of  4-gram and induction head. We show that two-layer transformers can effectively converge to this mixed target and provide a precise description of the learning process associated with each component. Importantly, we are able to capture the {\em abrupt transition} from learning 4-gram patterns to mastering the induction head mechanism---a critical phase  in the learning of induction heads, as highlighted in the seminal works \citep{elhage2021mathematical,olsson2022context}.

Now we discuss the relationship between our work and two closely related studies~ \citep{bietti2024birth,edelman2024evolution}.

{\bf Comparison with~\citet{bietti2024birth}.}
\begin{itemize}[leftmargin=2em]
    \item {\bf Approximation analysis:}
    
    \citet{bietti2024birth} focus primarily on the implementation of the vanilla induction head. In contrast, our study extends this analysis by investigating not only how two-layer transformers achieve vanilla induction heads (Eq.~\eqref{equ: induction head, type I}) but also how they implement generalized induction heads, i.e., in-context n-grams (Eqs.~\eqref{equ: induction head, type II} and~\eqref{equ: induction head, type III}).
    
    Furthermore, our work provides explicit approximation rate results, offering insights into the distinct roles of multiple heads, positional encoding, dot-product structure, and FFNs in implementing these induction heads. 

    \item {\bf Optimization analysis:}

    {\em Study objective:} While \citet{bietti2024birth} examines the transition from 2-gram to induction head, our work focuses on the transition from 4-gram to induction head.
    
    {\em study methods:} \citet{bietti2024birth} conducts extensive experiments supported by partial theoretical properties but does not fully characterize the training dynamics theoretically. In contrast, our study provides {\bf a precise theoretical analysis of the entire training process} in a toy model, uncovering the sharp transition from 4-gram to induction head.

    {\em Main insights:} \citet{bietti2024birth} emphasizes the the role of weight matrices as associative memories and the impact of data distributional properties. Our analysis, on the other hand, identifies two primary drivers of the transition: (1) the time-scale separation due to low- and high-order parameter dependencies in self-attention; (2) the speed differences caused by the relative proportions of the two components in the mixed target.
    
\end{itemize}

{\bf Comparison with~\citet{edelman2024evolution}.}
The primary connection between~\citet{edelman2024evolution} and our work lies in the optimization analysis. Specifically,~\citet{edelman2024evolution} focuses on the transition from uni-gram to bi-gram mechanisms in Markov Chain data. In contrast, our study investigates the transition from 4-gram to in-context 2-gram mechanisms (induction head).
Additionally, we theoretically identify two primary drivers of the transition: (1) the time-scale separation due to low- and high-order parameter dependencies in self-attention; (2) the speed differences caused by the relative proportions of the two components in the mixed target.

\vspace{1.cm}

\section{Experiments}\label{appendix: experiment}

\subsection{Experiments supporting approximation results}
\label{appendix: experiment: approximation}

\underline{1. Supporting our {\bf construction} in Theorem~\ref{theorem: type II}.}

{\bf Setup.} We linear probing experiments~\citep{alain2016understanding} on the transformers with $H=D=8$ trained in the above experiment (Figure~\ref{fig: approximation: hardness}). For each checkpoint model ${\rm TF}$, we denote its output in the first layer on the input sequence $X$ as ${\rm TF}^{(1)}(X)$. 
The probing loss is measured by $\text{dist}\left(X_{\cdot-n+1:\cdot};{\rm TF}^{(1)}(X)\right)=\min\limits_{P\in\bbR^{D\times n}}:\sum\limits_{s=n}^{L}\norm{X_{s-n+1:s}-{\rm TF}_s^{(1)}(X) P}$, where $n=4$, $L=10$, and $X=(x_1,\cdots,x_L)$ is generated by $x_i\overset{i.i.d.}{\sim}{\rm Unif}(\{\pm1\})$ with testing batch $1000$. The results are shown in Figure~\ref{fig: approximation: probing}.

\begin{figure}[!ht]
    \centering
     \includegraphics[width=0.4\linewidth]{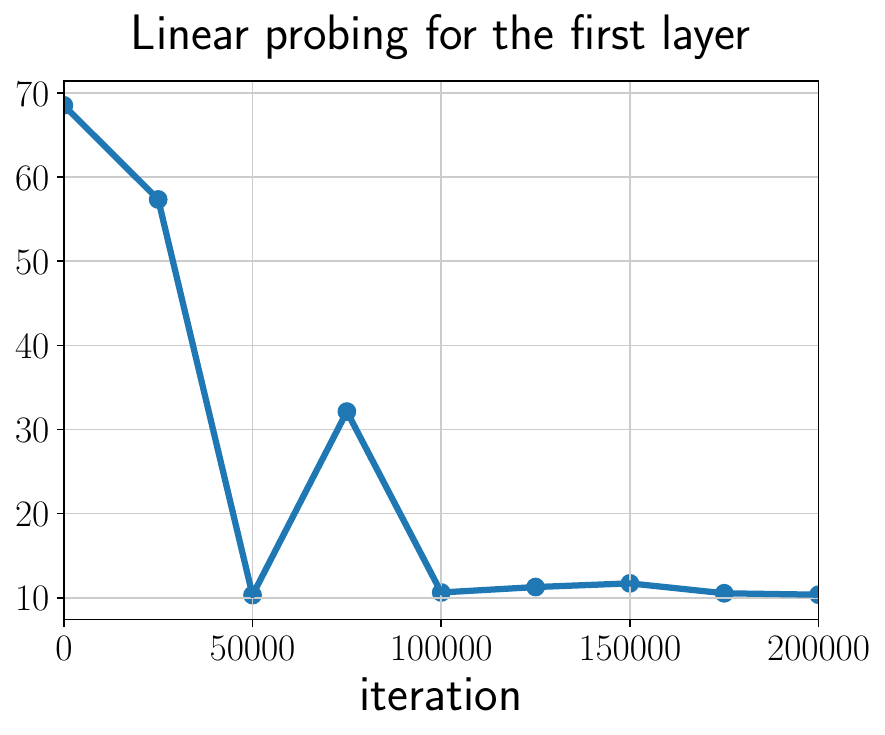}
    \caption{Probing results supporting our {\bf construction} in Theorem~\ref{theorem: type II}. 
    First, we train a two-layer two-layer transformer with head $H=8$ and embedding dimension $D=8$ to learn Eq.~\eqref{equ: induction head, type II} with $n=4$, and the checkpoints are stored during training.
    For each checkpoint model ${\rm TF}$, we denote its output in the {\em first layer} on the input sequence $X$ as ${\rm TF}^{(1)}(X)$. 
    To validate whether it encodes the semantic information $X_{s-n+2:s}$ near each $x_s$, as predicted by our construction, we conduct a standard linear probing experiment~\citep{alain2016understanding}. Specifically, we measured $\text{dist}\left(X_{\cdot-n+1:\cdot};{\rm TF}^{(1)}(X)\right)=\min\limits_{P\in\bbR^{D\times n}}:\sum\limits_{s=n}^{L}\norm{X_{s-n+1:s}-{\rm TF}_s^{(1)}(X) P}$.
    As the results shown, the probing loss decreases significantly during training, confirming our key construction in Theorem 4.3: {\bf the first layer is responsible for extracting local semantic information $X_{s-n+2:s}$ near each $x_s$}, enabling the second layer to generate the final output.}
    \label{fig: approximation: probing}
\end{figure}

\vspace{2.cm}

\underline{2. Supporting the {\bf necessity} of the required $H$ and $D$ in Theorem~\ref{theorem: type II}.}

{\bf Setup.} We train two-layer transformers (without FFN layers) with varying $H$ and $D$ to learn the generalized induction head~\eqref{equ: induction head, type II} with $n=4$.
The input sequence $X=(x_1,\cdots, x_L)$ is boolean, with $x_i\overset{i.i.d.}{\sim}{\rm Unif}(\{\pm1\})$ and $L=10$.  
Each model is trained for 200,000 iterations using squared loss and (online) Adam optimizer with learning rate \texttt{5e-4} and batch size $B=100$. Both layers are trained simultaneously. The results for the models with $D=H=8$ and $D=H=2$ are presented in Figure~\ref{fig: dynamics: wiki}.

\begin{figure}[!ht]
    \centering
    \includegraphics[width=0.4\linewidth]{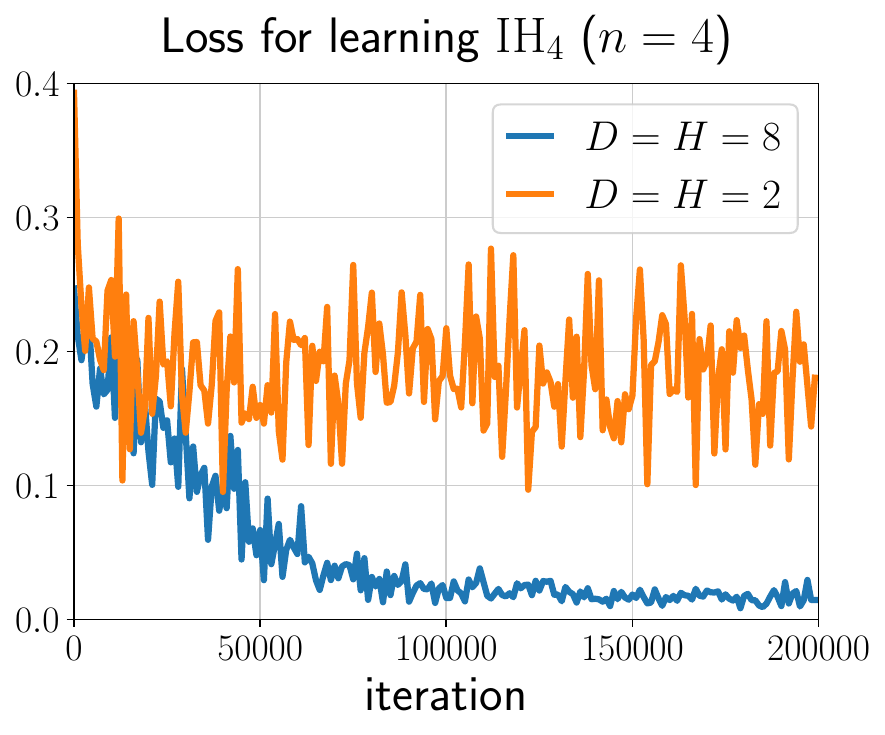}
    \caption{Results supporting the {\bf necessity} of the required number of heads $H$ and embedding dimension $D$ in Theorem~\ref{theorem: type II}. 
    We train two-layer transformers with varying $H$ and $D$ to learn the target in Eq.~\eqref{equ: induction head, type II} with $n=4$. 
    The results indicate that the transformer with $H=D=8$ ($>n$) successfully expresses this task, while the transformer with $H=D=2$ ($<n$) fails. 
    These results confirm that the sufficient conditions provided in Theorem~\ref{theorem: type II} ($H\gtrsim n$ and $D\geq nd$, where $d=1$ in our setting) are also nearly necessary.}
    \label{fig: approximation: hardness}
\end{figure}

\subsection{Additional experiments supporting optimization dynamics}
\label{appendix: experiment: dynamics}

\underline{1. {\bf Standard transformers on real-world natural language dataset.}}

{\bf Setup.} We train a two-layer two-head {\bf standard transformer} with RPE~\eqref{equ: RPE} (without any simplification) on the {\bf wikitext-2} dataset, a natural language dataset~\citep{merity2016pointer}.
The transformer has an embedding dimension $D=128$ and FFN width $W=512$. For this dataset, the input dimension is $d=33278$. We use a context length $L=200$ and batch size $B=32$. The parameters are initialized with the scale $0.01$.
The model is trained for 1,500 epochs on 1 H100, using cross-entropy loss and SGD with learning rate $0.1$, and the initialization scale is $0.01$. It is important to note that  {\bf both layers are trained simultaneously}. The results are presented in Figure~\ref{fig: dynamics: wiki}.

\begin{figure}[!ht]
    \centering
    \includegraphics[width=0.32\linewidth]{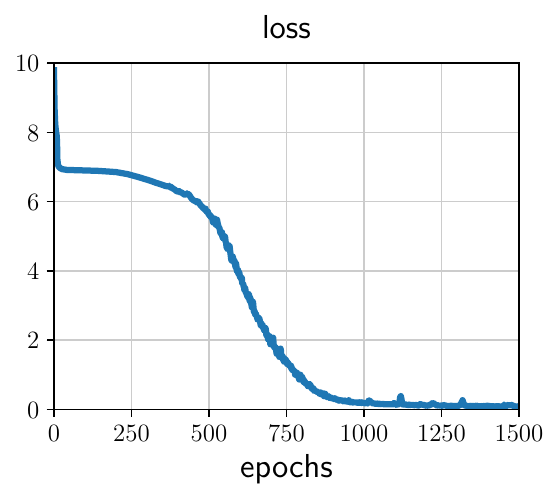}
     \includegraphics[width=0.39\linewidth]{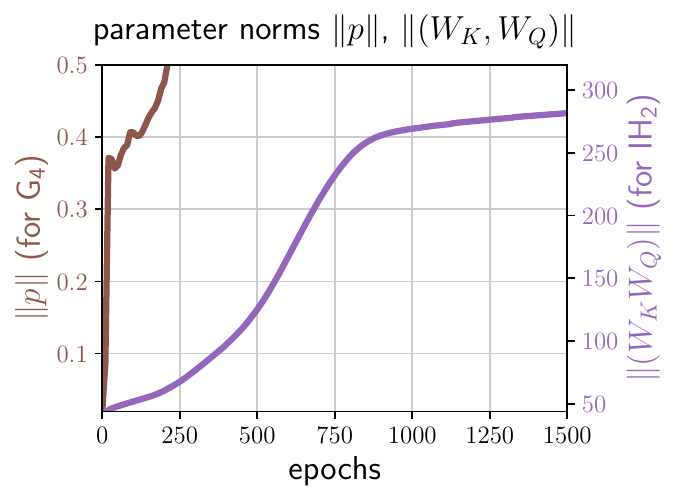}
    \caption{The loss and parameters for the experiment training a two-layer two-head {\bf standard transformer} (without any simplification) on the {\bf wikitext-2} dataset~\citep{merity2016pointer}. Here, $\|p\|$ and $\|(W_K,W_Q)\|$ denote the Frobenius norms of all positional encoding parameters and all $W_K,W_Q$ parameters across layers and heads, respectively,
    The results show that: the loss exhibits a clear plateau; position encoding $p$'s are learned first; and the dot-product structure $W_K,W_Q$ are learned slowly at the beginning, resembling an exponential increase; additionally, as $W_K,W_Q$ are learned, the loss escapes that plateau. These findings closely resemble the behavior observed in our toy model (Figure~\ref{fig: dynamics}). This experiment provides further support for our theoretical insights regarding the {\bf time-scale separation} between the learning of positional encoding and the dot-product structure.
    }
    \label{fig: dynamics: wiki}
\end{figure}

\vspace{.2cm}

\underline{2. {\bf Discrete token distribution in toy setting.}}

{\bf Setup.} We modified the Gaussian input distribution used in the setup for Figure~\ref{fig: dynamics} to a boolean input distribution, where each input token, where each input token $x_i\stackrel{iid}{\sim} {\rm Unif}(\{\pm1\})$ for $i\in [L]$, 
All other experimental setups remain the same as in the setup for Figure~\ref{fig: dynamics}. The training dynamics of Stage (ii) are presented in Figure~\ref{fig: dynamics: boolean}. We can see clearly that the dynamical behavior of the learning process is nearly the same as the one observed for Gaussian inputs in Figure \ref{fig: dynamics}.

\begin{figure}[!ht]
    \centering
    \includegraphics[width=0.305\linewidth]{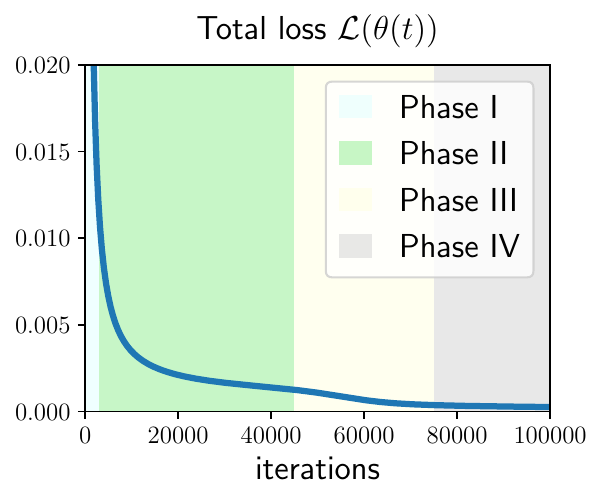}
    \includegraphics[width=0.38\linewidth]{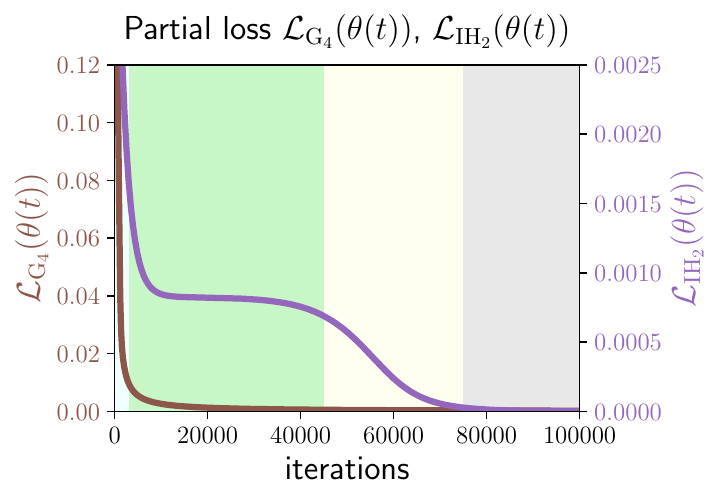}
    \includegraphics[width=0.295\linewidth]{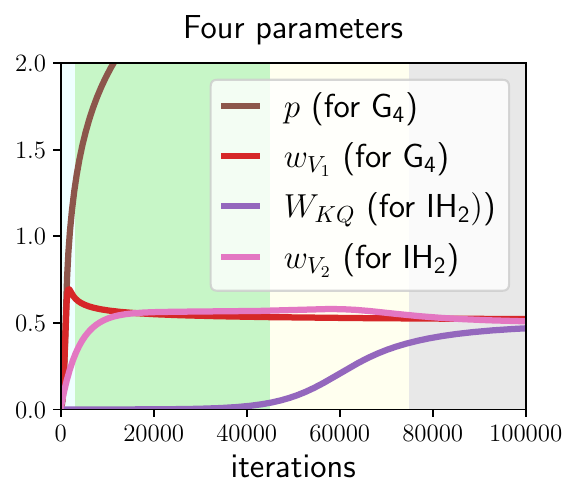}
    \caption{Visualization of the total loss, partial loss, and the parameter dynamics, for the experiment on {\bf discrete token distribution} (Boolean, $X\sim{\rm Unif}(\{\pm1\}^L)$) in our toy setting with $\alpha^\star=1,w^\star=0.49,\sigma_{\rm init}=0.01,L=40$. 
    The figure clearly shows that transformer learns the $4$-gram component first and then, starts to learn the induction head mechanism. Notably, the entire dynamics exhibit four phases. These results are  {\bf extremely similar} to that observed with Gaussian inputs, as shown in Figure~\ref{fig: dynamics}.}
    \label{fig: dynamics: boolean}
\end{figure}

\vspace{.2cm}

\underline{3. {\bf Adam in high-dimensional toy setting.}}

{\bf Setup.} We modified the setup for Figure~\ref{fig: dynamics} to employ a high-dimensional model $(D=100)$. Specifically, the target is $w^\star=0.49 I_D/D$, the dot-produce parameters are $W_K,W_Q\in\bbR^{D\time 1}$, initialized such that $\|W_K\|_F,\|W_Q\|_F=\sigma_{\rm init}$. Additionally, for the Adam optimizer, we use learning rate \texttt{5e-4}. All other experimental setups remain the same as in the setup for Figure~\ref{fig: dynamics}. 

The training dynamics are depicted in Figure~\ref{fig: dynamics: Adam}, where, for comparison, results using GD are also presented. In both scenarios, the learning process begins with the 4-gram pattern, followed by a gradual learning phase of the induction head mechanism. Notably, within the given number of iterations, GD remains stuck in the plateau, whereas Adam successfully escapes that plateau.
\begin{figure}[!ht]
    \centering
    \includegraphics[width=0.4\linewidth]{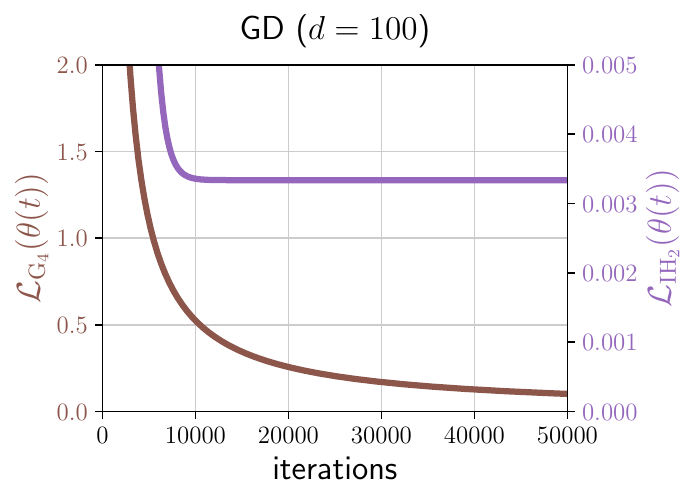}
     \includegraphics[width=0.4\linewidth]{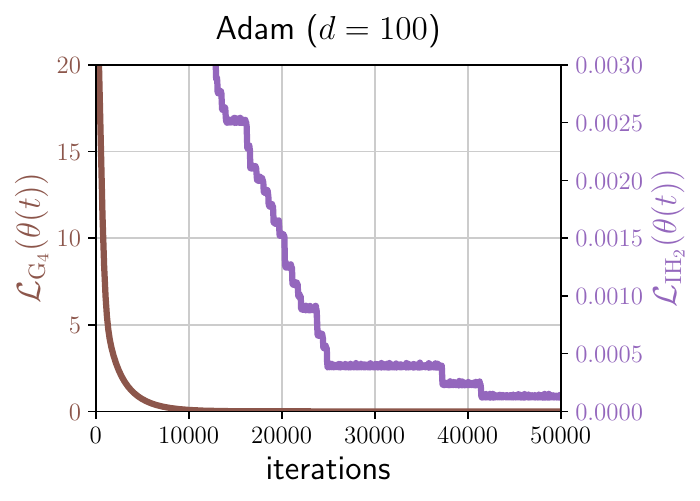}
    \caption{Partial loss for the experiment comparing {\bf GD v.s. Adam optimizer} in high-dimensional settings ($D=100$).
    In this setting, a larger $D$ increases the difficulty of the transition from the lazy regime (learning $4$-gram) to the rich regime (learning induction head).
    The results indicate that: (1) GD learns the 4-gram component first but becomes stuck in a plateau when learning induction head; (2) Adam, while eventually transitioning from the lazy regime (learning 4-gram) to the rich regime (learning induction head), experiences a {\bf challenging} transition characterized by {\bf multiple plateaus} during learning induction heads. This finding closely resembles the dynamics for GD.}
    \label{fig: dynamics: Adam}
\end{figure}

\subsection{Experimental details for  Figure~\ref{fig: dynamics}}
\label{sec: experimental-details-fig2}

In line with our theoretical setting, we examine a simplified two-layer transformer, as  described in~\eqref{equ: output}. Specifically, the first layer only contains RPE~\eqref{equ: RPE} and the second layer consists of two heads: one uses only RPE and the other  employs only dot-product structure. The target function is specified by~\eqref{target: optimization} with
$\alpha^\star=1,w^\star=0.49,\sigma_{\rm init}=0.01,L=40$, and the  distribution of each token is Gaussian, i.e., $x_i\stackrel{iid}{\sim}\cN(0,1)$ for $i\in [L]$. Training is conducted by minimizing the squared loss~\eqref{equation: loss} using online SGD with learning rate $0.1$ and batch size $B=1,000$.
Following our theoretical analysis, the two layers are trained sequentially:
\begin{itemize}[leftmargin=2em]
\item Training Stage I: only the first layer is trained for 100,000 iterations; 
\item Training Stage II: Subsequently, only the second layer undergoes training for another 100,000 iterations.
\end{itemize}
 The dynamical behavior of the Training Stage II is visualized in Figure~\ref{fig: dynamics}.

\vspace{1.cm}

\section{Proofs in Section \ref{section: approximation}}
\label{appendix: approximation}

% {\bf Proof Setting.} 
% For the simplicity of proof, we consider the infinite sequence, i.e., $\bX_t=(\bx_s)_{s\leq t}$, where $\bx_s=\bzero$ for all $s\leq 0$, and $\exp(-\infty)=0$. Notice that this technical simplification does not violate our motivations.

% {\bf Absolute} Positional Encoding (APE). The position information of each token is embedded directly into the input sequence $X_t=(x_1,\cdots,x_t)$.
%     For simplicity, we consider the following form: 
%     $$\bx_s\mapsto\tilde{x}_s=(\bx_{s}^\top,s)^\top,\ s\in[t].$$
%     The sequence $\tilde{X}_t=(\tilde{x}_1,\cdots,\tilde{x}_t)$ is then fed into the embedding layer and processed by the multi-layer Transformer.

\subsection{Proof of Theorem \ref{theorem: standard}}\label{subsecton: thm standard proof}
{\begin{equation}\label{equ: induction head, type I, restate}
    \IH(\bX_L)=(\bx_{s})_{s=2}^{L-1}\ {\sm}    \Big(\big(x_{L}^\top W^\star x_{s-1}\big)_{s=2}^{L-1}\Big)^\top,
\end{equation}}

\begin{theorem}[Restatement of Theorem \ref{theorem: standard}]\label{theorem: standard restate}
Let $\IH$ satisfy Eq.~\eqref{equ: induction head, type I, restate}. 
Then, there exists a constant $C>0$ and a two-layer single-head transformer $\TF$ (without FFNs), with $D=2d$, $W_K^{(1,1)}=W_Q^{(1,1)}=0$,
% $|p^{(1,1)}|\leq\cO(\log(\norm{W^\star}_{1,1}/\epsilon))$,
$p^{(2,1)}=0$, and $\|W_K^{(2,1)}\|,\|W_Q^{(2,1)}\|\leq\cO(1,\|W^\star\|_{F})$, such that
\begin{align*}
\sup_{L\in\bbN^{+}}\triplebar{\IH-\TF}_{L,\infty}\leq \frac{C}{e^{p^{(1,1)}}}.
\end{align*}
\end{theorem}

\begin{proof}

We consider two-layer single-head transformer without FFN, where the first layer has the residual block, while the second layer does not have the residual block.

We first embed each token into $\mathbb{R}^{D}$ as 
$
\begin{pmatrix}
x_s
\\0
\end{pmatrix}
$ and take $W_V^{(1)}=\begin{pmatrix}0&0\\I_{d\times d}&0\end{pmatrix}$
, then the $s$-th output token of the first layer is 
\begin{align*}
\begin{pmatrix}
x_s
\\y_s
\end{pmatrix}
=\begin{pmatrix}
x_s
\\ (x_{\tau})_{\tau=1}^{s-1}\ {\sm}\left(\big(-p^{(1,1)}(s-1-\tau)\big)_{\tau=1}^{s-1}\right)^\top
\end{pmatrix}.
\end{align*}

Then for the second layer, we choose $p^{(2,1)}=0$,
\begin{align*}
W_Q^{(2,1)}=\begin{pmatrix}0&0\\I_{d\times d}&0\end{pmatrix},\ W_K^{(2,1)}=\begin{pmatrix}0&0\\0&W^\star\end{pmatrix},\ 
W_V^{(2,1)}=\begin{pmatrix}
I_{d\times d} & 0
\\
0 & 0
\end{pmatrix}\in\bbR^{D\times D},
\end{align*}
and the projection $W_O^{(2)}=\begin{pmatrix}
    I_{d\times d} & 0_{d\times d}
\end{pmatrix}\in\bbR^{d\times D}$.

Then the last output token of the second layer is 
 $$(x_s)_{s=2}^{L-1}\ {\sm}\Big(\big(x_L^{\top}W^\star y_s\big)_{s=2}^{L-1}\Big)^\top.$$

By Lemma \ref{lemma: softmaxlipschitz} 
% and take  $p^{(1,1)}=\log(4\norm{W^\star}_{1,1}/\epsilon+1)$
, for any $L\in\mathbb{N}^+$ 
\begin{align*}
&\triplebar{\IH-\TF}_{L,\infty}
\\=&\sup_{X_L}\|\textsf{IH}(X_L)-\TF_{-1}(X_L)\|_{\infty}
\\=& \left\|{ (x_s)_{s=2}^{L-1}\ \sm \left(\big(x_L^{\top}W^\star y_s\big)_{s=2}^{L-1}\right)^\top}-{ (x_s)_{s=2}^{L-1}\ \sm\left(\big(x_L^{\top}W^\star x_{s-1}\big)_{s=2}^{L-1}\right)^\top}\right\|_{\infty}
\\\leq& \|(x_s)_{s=2}^{L-1}\|_{\infty,\infty} \left\|{ \sm\left(\big(x_L^{\top}W^\star y_{s}\big)_{s=2}^{L-1}\right)^\top}-{\sm\left(\big(x_L^{\top}W^\star x_{s-1}\big)_{s=2}^{L-1}\right)^\top}\right\|_1
\\\leq& 2 \sup_{2\leq s\leq L-1} \left|x_L^{\top}W^\star y_s-x_L^{\top}W^\star x_{s-1}\right|
\\\leq& 2\|x_L^{\top}W^\star \|_1\sup_s\|y_s-x_{s-1}\|_{\infty}
\\\leq& 2 \sum_{i,j}|W_{i,j}^\star |\sup_s\left\|(\bx_\tau)_{\tau=1}^{s-1}\ \sm\left(\left(-p^{(1,1)}(s-1-\tau)\right)_{\tau=1}^{s-1}\right)^\top-x_{s-1}\right\|_{\infty}
\\\leq& 2\|W^\star \|_{1,1}\sup_s\left\|\sm\left(\left(-p^{(1,1)}(s-1-\tau)\right)_{\tau=1}^{s-1}\right)^\top-\be_{s-1}\right\|_1
\\=& 2\|W^\star\|_{1,1} \sup_{2\leq s\leq L-1}\frac{\sum_{\tau=1}^{s-2} \exp\left(-p^{(1,1)}(s-1-\tau)\right)}{\sum_{\tau=1}^{s-1} \exp\left(-p^{(1,1)}(s-1-\tau)\right)}
\\\leq&2\|W^\star\|_{1,1} \lim_{s\to+\infty}\frac{\sum_{\tau=1}^{s-2} \exp\left(-p^{(1,1)}(s-1-\tau)\right)}{\sum_{\tau=1}^{s-1} \exp\left(-p^{(1,1)}(s-1-\tau)\right)}
=2\|W^\star\|_{1,1} e^{-p^{(1,1)}}.
% \\\leq& 4\|W^\star \|_{1,1}\frac{e^{-p^{(1,1)}}}{1-e^{-p^{(1,1)}}}
% \leq \frac{e^{-p^{(1,1)}}}{1-e^{-p^{(1,1)}}}.
\end{align*}

% For the parameters selected above, $\sup_{L\in\bbN^{+}}\triplebar{\IH-\TF}_{L,\infty} <\epsilon$.
\end{proof}

% \todo[inline]{wmz: absolute positional encoding version for Theorem 4.1.}

\subsection{Proof of Theorem \ref{theorem: type II}}\label{subsecton: thm type II proof}

\begin{equation}\label{equ: induction head, type II, restate}
    \nIH(\bX_L)=(x_s)_{s=n}^{L-1}\ \sm\Big(\big(X_{L-n+2:L}^\top W^\star X_{s-n+1:s-1}\big)_{s=n}^{L-1}\Big)^\top,
\end{equation}

\begin{theorem}[Restatement of Theorem \ref{theorem: type II}]
Let $\nIH$ satisfy Eq.~\eqref{equ: induction head, type II, restate}. 
Then for any  $H\in\bbN^{+}$ and rate $q\in\bbN^+$, there exists a constant $C_{n,q}>0$ and a two-layer $H$-head transformer $\TF(\cdot)$ (without FFNs), with $D=nd$, such that:
\begin{align*}
    \sup_{L\in\bbN^+}\triplebar{\nIH-\TF}_{L,\infty}\leq \left(\frac{C_{n,q}}{H}\right)^q,
\end{align*}
where $C_{n,q}=\cO(n q^2)$.
\end{theorem}

\begin{proof}

We consider two-layer multi-head transformer without FFN, where the first layer has the residual block, while the second layer does not have the residual block.

First, we choose the embedding dimension $D = nd$, and parameters in the embedding map 
$$\bW_E = \begin{pmatrix}
    I_{d\times d}\\
    0_{(D-d)\times d}
\end{pmatrix}\in\mathbb{R}^{D\times d},\quad b_E=0\in\mathbb{R}^D,$$

% We set $x_{L+1}=0$ at the first place, 
then each token $\bx_s^{(0)}$ after embedding is
$$\bx_s^{(0)} = \bW_Ex_s + \bb_E = \begin{pmatrix}
    \bx_s\\0
\end{pmatrix}\in\mathbb{R}^{D}.$$

% For $2$-layer dot-product Transformer 
% $\TF(X)=X^{(2)}\in\bbR^{D\times L}$:
% \begin{equation*}\label{model: Transformer}
% \begin{aligned}
%     \bX^{(u-\frac{1}{2})}&=\bX^{(u-1)}+\SA^{(u)}(\bX^{(u-1)}),\quad u\in[U];
%     \\
%     \bX^{(u)}&=\bX^{(u-\frac{1}{2})}+\FFN^{(u)}(\bX^{(u-\frac{1}{2})}),\quad u\in[U].
% \end{aligned}
% \end{equation*}
% where
% \begin{align*}
%     \SA^{(u)}(x_L)
% =\sum_{h=1}^H \sum_{s=0}^{L-1}\frac{\bW_V^{(u,h)}x_{L-s}\exp\left(\<\bW_Q^{(u,h)}x_{L},\bW_K^{(u,h)}x_{L-s-1}\>+R_{L,L-s-1}^{(u,h)}\right)}{\sum_{j=0}^{+\infty}\exp\left(\<\bW_Q^{(u,h)}x_{L},\bW_K^{(u,h)}x_{L-j-1}\>+R_{L,L-j-1}^{(u,h)}\right)}.
% \end{align*}

This proof can be summarized as the following process for $\TF_{-1}$:
\begin{align*}
    	(x_s)_{s=n}^{L-1}\ \sm\Big(\big(\hat{X}_{L-n+2:L}^\top & W^\star \hat{X}_{s-n+1:s-1}\big)_{s=n}^{L-1}\Big)^\top
	\\
	\text{Step II. 2-st Attn}\ &\uparrow
	\\ 
        (\bx_L^\top,\hat{\bx}_{L-1},\ldots,&\hat{\bx}_{L-n+1})^\top
	\\
	\text{Step I. 1-st Attn}\ &\uparrow
	\\
	(\bx_{L}^\top,&0^\top)^\top
\end{align*}

Additionally, in this proof, the following projection matrices are used:
$$P_i:=\left(0_{d\times (i-1)d},I_{d\times d},0_{d\times(D-id)}\right)\in\mathbb{R}^{d\times D}, \quad i\in[n].$$

{\bf Step I. The first layer.}
We use 1-st Attn with residual to copy the previous tokens $(x_{s-n+1},\cdots,\bx_{s-1})$ of each token $x_{s}$.
We use $H=\sum_{i=1}^{n-1}H_i$ attention heads to realize this step.

By lemma \ref{lemma E.1}, for any rate $q\in\mathbb{N}^+$, there exists a function 
$$\phi_i^{\exp}(t)=\sum_{h=1}^{H_i}\alpha_{h,i} e^{-\beta_{h,i} (t-1)}$$
such that $\beta_h>0$ and 
$$\Vert \mathbb{I}\left\{\cdot=i\right\}-\phi_i^{\exp}(\cdot) \Vert_{\ell_1(\mathbb{N})}=\sum_{s=i}^{+\infty}|\mathbb{I}\left\{s=1\right\}-\phi^{\exp}(s)|\leq  C\frac{A^q (q^2)^q e^{0.01(q+1)i}}{H_i^{q}},
$$
where $A,C>0$ are absolute constants.

For $h=\sum_{j=1}^{i-1}H_j,1+\sum_{j=1}^{i-1}H_j,\ldots,\sum_{j=1}^{i}H_j$, we choose parameters as follows
$$p^{(1,h)}=\beta_{h,i},\quad \bW_V^{(1,h)}=\alpha_{h,i}\left(\sum_{j=0}^{H_i}\exp(-\beta_{h,i} (j-1))\right)\bS_i,$$
$$\bW_{K}^{(1,h)}=\bW_{Q}^{(1,h)}={0},\quad\bW_O^{(1)}=\bI_{D\times D}$$
where $\bS_{i}\in\mathbb{R}^{D\times D}$ is a shift matrix that takes out the first $d$ elements of a vector and shifts it backward to the $(id+1)$-th to $(i+1)d$-th elements. Then
\begin{align*}
	\left(P_{i+1}\sum_{h=\sum_{j=1}^{i-1}H_j}^{\sum_{j=1}^{i}H_j}\SA^{(1,h)}(\bX_L^{(0)})\right)_{-1}=\sum_{h=\sum_{j=1}^{i-1}H_j}^{\sum_{j=1}^{i}H_j}\alpha_{h,i}\sum_{s=1}^{L-1}e^{-\beta_{h,i} (s-1)}\bx_{L-s}.
\end{align*}

We denote the $s$-th output token of the first layer 
$$
\bx_s^{(1)}:=\SA^{(1)}\left(\bX_{0:s}^{(0)}\right)_{-1},
$$

Then the approximation error of this step is 
\begin{align*}
	&\varepsilon_{\SA}^{(1)}:=\sup_{s}\left\| \bx_s^{(1)}-\bX_{s-n+1:s}\right\|_{\infty} 
    \leq\sup_s\left(\|x_{s}-x_s\|_{\infty}+\sum_{i=1}^{n-1}\left\|P_{i+1}\bx_s^{(1)}-\bx_{s-i}\right\|_{\infty}\right)
    \\=&\sup_{s}\sum_{i=1}^{n-1}\left\|P_{i+1}\bx_s^{(1)}-\bx_{s-i}\right\|_{\infty}
    \leq\sup_s\sum_{i=1}^{n-1}\Vert \mathbb{I}\left\{\cdot=i\right\}-\phi_i^{\exp}(\cdot) \Vert_{\ell_1(\mathbb{N})}
    \\\leq& 
    C(Aq^2)^q\sum_{i=1}^{n-1}\frac{e^{0.01(q+1)i}}{H_i^q}.
\end{align*}
Consequently, one detail is to assign the head number $\{H_i\}_{i=1}^n$ such that the error’s sum
$\sum_{i=1}^{n-1}\frac{e^{0.01(q+1)i}}{H_i^q}$ is as small as possible. Our way is solving the minimization problem
\begin{align*}
    \min:&\sum_{i=1}^{n-1}\frac{e^{0.01(q+1)i}}{H_i^q}
    \\{\rm s.t.}\  &\sum_{i=1}^{n-1} H_i=H,
\end{align*}
which suggests that we should choose the head number: 
$$H_i=\frac{e^{0.01i}}{\sum_{j=1}^{n-1}e^{0.01j}},i\in[n-1].$$
Thus, we obtain the bound
\begin{align*}
    \varepsilon_{\SA}^{(1)}\leq \frac{C(Aq^2)^q}{H^q}\left(\sum_{i=1}^{n-1}e^{0.01 i}\right)^{q}\leq C\left(\frac{ Aq^2 n e^{0.01 n} }{H}\right)^q.
\end{align*}
Noticing $n<100$, we have:
\begin{equation*}
\label{equ: approx proof of Thm 2: layer 1 error}
    \varepsilon_{\SA}^{(1)}\leq C\left(\frac{e Aq^2 n }{H}\right)^q.
\end{equation*}

We focus on \underline{\bf the case of large $H$:}
$$H\geq CAenq^2,$$ 
which ensures $\varepsilon_{\SA}^{(1)}\leq 1$. Therefore, $\bx_s^{(1)}\in [-2,2]^{D}$: 
\begin{gather*}
\norm{\bx_s^{(1)}}_\infty\leq\sup_{s}\left\| \bx_s^{(1)}-\bX_{s-n+1:s}\right\|_{\infty}+ \norm{\bX_{s-n+1:s}}_\infty\leq \varepsilon_{\SA}^{(1)}+1\leq 2,\quad\forall s.
\end{gather*}
% Thus, we have the bound:
% \begin{align*}
%     \norm{\hat{\bX}_{s-n+1:s-1}}_{\infty}\leq 2,\quad\forall s.
% \end{align*}

{\bf Step II. The second layer.}
For the second $\text{Attn}$, we only need use the first head (by setting $W_{V}^{(2,h)}=0$ for $h\geq 1$). Specifically, we choose $p^{(2,1)}=0$,
\begin{align*}
W_Q^{(2,1)}=\begin{pmatrix}0&0\\I_{(D-d)\times (D-d)}&0\end{pmatrix},\ W_K^{(2,1)}=\begin{pmatrix}0&0\\0&W^\star\end{pmatrix},\ 
W_V^{(2)}=\begin{pmatrix}
I_{d\times d} & 0
\\ 0 & 0
\end{pmatrix}\in\bbR^{D\times D},
\end{align*}
and the projection $W_O^{(2)}=\begin{pmatrix}
    I_{d\times d} & 0_{(D-d)\times d}
\end{pmatrix}\in\bbR^{d\times D}$.

For simplicity, we use the following notations:
\begin{gather*}
\hat{X}_{s-n+1:s-1}:=W_{Q}=\begin{pmatrix}
P_2{\bx}_{s}^{(1)}\\\vdots\\ P_{n} {\bx}_{s}^{(1)}
\end{pmatrix},\ \text{ for } n\leq s\leq L-1;\quad\quad
\hat{X}_{L-n+2:L}:=\begin{pmatrix}
P_1{\bx}_{L}^{(1)}\\\vdots\\ P_{n-1} {\bx}_{L}^{(1)}.
\end{pmatrix}.
\end{gather*}

Then the output of the second layer is 
$$\bx_L^{(2)}=(x_s)_{s=n}^{L-1}\ \sm\Big(\big(\hat{X}_{L-n+2:L}^\top W^\star \hat{X}_{s-n+1:s-1}\big)_{s=n}^{L-1}\Big)^\top$$

By using these bounds and Lemma \ref{lemma: softmaxlipschitz}, 
\begin{align*}
    &\left\|\bx_L^{(2)}-(x_s)_{s=n}^{L-1}\ \sm\Big(\big(X_{L-n+2:L}^\top W^\star X_{s-n+1:s-1}\big)_{s=n}^{L-1}\Big)^\top\right\|_\infty
    \\\leq& \sum_{s=n}^{L-1}\Bigg|\sm\Big(\big(\hat{X}_{L-n+2:L}^\top W^\star \hat{X}_{s-n+1:s-1}\big)_{s=n}^{L-1}\Big)^\top -\sm\Big(\big(X_{L-n+2:L}^\top W^\star X_{s-n+1:s-1}\big)_{s=n}^{L-1}\Big)^\top\Big)\Bigg|
    \\ \leq & 2\max_s \left|\hat{X}_{L-n+2:L}^\top W^\star \hat{X}_{s-n+1:s-1}-{X}_{L-n+2:L}^\top W^\star {X}_{s-n+1:s-1}\right|
    \\\leq&
    2\Big(\max_s \left|\left(\hat{X}_{L-n+2:L}-{X}_{L-n+2:L}\right)^\top W^\star \hat{X}_{s-n+1:s-1}\right| + \max_s \left| X_{L-n+2:L}^\top W^\star \left(\hat{X}_{s-n+1:s-1}-{X}_{s-n+1:s-1}\right) \right|\Big)
    \\\leq&
    2\Big(2\|W^\star\|_{(1,1)}\left\|\hat{X}_{L-n+2:L}-{X}_{L-n+2:L}\right\|_{\infty} + \|W^\star\|_{1,1} \max_s \left\|\hat{X}_{s-n+1:s-1}-{X}_{s-n+1:s-1}\right\|_\infty\Big)
    \\ \leq & 6\|W^\star\|_{1,1}\cdot \varepsilon_{\SA}^{(1)}.
\end{align*}

Since the above inequality holds for any $L$ and $\bX_L$, we obtain our bound: 
% for any $H\geq CAenq^2$, there exists a two-layer $H$-head Transformer \TF (without FFN layers) such that:
$$\sup_{L\in\bbN^+}\triplebar{\nIH-\TF}_{L,\infty}\leq 6C\|W^\star\|_{(1,1)}\left(\frac{6e A n q^2}{H}\right)^q
\leq \left(\frac{6C\|W^\star\|_{1,1} e A n q^2}{H}\right)^q,$$

For \underline{\bf the remained case}, $H< CAenq^2$, we can simply choose $\TF$ with all $0$ parameters. Then the approximation error can be trivially bounded by: 
$$\sup_{L\in\bbN^+}\triplebar{\nIH-\TF}_{L,\infty}= \sup_{L\in\bbN^+}\triplebar{\nIH-0}_{L,\infty}\leq 1.$$

Then, the two cases can be unified by: 
\begin{align*}
    \sup_{L\in\bbN^+}\triplebar{\nIH-\TF}_{L,\infty}\leq\begin{cases}
        6C\|W^\star\|_{1,1}\left(\frac{e A n q^2}{H}\right)^q,\ H\geq CAenq^2
        \\
        1,\ \text{otherwise}
    \end{cases}\leq\left(\frac{\left(6\|W^\star\|_{1,1}+1\right) C e A n q^2}{H}\right)^q.
\end{align*}

In the theorem, we can choose $C_{n,q}=\left(6\|W^\star\|_{1,1}+1\right) Ce A n q^2$, which satisfies $C_{n,q}\lesssim n q^2$.
\end{proof}

\subsection{Proof of Theorem \ref{theorem: type III}}\label{subsection: thm III proof}

\subsubsection{Approximation Results for FFNs}

Since the setting in this subsection includes FFNs, we introduce the following preliminary results about the approximation of FFNs.

% \paragraph{Barron space of two-layer FFNs.} 

The well-known universal approximation result for two-layer FNNs asserts that two-layer FNNs can approximate any continuous function~\citep{barron1992neural,barron1993universal,barron1994approximation}. 
Nonetheless, this result lacks a characterization of the approximation efficiency, i.e., how many neurons are needed to achieve a certain approximation accuracy? 
Extensive pre-existing studies aimed to address this gap by establishing approximation rates for two-layer FFNs. 
A representative result is the Barron theory~\citep{ma2018priori,weinan2021barron,ma2020towards}: any function $f$ in Barron space $\cB$ can be approximated by a two-layer FFN with $M$ hidden neurons can approximate $f$ efficiently, at a rate of $\cO(\Vert f\Vert_{\mathcal{B}}/\sqrt{M})$.
This rate is remarkably independent of the input dimension, thus avoiding the Curse of Dimensionality. Specifically, Barron space is defined in as follows:

\begin{definition}[Barron space~\citep{ma2018priori,weinan2021barron,ma2020towards}]
Consider functions $f:X\to\bbR$ that admit the following representation:
$f(\bx)=\int_{\Omega}a\sigma(\bb^\top\bx+c)\rho(\rd a,\rd \bb,\rd c),\ \bx\in X$.
For any $p\in[1,+\infty]$, we define the Barron norm as $\norm{f}_{\cB_p}:=\inf_{\rho}\Big    (\bbE_{\rho}\left[|a|^p(\norm{\bb}_1+|c|)^p\right]\Big)^{1/p}$.
Then the Barron space are defined as:
$\cB_p:=\{f\in\cC:\norm{f}_{\cB_p}<+\infty\}$.
\end{definition}

\begin{proposition}[\cite{ma2018priori}]
For any $p\in[1,+\infty]$, $\cB_p=\cB_{\infty}$ and $\norm{f}_{\cB_p}=\norm{f}_{\cB_\infty}$.
\end{proposition}

\begin{remark}
    From the Proposition above, the Barron spaces $\cB_p$ are equivalent for any $p\in[1,+\infty]$. Consequently, in this paper, we use $\cB$ and $\norm{\cdot}_{\cB}$ to denote the Barron space and Barron norm.
\end{remark}

The next lemma illustrates the approximation rate of two-layer FFNs for Barron functions.

\begin{lemma}[\cite{ma2020towards}]\label{lemma barron}
    For any $f\in\cB$, there exists a two-layer ReLU neural network $\FFN(\bx)=\sum\limits_{k=1}^M a_w\sigma(\bb_k^\top\bx+c_k)$ with $M$ neurons such that
    \begin{align*}
        \norm{f-\FFN}_{L^\infty[-2,2]}\leq \cO\bracket{\frac{\norm{f}_{\cB}\sqrt{\log M}}{\sqrt{M}}}.
    \end{align*}
\end{lemma}

\subsubsection{Proper Orthogonal Decomposition}

Proper orthogonal decomposition (POD) can be viewed as an extension of the matrix singular value decomposition (SVD) applied to functions of two variables.
Specifically, for a square integrable function $g:\cI\times\cI\rightarrow\mathbb{R}$, it has the following decomposition~(Theorem 3.4 in ~\cite{yarvin1998generalized}, Theorem VI.17 in~\cite{reed1980methods}):
\begin{equation}\label{equ: POD}
	g(\bu,\bv)=\sum_{k=1}^{\infty}\sigma_{k}\phi_{k}(\bu)\psi_{k}(\bv).
\end{equation}
Here, $\phi_{k},\psi_k$ are orthonormal bases for $L^2(\cI)$, and $\sigma_{k}\geq 0$ are the singular values, arranged in descending order. 

Recently,~\cite{jiang2023approximation} also used POD to study the approximation rate of single-layer single-head Transformer for the targets with nonlinear temporal kernels.

Given that two-layer FFNs can efficiently approximate Barron functions~\citep{ma2020towards}, which is dense in $L^2([0,1]^d)$~\citep{siegel2020approximation}, we introduce the following technical definition regarding the well-behavior POD, which is used for our theoretical analysis.

% square integrable function
\begin{definition}[Well-behaved POD]\label{definition: POD}
    Let the POD of $g:[-2,2]^{D}\times[-2,2]^D\mapsto\bbR$ be $g(\bu,\bv)=\sum_{k=1}^{\infty}\sigma_{k}\phi_{k}(\bu)\psi_{k}(\bv)$.
    We call the function $g$ has $\alpha$-well-behaved POD ($\alpha>0$) if: 
    \begin{itemize}[leftmargin=2em]
        \item The decay rate of singular values satisfies $\sigma_k=\cO(1/k^{1+\alpha})$;
        \item The $L^{\infty}$ norms, Barron norms, and Lipschitz norms of the POD bases are all uniformly bounded:
        $\sup_k\left(\norm{\phi_k}_{L^\infty}\vee\norm{\psi_k}_{L^\infty}\vee\norm{\phi_k}_\cB\vee\norm{\psi_k}_\cB\vee\norm{\phi_k}_{\Lip}\vee\norm{\psi_k}_{\Lip}\right)<\infty$.
        % \item The POD bases have a uniform Lipschitz constant: $\sup_k\Vert\phi_k\Vert_{\rm Lip}\vee\sup_k\Vert\psi_k\Vert_{\rm Lip}<\infty$.
    \end{itemize}
\end{definition}

\subsubsection{Proof of Theorem~\ref{theorem: type III}}
\begin{equation}\label{equ: induction head, type III, restate}
    \GnIH(\bX_L)=(\bx_s)_{s=n}^{L-1}\ \sm\Big(\big(g\big(X_{L-n+2:L};X_{s-n+1:s-1}\big)\big)_{s=n}^{L-1}\Big)^\top,
\end{equation}

\begin{theorem}[Restatement of Theorem \ref{theorem: type III}]\label{theorem: type III, restate}
Suppose the similarity function $g$ is $\alpha$-well-behaved (see Definition~\ref{definition: POD}).
Then there exist two absolute constants $A,B>0$ (only depending on the properties of $g$) such that: for any $H,M\in\bbN^{+}$ and rate $q\in\bbN^+$, there exists a constant $C_{n,q}>0$ and a two-layer $H$-head transformer $\TF(\cdot)$ with FFNs of width $M$, such that 
% \begin{align*}
%     \sup_{t}\norm{\GnIH-\TF}_{t,2}\leq\cO\left(\frac{C_1(g;q)n^2}{H^q}\right)+\tilde{\mathcal{O}}\left(\frac{C_2(g)K}{\sqrt{M}}\right)+\cO\left(\frac{1}{K^\alpha}\right).
% \end{align*}
\begin{align*}
    \triplebar{\GnIH-\TF}_{L,2}\leq A\left(\frac{C_{n,q}}{H}\right)^q+ B\frac{L^{1/(1+2\alpha)}}{M^{\alpha/(1+3\alpha)}},
\end{align*}
where $C_{n,q}=\cO(n q^2)$.
\end{theorem}

\begin{proof}

We consider two-layer multi-head transformer with FFN, where the first layer has the residual block.

First, we set an constant $K\in\bbN^+$, and we will optimize it finally.
We choose the embedding dimension $D = nd+2(n-1)K$.

Additionally, in this proof, the following projection matrices are used:
$$P_i:=\left(0_{d\times (i-1)d},I_{d\times d},0_{d\times(D-id)}\right)\in\mathbb{R}^{d\times D}, \quad i\in[n].$$

The proof sketch can be summarized as follows:
\begin{align*}
	(\bx_s)_{s=n}^{L-1}\ \sm\Big(\big(g\big(X_{L-n+2:L}&;X_{s-n+1:s-1}\big)\big)_{s=n}^{L-1}\Big)^\top
	\\
	\text{Step III. 2-st Attn}\ &\uparrow
	\\ \Big(\bx_L^\top,\ldots,\hat{\bx}_{L-n+1},\hat{\phi}_1(\hat{X}_{L-n+2:L}),\ldots,&\hat{\phi}_K(\hat{X}_{L-n+2:L}),\hat{\psi}_1(\hat{X}_{L-n+1:L-1}),\ldots,\hat{\psi}_K(\hat{X}_{L-n+1:L-1})\Big)^\top
	\\
	\text{Step II. 1-st FFN}\ &\uparrow
	\\(\bx_L^\top,\hat{\bx}_{L-1},\ldots,&\hat{\bx}_{L-n+1},{0}^\top)^\top
	\\
	\text{Step I. 1-st Attn}\ &\uparrow
	\\(\bx_L^\top,&{0}^\top)^\top
\end{align*}

Recalling Definition~\ref{definition: POD}, there exists constants $C_{g}^{\infty}, C_{g}^{\cB}, C_{g}^{{\rm Lip}}>0$ such that:
\begin{align*} 
\sup_{k}\left(\norm{\phi_k}_{\infty}\vee\norm{\psi_k}_\infty\right)\leq C_{g}^{\infty},\ 
\sup_{k}\left(\norm{\phi_k}_\cB\vee\norm{\psi_k}_\cB\right)\leq C_{g}^{\cB},\ 
\sup_{k}\left(\norm{\phi_k}_{\Lip}\vee\norm{\psi_k}_\Lip\right)\leq C_{g}^{\Lip}.
\end{align*}

Additionally, $\sigma_{k}=\cO(1/k^{1+\alpha})$ implies that there exits a $C_{\alpha}>0$ such that:
\begin{align*}
    \sum_{k=K}^{\infty}\sigma_k<\frac{C_\alpha}{K^\alpha},\quad \forall K\geq1.
\end{align*}

\underline{\bf Step I: Error in 1-st Attn layer.}  This step is essentially the same as Step I in the proof of Theorem \ref{theorem: type II}, so we write down the estimate of the first Attn layer directly: there exists absolute constants $C_1,C_2>0$ such that
\begin{align*}
    \varepsilon_{\SA}^{(1)}:=\sup_s\norm{x_s^{(1)}-\bX_{s-n+1:s}}_{\infty}\leq C_2\left(\frac{ C_1 n q^2  }{H}\right)^q.
\end{align*}

Similar to the proof of Theorem \ref{theorem: type II}, we first focus on \underline{\bf the case of large $H$:}
$$H\geq C_1 C_2 enq^2,$$ 
which ensures $\varepsilon_{\SA}^{(1)}\leq 1$. Therefore, $\bx_s^{(1)}\in [-2,2]^{D}$:

For simplicity, we use the following notations:
\begin{gather*}
\hat{X}_{s-n+1:s-1}:=W_{Q}=\begin{pmatrix}
P_2{\bx}_{s}^{(1)}\\\vdots\\ P_{n} {\bx}_{s}^{(1)}
\end{pmatrix},\ \text{ for } n\leq s\leq L-1;\quad\quad
\hat{X}_{L-n+2:L}:=\begin{pmatrix}
P_1{\bx}_{L}^{(1)}\\\vdots\\ P_{n-1} {\bx}_{L}^{(1)}.
\end{pmatrix}.
\end{gather*}

\underline{\bf Step II: Error in 1-st FFN layer.} 
The $1$-st FFN is used to approximate $\phi_k,\psi_k\ (k=1,\ldots,K)$. Each function is approximated by a $2$-layer neural networks with $\frac{M}{2K}$ neurons defined on $\mathbb{R}^D$, and the FFNs are concatenated together ( refer to section 7.1 "Parallelization" in ~\cite{schmidt2020nonparametric} ) as $\text{FFN}^{(1)}$. We denote them as 
$$\hat{\phi}_k(\by)=\sum_{m=1}^\frac{M}{2K}
a_m^k\sigma(\bb_m^{k^\top}\by+c_m^k)$$
$$\hat{\psi}_k(\by)=\sum_{m=1}^\frac{M}{2K}\tilde{a}_m^k\sigma(\tilde{\bb}_m^{k^\top}\by+\tilde{c}_m^k)$$

Then by lemma \ref{lemma barron}, such FFNs exist and satisfy the following properties hold for all $1\leq k\leq K$:
$$\Vert \hat{\phi}_k-\phi_{k} \Vert_{L^{\infty}([-2,2]^D)}\leq{\mathcal{O}}\left(\Vert \phi_k\Vert_{\mathcal{B}}\sqrt{\frac{K\log M}{{M}}}\right)\leq \epsilon_\FFN^{\rm (1)},$$
$$\Vert \hat{\psi}_k-\psi_{k} \Vert_{L^{\infty}([-2,2]^D)}\leq{\mathcal{O}}\left(\Vert \psi_k\Vert_{\mathcal{B}}\sqrt{\frac{K\log M}{M}}\right)\leq \epsilon_\FFN^{\rm (1)},$$
where
$$
\epsilon_{\FFN}^{(1)}:=\mathcal{O}\left(C_g^\cB\sqrt{\frac{K\log M}{{M}}}\right),\quad
C_g^\cB=\max\{\norm{\phi_k}_\cB,\norm{\psi_k}_\cB\}.
$$

\underline{\bf Step III: Error in 2nd Attn layer.} 

We use matrices in the second layer to take out elements needed 
$$\bW_V^{(2)}=({I}_{d\times d},{0}_{d\times D})\in\mathbb{R}^{d\times D},$$
$$\bW_K^{(2,1)}=\sum_{i=k}^K \sqrt{\sigma_k}{e}_{k,(n-1)d+k}\in\mathbb{R}^{D\times D},$$
$$\bW_Q^{(2,1)}=\sum_{k=1}^K \sqrt{\sigma_k}{e}_{k,(n-1)d+K+k}\in\mathbb{R}^{D\times D}.$$
We denote the rank-$K$ truncation of $g$ as $$g_K:=\sum_{k=1}^K\sigma_k\phi_k\psi_k,$$
and its approximation as
$$\hat{g}_K:=\sum_{k=1}^K\sigma_k\hat{\phi}_k\hat{\psi}_k$$
The second FFN is set to be identity map and we denote the final output as
$$\bx_L^{(2)}:=(\bx_s)_{s=n}^{L-1}\ \sm\Big(\big(\hat{g}_K\big(\hat{X}_{L-n+2:L};
\hat{X}_{s-n+1:s-1}\big)\big)_{s=n}^{L-1}\Big)^\top,
$$

First, we consider the error under the first norm, $\norm{\cdot}_{\infty}$, which can be divided the total error into three components:
\begin{equation}\label{proof: equ: thm III: error analysis}
\begin{aligned}
    &\norm{\bx_L^{(2)}-(\bx_s)_{s=n}^{L-1}\ \sm\Big(\big(g\big(X_{L-n+2:L};X_{s-n+1:s-1}\big)\big)_{s=n}^{L-1}\Big)^\top}_\infty
    \\
    \leq&\norm{\sm\Big(\big(\hat{g}_K\big(\hat{X}_{L-n+2:L};\hat{X}_{s-n+1:s-1}\big)\big)_{s=n}^{L-1}\Big)^\top-\sm\Big(\big(g\big(X_{L-n+2:L};X_{s-n+1:s-1}\big)\big)_{s=n}^{L-1}\Big)^\top}_\infty
    \\\leq&
    \norm{\sm\Big(\big(\hat{g}_K\big(\hat{X}_{L-n+2:L};\hat{X}_{s-n+1:s-1}\big)\big)_{s=n}^{L-1}\Big)^\top-\sm\Big(\big(g_K\big(\hat{X}_{L-n+2:L};\hat{X}_{s-n+1:s-1}\big)\big)_{s=n}^{L-1}\Big)^\top}_\infty
    \\&+
    \norm{\sm\Big(\big(g_K\big(\hat{X}_{L-n+2:L};\hat{X}_{s-n+1:s-1}\big)\big)_{s=n}^{L-1}\Big)^\top-\sm\Big(\big(g_K\big(X_{L-n+2:L};X_{s-n+1:s-1}\big)\big)_{s=n}^{L-1}\Big)^\top}_\infty
    \\&+
    \norm{\sm\Big(\big(g_K\big(X_{L-n+2:L};X_{s-n+1:s-1}\big)\big)_{s=n}^{L-1}\Big)^\top-\sm\Big(\big(g\big(X_{L-n+2:L};X_{s-n+1:s-1}\big)\big)_{s=n}^{L-1}\Big)^\top}_\infty
    % \\\leq&\sum_{s=n}^{t-1}\left|\sm\Big(\hat{g}_K\big(\hat{X}_{t-n+2:t}, \hat{X}_{s-n+1:s-1}\big)\Big)-\sm\Big(g\big(X_{t-n+2:t}, X_{s-n+1:s-1}\big)\Big)\right|
    \\\leq&
    \max_s\left|\hat{g}_K\big(\hat{X}_{L-n+2:L}, \hat{X}_{s-n+1:s-1}\big)-{g}_K\big(\hat{X}_{L-n+2:L}, \hat{X}_{s-n+1:s-1}\big)\right|
    \\&+\max_s\left|{g}_K\big(\hat{X}_{L-n+2:L}, \hat{X}_{s-n+1:s-1}\big)-{g}_K\big(X_{L-n+2:L}, X_{s-n+1:s-1}\big)\right|
    \\&+\sum_{s=n}^{L-1}\left|{g}_K\big({X}_{L-n+2:L}, {X}_{s-n+1:s-1}\big)-g\big(X_{L-n+2:L}, X_{s-n+1:s-1}\big)\right|
\end{aligned}
\end{equation}
% For the first term, we introduce an inequality
% \begin{align*}
% 	|ab-\hat{a}\hat{b}|&\leq|a||b-\hat{b}|+|\hat{b}||a-\hat{a}|\\
% 	&\leq|a||b-\hat{b}|+|b||a-\hat{a}|+|a-\hat{a}||b-\hat{b}|,
% \end{align*}

For the first term in RHS of~\eqref{proof: equ: thm III: error analysis}, it holds that:
\begin{align*}
&\max_s\left|\hat{g}_K\big(\hat{X}_{L-n+2:L}, \hat{X}_{s-n+1:s-1}\big)-{g}_K\big(\hat{X}_{L-n+2:L}, \hat{X}_{s-n+1:s-1}\big)\right|
\\\leq&
\max_s\sum_{k=1}^K\sigma_k\left|\hat{\phi}_k(\hat{X}_{L-n+2:L})\hat{\psi}_k(\hat{X}_{s-n+1:s-1})-\phi_k(\hat{X}_{L-n+2:L})\psi_k(\hat{X}_{s-n+1:s-1})\right|
\\\leq&
\sum_{k=1}^K \sigma_k\left(\norm{\hat{\phi}_k}_{L^{\infty}}\norm{\hat{\psi}_k-\psi_k}_{L^{\infty}}+\norm{{\psi}_k}_{L^{\infty}}\norm{\hat{\phi}_k-\phi_k}_{L^{\infty}}\right)
\\\leq&
\epsilon_{\FFN}^{(1)}\cdot\sum_{k=1}^K \sigma_k\left(\norm{\hat{\phi}_k}_{L^{\infty}}+\norm{{\psi}_k}_{L^{\infty}}\right)
\\\leq&\epsilon_{\FFN}^{(1)}\cdot\sum_{k=1}^K \sigma_k\left(\norm{{\phi}_k}_{L^{\infty}}+\norm{\hat{\phi}_k-\phi_k}_{L^{\infty}}+\norm{{\psi}_k}_{L^{\infty}}\right)
\\\leq&
\epsilon_{\FFN}^{(1)}\cdot(2C_{g}^{\infty}+1)\sum_{k=1}^K\sigma_k\leq (2C_{g}^{\infty}+1)C_\alpha \epsilon_{\FFN}^{(1)}.
\end{align*}

For the second term in RHS of~\eqref{proof: equ: thm III: error analysis}, we have:
\begin{align*}
&\max_s\left|{g}_K\big(\hat{X}_{L-n+2:L}, \hat{X}_{s-n+1:s-1}\big)-{g}_K\big(X_{L-n+2:L}, X_{s-n+1:s-1}\big)\right|
\\\leq&\max_s\sum_{k=1}^{K}\sigma_k \Big(\|{\phi}_k\|_{L^{\infty}}|{\psi}_k(\hat{X}_{s-n+1:s-1})-\hat{\psi}_k(X_{s-n+1:s-1})|
\\&\quad\quad\quad\quad\quad+\norm{\psi_k}_{L^{\infty}}|{\phi}_k(\hat{X}_{L-n+1:L-1})-{\phi}_k(X_{L-n+1:L-1})|\Big)
\\\leq&\max_s\sum_{k=1}^{K}\sigma_k \Big(\|{\phi}_k\|_{L^{\infty}}\norm{{\psi}_k}_{\Lip}\norm{\hat{X}_{s-n+1:s-1}-X_{s-n+1:s-1}}
\\&\quad\quad\quad\quad\quad+\norm{\psi_k}_{L^{\infty}}\norm{{\phi}_k}_{\Lip}\norm{\hat{X}_{L-n+1:L-1}-X_{L-n+1:L-1}}\Big)
\\\leq&2C_g^{\infty}C_g^{\Lip}\epsilon_{\SA}^{(1)}\cdot\left(\max_s\sum_{k=1}^K\sigma_k\right)\leq
2C_g^{\infty}C_g^{\Lip}C_\alpha\epsilon_{\SA}^{(1)}.
\end{align*}

Additioanlly, the third term in RHS of~\eqref{proof: equ: thm III: error analysis}, its $L^2$ holds that:
\begin{align*}
&\int_{[0,1]^{d\times L}}\left(\sum_{s=n}^{L-1}\left|{g}_K\big({X}_{L-n+2:L}, {X}_{s-n+1:s-1}\big)-g\big(X_{L-n+2:L}, X_{s-n+1:s-1}\big)\right|\right)^2 d X
\\\leq&
(t-1-n)\sum_{s=n}^{L-1}\int_{[0,1]^{D\times L}}\left|{g}_K\big({X}_{-n+2:t}, {X}_{s-n+1:s-1}\big)-g\big(X_{L-n+2:L}, X_{s-n+1:s-1}\big)\right|^2 d X
\\=&
(L-1-n)^2\int_{[0,1]^D\times[0,1]^D}\left|g(u,v)-g_K(u;v)\right|^2\rd u \rd v
\\=&
(L-1-n)^2\int\left(\sum_{k=K+1}^{+\infty}\sigma_k\phi_k(\bu)\psi_k(\bv)\right)^2\ du\ dv
\\\leq&\int\left(\sum_{k=K+1}^{+\infty}\sigma_k\phi_k^2(\bu)\right)\left(\sum_{k=K+1}^{+\infty}\sigma_k\psi_k^2(\bv)\right)\ du\ dv
 \\\leq& 
 (L-1-n)^2\left(\sum_{k=K+1}^{\infty}\sigma_k\right)^2\leq\frac{(L-1-n)^2 C_\alpha^2}{K^{2\alpha}}.
\end{align*}

Now we combine three error terms together to obtain the total $L^2$ error for the output of this layer:

\begin{align*}
    &\int_{X\in[0,1]^{d\times L}}\norm{\bx_L^{(2)}-(\bx_s)_{s=n}^{L-1}\ \sm\Big(\big(g\big(X_{L-n+2:L};X_{s-n+1:s-1}\big)\big)_{s=n}^{L-1}\Big)^\top}_\infty^2\rd X
    \\\leq& 3\int_{X\in[0,1]^{d\times L}}\max_s\left|\hat{g}_K\big(\hat{X}_{L-n+2:L}, \hat{X}_{s-n+1:s-1}\big)-{g}_K\big(\hat{X}_{L-n+2:L}, \hat{X}_{s-n+1:s-1}\big)\right|^2\rd X
    \\& + 3\int_{X\in[0,1]^{d\times L}}\max_s\left|{g}_K\big(\hat{X}_{L-n+2:L}, \hat{X}_{s-n+1:s-1}\big)-{g}_K\big(X_{L-n+2:L}, X_{s-n+1:s-1}\big)\right|^2\rd X
    \\& + 3
    \int_{X\in[0,1]^{d\times L}}\left(\sum_{s=n}^{L-1}\left|{g}_K\big({X}_{L-n+2:L}, {X}_{s-n+1:s-1}\big)-g\big(X_{L-n+2:L}, X_{s-n+1:s-1}\big)\right|\right)^2\rd X
    \\\leq&
    3\max_s\left|\hat{g}_K\big(\hat{X}_{L-n+2:L}, \hat{X}_{s-n+1:s-1}\big)-{g}_K\big(\hat{X}_{L-n+2:L}, \hat{X}_{s-n+1:s-1}\big)\right|^2
    \\&+3\max_s\left|{g}_K\big(\hat{X}_{L-n+2:L}, \hat{X}_{s-n+1:s-1}\big)-{g}_K\big(X_{L-n+2:L}, X_{s-n+1:s-1}\big)\right|^2
    \\&+3\left(\frac{(L-1-n) C_\alpha}{K^{\alpha}}\right)^2
    \\\leq& 3\left((2C_{g}^{\infty}+1)C_\alpha \epsilon_{\FFN}^{(1)}\right)^2
    +3\left(2C_g^{\infty}C_g^{\Lip}C_\alpha\epsilon_{\SA}^{(1)}\right)^2
    +3\left(\frac{(L-1-n) C_\alpha}{K^{\alpha}}\right)^2
    \\\leq& 3\left((2C_{g}^{\infty}+1)C_\alpha \epsilon_{\FFN}^{(1)}+2C_g^{\infty}C_g^{\Lip}C_\alpha\epsilon_{\SA}^{(1)}+\frac{(L-1-n) C_\alpha}{K^{\alpha}}\right)^2.
\end{align*}

This estimate implies that

\begin{equation}\label{proof: equ: III: K}
\begin{aligned}
&\triplebar{\GnIH-\TF}_{L,2}
\\\leq&\sqrt{3}\left(2C_g^{\infty}C_g^{\Lip}C_\alpha\epsilon_{\SA}^{(1)}+(2C_{g}^{\infty}+1)C_\alpha \epsilon_{\FFN}^{(1)}+\frac{(L-1-n) C_\alpha}{K^{\alpha}}\right)
\\\leq&\cO\left(A_{g,\alpha}C_2\left(\frac{ C_{1} n q^2}{H}\right)^q\right)+\cO\left(\frac{B_{g,\alpha}\sqrt{K\log M}}{\sqrt{M}}\right)+\cO\left(\frac{L C_\alpha}{K^\alpha}\right),
\end{aligned}
\end{equation}

where
\begin{align*}
    A_{g,\alpha}:=C_g^{\infty}C_g^{\Lip}C_\alpha,\quad B_{g,\alpha}:=(2C_{g}^{\infty}+1)C_\alpha.
\end{align*}

\underline{\bf Step IV. Optimizing $K$ in~\eqref{proof: equ: III: K}.}

Notice that in RHS of~\eqref{proof: equ: III: K}, only ${\cO}\left(\frac{C_{g,\alpha}\sqrt{K\log M}}{\sqrt{M}}\right)$ and $\cO\left(\frac{L C_\alpha}{K^\alpha}\right)$ depend on $K$.

By Young's inequality, with $p=\frac{\alpha+\frac{1}{2}}{\alpha}$ and $q=2(\alpha+\frac{1}{2})$, we have:

\begin{align*}
    &\min_K:\frac{\alpha}{\frac{1}{2}+\alpha}\frac{B_{g,\alpha}\sqrt{K\log M}}{\sqrt{M}}+\frac{\frac{1}{2}}{\frac{1}{2}+\alpha}\frac{L C_\alpha}{K^\alpha}
    \\=&\min_K:\frac{\alpha}{\frac{1}{2}+\alpha}\left(\left(\frac{B_{g,\alpha}\sqrt{K\log M}}{\sqrt{M}}\right)^{\frac{\alpha}{\frac{1}{2}+\alpha}}\right)^{\frac{\frac{1}{2}+\alpha}{\alpha}}+\frac{\frac{1}{2}}{\frac{1}{2}+\alpha}\left(\left(\frac{L C_\alpha}{K^\alpha}\right)^\frac{\frac{1}{2}}{\frac{1}{2}+\alpha}\right)^{2(\frac{1}{2}+\alpha)}
    \\=&\frac{B_{g,\alpha}' L^{1/(1+2\alpha)}}{(M/\log M)^{\alpha/(1+2\alpha)}},
\end{align*}
where $B_{g,\alpha}'$ only depends on the properties of $g$ and $\alpha$.

Thus, we obtain our final bound:
\begin{align*}
    &\triplebar{\GnIH-\TF}_{L,2}
    \leq\cO\left(A_{g,q}C_2\left(\frac{C_1 n q^2}{H}\right)^q\right)+\left\{{\cO}\left(\frac{B_{g,\alpha}\sqrt{K\log M}}{\sqrt{M}}\right)+\cO\left(\frac{L C_\alpha}{K^\alpha}\right)\right\}_{\min:K}
    \\\leq&
    \cO\left(A_{g,q}\left(\frac{C_1 n q^2}{H}\right)^q\right)+{\cO}\left(\frac{B_{g,\alpha}' L^{1/(1+2\alpha)}}{(M/\log M)^{\alpha/(1+2\alpha)}}\right)
    \\\leq& A_{g,q}'\left(\frac{C_1 n q^2}{H}\right)^q+B_{g,\alpha}''\frac{ L^{1/(1+2\alpha)}}{M^{\alpha/(1+3\alpha)}},
\end{align*}
where $A_{g,\alpha}',B_{g,\alpha}''$ only depends on the properties of $g$ and $\alpha$.

For \underline{\bf the remained case}, $H< C_1C_2 nq^2$, similar to the proof of Theorem~\ref{theorem: type II} we can simply choose $\TF$ with all $0$ parameters. Then the approximation error can be trivially bounded by: 
$$\triplebar{\GnIH-\TF}_{L,2}= \triplebar{\GnIH-0}_{L,2}\leq 1.$$

Then, the two cases can be unified by: 

\begin{align*}
    \triplebar{\GnIH-\TF}_{L,2}\leq&\begin{cases}
        A_{g,\alpha}'\left(\frac{C_1 n q^2}{H}\right)^q+B_{g,\alpha}''\frac{ L^{1/(1+2\alpha)}}{M^{\alpha/(1+3\alpha)}},\quad H\geq C_1C_2 nq^2 
        \\ 1,\quad \text{otherwise}
    \end{cases}
        \\\leq& \max\{A_{g,\alpha}',1\}\left(\frac{C_1\max\{C_2,1\} n q^2}{H}\right)^q+ B_{g,\alpha}'' \frac{ L^{1/(1+2\alpha)}}{M^{\alpha/(1+3\alpha)}}.
\end{align*}

Finally, we can choose $A_{g,\alpha}''=\max\{A_{g,\alpha}',1\}$ and $C_{n,q}=C_1\max\{C_2,1\} n q^2$. Then 
\begin{align*}
    \triplebar{\GnIH-\TF}_{L,2}\leq A_{g,\alpha}''\left(\frac{C_{n,q}}{H}\right)^q+ B_{g,\alpha}'' \frac{ L^{1/(1+2\alpha)}}{M^{\alpha/(1+3\alpha)}},
\end{align*}

where $A_{g,\alpha}''$ and $B_{g,\alpha}''$ only depend on $g$ and $\alpha$, thus only depending on the properties of $g$. Moreover, $C_{n,q}=\cO(nq^2)$.
\end{proof}

\vspace{1.cm}

\section{Proofs in Section~\ref{section: optimization}}

\subsection{Optimization Dynamics in Training Stage I}\label{appendix: dynamics: stage I}

% This section is dedicated to presenting {\red graph} and proving the {\red theorem} ...

% {\bf Technical Simplification.}

% \subsection{Training Stage I: First Layer}\label{appendix: first layer}

In this subsection we focus on training the first layer of Transformer model to capture the token ahead. For simplicity, we introduce some notations:
\begin{align*}
    \tilde{p}:=p^{(1,1)},\quad p:=p^{(2,1)},\quad g:=w_{V}^{(2,1)},\quad h:=w_{V}^{(2,2)},\quad w_K:=w_{K}^{(2,2)},\quad w_Q:=w_Q^{(2,2)},
\end{align*}
and denote the initialization of each parameter as 
$\tilde{p}(0), p(0),g(0),w_Q(0),w_K(0),h(0)$
respectively.

We initialize $p(0),w_k(0),w_Q(0)=0$ while the other parameters are all initialized at $\sigma_{\init}$. 
In this training stage, we only train $\tilde{p}$. And our goal, {\bf the proof of Lemma \ref{lemma: layer I}} can be deduced from which, is to prove:
$$\lim_{t\to+\infty} \tilde{p}(t)=+\infty.$$

% \todo[inline]{Read through and organize the proof to make it easier to understand}

% \todo[inline]{Use unified notations: $p(0)$ rather than $p(0)$, use $T_1^g$, $T_1^h$, $T_2^h$, use $\star$, , etc...}
% \subsection{Setting and Definition}

% \todo[inline]{wmz: 
% 1. change to the original target: with standard softmax)
% 2. check the initialization, and be more rigorously.
% }

% \todo[inline]{wmz: introducte the notations: $\tilde{p},p,\cdots$}

In this stage, the $s$-th output token of the first layer is represented as
\begin{align*}
    \begin{pmatrix}
        x_s
        \\(x_\tau)_{\tau=1}^{s-1}\ \sm\left(\big(-\tilde{p}(s-1-\tau)\big)_{\tau=1}^{s-1}\right)^\top
    \end{pmatrix},
\end{align*}

and {the target function} and {output of transformer} are as follows
\begin{align*}
    f^*(\bX) = \begin{pmatrix}
        \frac{\alpha^\star}{1+\alpha^\star}x_{L-2} 
        \\ \frac{1}{1+\alpha^\star} (x_s)_{s=2}^{L-1}\ \sm\Big(\big(x_L {w^\star}^2 x_{s-1}\big)_{s=2}^{L-1}\Big)^\top
    \end{pmatrix},
\end{align*}
{\small
\begin{align*}
    f_{\btheta}(\bX) &= \begin{pmatrix}
    g(0) \Bigg((x_\tau)_{\tau=1}^{s-1}\sm\left((-\tilde{p}(s-1-\tau))_{\tau=1}^{s-1}\right)^\top\Bigg)_{s=2}^{L-1}\sm\Big((-p(0)(L-1-s))_{s=2}^{L-1}\Big)^\top
    \\h(0)(x_s)_{s=2}^{L-2} \sm\Bigg(\Big(w_K(0)w_Q(0)x_L\cdot (x_\tau)_{\tau=1}^{s-1}\ \sm\left((-\tilde{p}(s-1-\tau))_{\tau=1}^{s-1}\right)^\top\Big)_{s=2}^{L-2}\Bigg)^\top
     \end{pmatrix}
     \\&=\begin{pmatrix}
         g(0)\frac{1}{L-2}\sum_{\tau=1}^{L-2}\left(\sum_{s=\tau+1}^{L-1}\sm\Big((-\tilde{p}(s-1-t))_{t=1}^{s-1}\Big)\right)_{t=\tau} x_{\tau}
         \\h(0)\frac{1}{L-2}\sum_{s=2}^{L-2}x_s
     \end{pmatrix}.
\end{align*}}

% \begin{align*}
%     f_{\btheta}(\bX) &= \begin{pmatrix}
%         g(0)\sum_{s=2}^{L-1}\sm(-p(0)(L-1-s))\cdot\left(\sum_{\tau=1}^{s-1}\sm_s(-\tilde{p}(s-1-\tau))x_{\tau}\right)
%         \\ h(0)\sum_{s=2}^{L-2}\sm\bigg( w_K(0)w_Q(0) x_L\left(\sum_{\tau=1}^{s-1}\sm_s(-\tilde{p}(s-1-\tau))x_{\tau}\right)\bigg)x_s
%     \end{pmatrix}
%     \\&=\begin{pmatrix}
%         g(0)\frac{1}{L-2}\sum_{\tau=1}^{L-2}\left(\sum_{s=\tau+1}^{L-1}\sm_s(-\tilde{p}(s-1-\tau))\right)x_\tau
%         \\h(0)\frac{1}{L-2}\sum_{s=2}^{L-2}x_s
%     \end{pmatrix}.
% \end{align*}

% where $\sm:=\sm_{L}$ is for the second layer.

Since we only focus on $\tilde{p}$ and the other parameters remain the initialization value, the loss function can be simplified as 
% \begin{align*}
%     \mathcal{L}(\btheta)&=\mathop{\bbE}\limits_{\bX\sim\bbN(0,1)^L}\Bigg[   \frac{{\alpha^\star}^2}{(1+\alpha^\star)^2}x_{L-2}^2 +\frac{g(0)^2}{(L-2)^2}\sum_{\tau=1}^{L-2}\left(\sum_{s=\tau+1}^{L-1}\sm_s(-\tilde{p}(s-1-\tau))\right)^2x_\tau^2
%     \\&\quad\quad\quad\quad\quad\quad+\frac{2g(0)}{L-2}\frac{\alpha^\star}{1+\alpha^\star}\sm_{L-1}(0)x_{L-2}^2    \Bigg]+ C(w^\star,\alpha^\star,w(0),h(0))
% \end{align*}
\begin{align*}
    \mathcal{L}(\btheta)&=\mathop{\bbE}\limits_{\bX\sim\bbN(0,1)^L}\Bigg[   \frac{{\alpha^\star}^2}{(1+\alpha^\star)^2}x_{L-2}^2 +\frac{g(0)^2}{(L-2)^2}\sum_{\tau=1}^{L-2}\left(\sum_{s=\tau+1}^{L-1}\sm\Big((-\tilde{p}(s-1-t))_{t=1}^{s-1}\Big)_{t=\tau}\right)^2x_\tau^2
    \\&\quad\quad\quad\quad\quad\quad+\frac{2g(0)}{L-2}\frac{\alpha^\star}{1+\alpha^\star}\sm\Big((-p(0)(L-1-s))_{s=2}^{L-1}\Big)_{s=L-1}x_{L-2}^2    \Bigg]+ C(w^\star,\alpha^\star,w(0),h(0))
\end{align*}

where the second term $C(w^\star,\alpha^\star,w(0),h(0))$ is a constant depends on $w^\star,\alpha^\star,w(0)$ and $h(0)$, produced by calculating the error of the second head, i.e., loss of induction head, while the first term is $4$-gram loss.

% Consequently, we develop several calculation to simplify our computation:

% \todo[inline]{what's the mearning of these quantities?}

We first define several functions that will be useful for calculation in this stage and the second one:

{\em Function I.} This function is purely defined for the calculation of $\frac{\rd p}{\rd t}$. Denoted by 
    $q(\tilde{p}):= \sum_{\tau=1}^{L-2}\left(\sum_{s=\tau+1}^{L-1}\frac{e^{\tilde{p}(s-1-\tau)}}{\sum_{k=0}^{s-2}e^{-\tilde{p}k}}\right)^2$, we first prove $\frac{\rd q}{\rd \tilde{p}}\leq 0$.
% \begin{proof}then
\begin{align*}
    q(\tilde{p}) & := \sum_{\tau=1}^{L-2}\left(\sum_{s=\tau+1}^{L-1}\frac{e^{\tilde{p}(s-1-\tau)}}{\sum_{k=0}^{s-2}e^{-\tilde{p}k}}\right)^2
    \\& = \sum_{\tau=1}^{L-2}\left(\sum_{s=\tau+1}^{L-1}\frac{e^{-\tilde{p}(s-1-\tau)}}{1-e^{-\tilde{p}(s-1)}}(1-e^{-\tilde{p}})\right)^2
    \\& = (1-e^{-\tilde{p}})^2\sum_{\tau=1}^{L-2}\left(\sum_{s=\tau+1}^{L-1}\frac{e^{-\tilde{p}(s-1-\tau)}}{1-e^{-\tilde{p}(s-1)}}\right)^2
    \\& = (1-e^{-\tilde{p}})^2\sum_{\tau=1}^{L-2}e^{2\tilde{p}\tau}\left(\sum_{s=\tau+1}^{L-1}\frac{e^{-\tilde{p}(s-1)}}{1-e^{-\tilde{p}(s-1)}}\right)^2
    \\& = (1-e^{-\tilde{p}})^2\sum_{\tau=1}^{L-2}e^{2\tilde{p}\tau}\left(\sum_{s=\tau+1}^{L-1}\frac{1}{e^{\tilde{p}(s-1)}-1}\right)^2
\end{align*}
Then we take its derivative of $\tilde{p}$
\begin{align*}
    \frac{\rd q}{\rd \tilde{p}} & = 2(1-e^{-\tilde{p}})e^{-\tilde{p}}\sum_{\tau=1}^{L-2}e^{2\tilde{p}\tau}\left(\sum_{s=\tau+1}^{L-1}\frac{1}{e^{\tilde{p}(s-1)}-1}\right)^2
    \\& + (1-e^{-\tilde{p}})^2\sum_{\tau=1}^{L-2}2\tau e^{2\tilde{p}\tau}\left(\sum_{s=\tau+1}^{L-1}\frac{1}{e^{\tilde{p}(s-1)}-1}\right)^2
    \\& + (1-e^{-\tilde{p}})^2\sum_{\tau=1}^{L-2}2e^{2\tilde{p}\tau}\left(\sum_{s=\tau+1}^{L-1}\frac{1}{e^{\tilde{p}(s-1)}-1}\right)\left(\sum_{s=\tau+1}^{L-1}\frac{-(s-1)e^{\tilde{p}(s-1)}}{(e^{\tilde{p}(s-1)}-1)^2}\right)
    \\& = 2(1-e^{-\tilde{p}})\sum_{\tau=1}^{L-2}e^{2\tilde{p}\tau}\left(\sum_{s=\tau+1}^{L-1}\frac{1}{e^{\tilde{p}(s-1)}-1}\right)\left(\sum_{s=\tau+1}^{L-1}\frac{e^{-\tilde{p}}+\tau(1-e^{-\tilde{p}})}{e^{\tilde{p}(s-1)}-1}-\frac{(s-1)e^{\tilde{p}(s-1)}}{(e^{\tilde{p}(s-1)}-1)^2}\right)
\end{align*}

% \todo[inline]{the numerators two terms? $h'(p)$?}

$\frac{\rd q}{\rd \tilde{p}}$'s last factor can be formed as  
\begin{align*}
    & \frac{\left(\tau-(\tau-1)e^{-\tilde{p}}\right)\left(e^{\tilde{p}(s-1)}-1\right)-(s-1)e^{\tilde{p}(s-1)}}{e^{\tilde{p}(s-1)}-1)^2}
    \\ =& \frac{(\tau+1-s)t^{s-1}-(\tau-1)t^{s-2}-\tau+\frac{\tau-1}{t}}{e^{\tilde{p}(s-1)}-1)^2} 
\end{align*}
where $t=e^{-\tilde{p}}\geq 1$. Since $s\geq \tau+1$, $\frac{\rd q}{\rd \tilde{p}}\leq 0$.
% \end{proof}

{\em Function II. } For simplicity, we define $M(p)$ and its derivative $m(p)$:
    $$M(p) := \sum_{s=2}^{L-1}\exp(-p(L-1-s)) = \sum_{s=0}^{L-3}\exp{-ps} = \frac{1-e^{-p(L-2)}}{1-e^{-p}},$$
    $$m(p) := \sum_{s=1}^{l-3}s\exp(-ps) = \frac{e^{-p}-(L-2)e^{-p(L-2)}+(L-3)e^{-p(L-1)}}{(1-e^{-p})^2}.$$

{\em Function III. } The third function is derivative of softmax. By straightfoward calculation, we obtain:
    $$\frac{\rd}{\rd p}\sm\Big((-p(L-1-t))_{t=2}^{L-1}\Big)_{t=L-1-s} = \frac{\rd}{\rd p}\frac{\exp(-ps)}{\sum_{\tau=0}^{L-3}\exp(-p\tau)} = \frac{-s\exp(-ps)M(p)+\exp(-ps)m(p)}{M(p)^2}.$$

% \todo[inline]{wmz: $\sm_{L-1}(0)? \sm(0)?$ what's the mean}
% {\bf Now let's return to our goal:} 

Through the quantities and their properties above, we obtain the dynamic of $\tilde{p}$
\begin{align*}
    \frac{\rd \tilde{p}}{\rd t}&=-\frac{g(0)^2}{(L-2)^2}q'(\tilde{p})+\frac{2\alpha^\star g(0)}{(1+\alpha^\star)(L-2)}\frac{m(p)}{M(p)^2}
    \\&\geq \frac{2\alpha^\star g(0)}{(1+\alpha^\star)(L-2)}e^{-\tilde{p}},
\end{align*}
% Thus $\tilde{p}$ grows like $\ln$.
which implies:
\begin{align*}
    \lim_{t\to+\infty}\tilde{p}(t)=+\infty.
\end{align*}

% \subsection{Training Stage II: the second layer}

\vspace{1.cm}

\subsection{Optimization Dynamics in Training Stage II}\label{appendix: dynamics: stage II}

In this training stage, the first layer is already capable of capturing the token ahead i.e. $y_s=x_{s-1}$. And we train the parameters $w_{V_1},w_{V_2},p,w_{KQ}$ in the second layer.
    
    % \begin{itemize}
    %     \item {\bf The model first learns 4-gram mechanism.} There is a critical point $t_1 = \mathcal{O}\left(p(0) + \left(\frac{\alpha^\star}{1+\alpha^\star}\right)^2\right)$, before which $p$ increases to $\mathcal{O}(1)$, and $g$ grows to a value between $\frac{\alpha^\star}{1+\alpha^\star}$ and $\frac{2\alpha^\star}{1+\alpha^\star}$. After $t_1$, $p$ tends to $+\infty$ logarithmically, while $g$ decreases to $\frac{\alpha^\star}{1+\alpha^\star}$ at a rate of $\mathcal{O}\left(\frac{1}{\sqrt{t}}\right)$.
    %     \item {\bf It then transitions to learning induction head mechanism.} Initially, $h$ increases beyond $\frac{1}{1+\alpha^\star}$, with $w$ barely changing. Then, $h$ increases slowly, dominated by $\frac{1.36}{1+\alpha^\star}$. After at least $\dot{t} = (L-2)\left(1-\frac{1}{2}{w^\star}^4\right)$, $w$ begins to increase exponentially, approaching $w^\star$ arbitrarily closely. After $T_2^h = \frac{(L-2)(1+\alpha^\star)^2}{16{w^\star}^2}\ln\left(\frac{1}{w(0)}\right)$, $h$ decreases steadily to $\frac{1}{1+\alpha^\star}$, while $w$ stably increases to $w^\star$.
    % \end{itemize}

% \begin{remark}
%     At $T_2^h$, $w$ and $h$ essentially reach the global minimizer, while at $t_1$, $p$ and $g$ have mostly learned their global minimizer and begin converging rapidly. The lower bound of $T_2^h$ is $\mathcal{O}(L)$ times larger than the upper bound of $t_1$, illustrating a distinct phase transition.
% \end{remark}
 We start from proving the parameter balance lemma:
 \begin{lemma}[Restate of Lemma \ref{lemma: parameter balance}]\label{lemma: parameter balance restate}
In Training Stage II, it holds that ${w_{Q}^{(2,2)}}^2(t)\equiv {w_{K}^{(2,2)}}^2(t)$.
\end{lemma}
\begin{proof}
Notice that
\begin{align*}
    &\frac{\rd }{2\rd t}\left({w_{Q}^{(2,2)}}^2(t)-{w_{K}^{(2,2)}}^2(t)\right)=-w_{Q}^{(2,2)}\frac{\partial \cL}{\partial w_{Q}^{(2,2)}}+w_{K}^{(2,2)}\frac{\partial \cL}{\partial w_{K}^{(2,2)}}
    \\=&- w_{Q}^{(2,2)}w_{K}^{(2,2)}\frac{\partial\cL}{\partial\left(w_{Q}^{(2,2)}w_{K}^{(2,2)}\right)}+ w_{K}^{(2,2)}w_{Q}^{(2,2)}\frac{\partial\cL}{\partial\left(w_{Q}^{(2,2)}w_{K}^{(2,2)}\right)}\equiv0.
\end{align*}

Thus, we have:
\begin{align*}
    {w_{Q}^{(2,2)}}^2(t)-{w_{K}^{(2,2)}}^2(t)
    \equiv
    {w_{Q}^{(2,2)}}^2(0)-{w_{K}^{(2,2)}}^2(0)=0.
\end{align*}
% From the continuity, and given that $w_{Q}^{(2,2)}(0)=w_{K}^{(2,2)}(0)\sigma_{\init}>0$, $ w_{Q}^{(2,2)}(t)\equiv w_{K}^{(2,2)}(t)$.
\end{proof}

% \subsection{Preliminaries}
For simplicity, we still use the following notations:
\begin{align*}
    p:=p_{1},\quad g:=w_{V_1},\quad w:=w_{KQ},\quad h:=w_{V_2}.
\end{align*}
and notations for initialization $p(0),g(0),w(0),h(0)$. Then the target function and output of Transformer can be formed as follows
\begin{align*}
    f^\star(\bX) &= \begin{pmatrix}
        \frac{{\alpha^\star}}{1+{\alpha^\star}}x_{L-2} 
        \\ \frac{1}{1+{\alpha^\star}}(x_s)_{s=2}^{L-1}\sm\Big( \left({w^\star}^2 x_L x_{s-1}\right)_{s=2}^{L-1}\Big)^\top
    \end{pmatrix},
    \\
    \TF(\bX;\theta) &= \begin{pmatrix}
        g\cdot(x_{s-1})_{s=2}^{L-1}\sm\Big((-p(L-1-s))_{s=2}^{L-1}\Big)^\top
        \\ h\cdot(x_s)_{s=2}^{L-1}\sm\Big(\left(w^2 x_Lx_{s-1}\right)_{s=2}^{L-1}\Big)^\top
    \end{pmatrix}.
\end{align*}

And the loss function is expressed as:
{\small
\begin{align*}
    &\mathcal{L}(\btheta) = \frac{1}{2} \mathop{\bbE}\limits_{\bX\sim\bbN(0,1)^L} \left[\Vert f^\star(\bx)-\TF(x;\theta) \Vert^2\right]
    \\&= \frac{1}{2} {\bbE}_{\bX}\Bigg[\left(\frac{{\alpha^\star}}{1+{\alpha^\star}}x_{L-2}-g\cdot(x_{s-1})_{s=2}^{L-1}\sm\Big((-p(L-1-s))_{s=2}^{L-1}\Big)^\top\right)^2\Bigg]
    \\&\quad+\frac{1}{2} {\bbE}_{\bX}\Bigg[\left(\frac{1}{1+{\alpha^\star}} (x_s)_{s=2}^{L-1}\sm\Big( \left({w^\star}^2 x_L x_{s-1}\right)_{s=2}^{L-1}\Big)^\top-h\cdot(x_s)_{s=2}^{L-1}\sm\Big(\left(w^2 x_Lx_{s-1}\right)_{s=2}^{L-1}\Big)^\top\right)^2\Bigg].
\end{align*}}

The total loss can naturally be divided into two parts:
\begin{align*}
    \cL(\theta)=\cL_{\FG}(\theta)+\cL_{\IH}(\theta),
\end{align*}

where
\begin{align*}
    \cL_{\FG}(\theta)&=\cL_{\FG}(p,g)
    \\&=\frac{1}{2} {\bbE}_{\bX}\Bigg[\left(\frac{{\alpha^\star}}{1+{\alpha^\star}}x_{L-2}-g\cdot(x_{s-1})_{s=2}^{L-1}\sm\Big((-p(L-1-s))_{s=2}^{L-1}\Big)\right)^\top)^2\Bigg],
\end{align*}

{\small
\begin{align*}
    \cL_{\IH}(\theta)&=\cL_{\IH}(w,h)
    \\&=\frac{1}{2} \bbE_{\bX}\Bigg[\left(\frac{1}{1+{\alpha^\star}} (x_s)_{s=2}^{L-1}\sm\Big( \left({w^\star}^2 x_L x_{s-1}\right)_{s=2}^{L-1}\Big)^\top-h\cdot(x_s)_{s=2}^{L-2}\sm\Big(\left(w^2 x_Lx_{s-1}\right)_{s=2}^{L-1}\Big)^\top\right)^2\Bigg].
\end{align*}}

Notably, the dynamics of $(p,g)$ and $(w,h)$ are {\bf decoupled}, which allows us to analyze them separately.
% For the following analysis, we take initialization small enough and assume $L$ is large enough. Since $\frac{dw}{dt}$ and $\frac{dh}{dt}$ have factor $\frac{1}{L-2}$, $p$ and $g$ are firstly learned. 

Additionally, we denote the optimal values of the parameters as:
\begin{align*}
    p^{\star}=+\infty,\quad 
    g^\star= \frac{{\alpha^\star}}{1+\alpha^\star},\quad
    w^\star:=w^\star,\quad
    h^\star=\frac{1}{1+\alpha^\star}.
\end{align*}

For the initialization scale and the sequence length, we consider the case:
\begin{align*}
    \sigma_{\init}=\cO(1)\ll1,\quad L=\Omega(1/\sigma_{\init})\gg1.
\end{align*}

% We structure the proof in three parts. First, we introduce the setting and relevant definitions. Next, we divide the four parameters into two groups and analyze each group separately. 
% For each group, we first establish the monotonicity and asymptotic property of their dynamics, followed by estimation for critical value of the parameters.

\subsubsection{Dynamics of the parameters for \texorpdfstring{$4$}{}-gram}
\label{subsection: dynamics 4-gram}

% \todo[inline]{wmz: what is the definition of $t_1$? Is it the tuning point of realistic dynamics?}

First, we define two useful auxiliary functions:
\begin{align*}
    M(p) &:=\frac{1-e^{-p(L-2)}}{1-e^{-p}},\\ 
    m(p) &:=\frac{e^{-p}-(L-2)e^{-p(L-2)}+(L-3)e^{-p(L-1)}}{(1-e^{-p})^2}.
\end{align*}

Then, a straightforward calculation, combined with Lemma~\ref{lemma: calculate Gaussian} and Lemma~\ref{lemma: calculate softmax}, yields the explicit formulation of $\mathcal{L}_{\FG}(\btheta)$ and the GF dynamics of $p$ and $g$: 
% The proof of the above lemma can be established through easy calculation. Then 
\begin{equation}\label{proof: equ: FG loss}
    \mathcal{L}_{\FG}(\btheta) = \frac{1}{2}\left(\frac{\alpha^\star}{1+\alpha^\star}\right)^2 + \frac{1}{2}g^2\frac{M(2p)}{M(p)^2} - \frac{\alpha^\star g}{1+\alpha^\star}\frac{1}{M(p)}.
\end{equation}
% As we proved, in the total loss $\cL$, only $\cL_{\FG}$ depends on $p,g$. Then the dynamics of $p,g$ satisfy:
\begin{align*}
    \frac{\rd p}{\rd t} &= -\frac{\partial\cL}{\partial p}= -\frac{\partial\cL_{\FG}}{\partial p}=\frac{m(p)}{M(p)^2}\left[ g^2\frac{m(2p)}{m(p)}-g^2\frac{M(2p)}{M(p)}+\frac{\alpha^\star g}{1+\alpha^\star} \right],
    \\
    \frac{\rd g}{\rd t} &= -\frac{\partial\cL}{\partial g}= -\frac{\partial\cL_{\FG}}{\partial g}=\frac{\alpha^\star}{1+\alpha^\star}\frac{1}{M(p)}-g\frac{M(2p)}{M(p)^2},
\end{align*}

Equivalently, the dynamics can be written as:
\begin{align*}
    \frac{\rd p}{\rd t} &=\frac{m(p) g}{M(p)^2}\left( g^\star-g\frac{M(2p)}{M(p)}+g\frac{m(2p)}{m(p)}\right),
    \\
    \frac{\rd g}{\rd t}&=\frac{1}{M(p)}\left(g^\star-g\frac{M(2p)}{M(p)}\right).
\end{align*}

% \begin{align*}
%     \frac{dp}{dt} &= \frac{m(p)}{M(p)^2}\left[ g^2\frac{m(2p)}{m(p)}-g^2\frac{M(2p)}{M(p)}+\frac{\alpha^\star g}{1+\alpha^\star} \right],
%     \\
%     \frac{dg}{dt} &= \frac{\alpha^\star}{1+\alpha^\star}\frac{1}{M(p)}-g\frac{M(2p)}{M(p)^2}
% \end{align*}

Notice that at the initialization, it holds that $\frac{\rd p}{\rd t}|_{t=0}>0$ and $\frac{\rd g}{\rd t}|_{t=0}>0$.
Then we first define a hitting time:
\begin{align*}
    T_{1}^g:=\inf\{t>0:g(t)>g^\star\}.
\end{align*}

Noticing $g(0)=\sigma_{\rm init}\ll g^\star$ and the continuity, $T_{1}^g>0$.

Our subsequent proof can be divided into {\bf two phases}: a monotonic phase $t<T_{1}^g$, and a stable convergence phase $t>T_{1}^g$.

\underline{\bf Part I. Analysis for the monotonic phase $t<T_{1}^g$.}

\begin{align*}
    \frac{\rd p}{\rd t} &=\frac{m(p) g}{M(p)^2}\left( g^\star-g\frac{M(2p)}{M(p)}+g\frac{m(2p)}{m(p)}\right)=\frac{m(p) g}{M(p)^2}\left( g^\star-g\frac{1+e^{-p(L-2)}}{1+e^{-p}}+g\frac{m(2p)}{m(p)}\right),
    \\
    \frac{\rd g}{\rd t}&=\frac{1}{M(p)}\left(g^\star-g\frac{M(2p)}{M(p)}\right)=\frac{1}{M(p)}\left(g^\star-g\frac{1+e^{-p(L-2)}}{1+e^{-p}}\right).
\end{align*}
It is easy to see that $p,g$ are monotonically increasing for $t<T_1^g$. We can choose sufficiently large 
\begin{align*}
    L=\Omega(1/p(0))=\Omega(1/\sigma_{\init})
\end{align*} 
% (or choose $L=\Omega(1/\sigma_{\init})$ directly),[todo ......]

such that:
\begin{align*}
    (L-3)e^{-(L-3)p(t)},\ e^{-(L-5)p(t)}< 0.0001,\quad \forall p>\sigma_{\init}.
\end{align*}

Then we can calculate the following three terms in the dynamics:
\begin{gather*}
    \frac{m(p)}{M^2(p)}=\frac{e^{-p}\left(1-(L-2)e^{-p(L-3)}+(L-3)e^{-p(L-2)}\right)}{1-e^{-p(L-2)}}=\frac{e^{-p}(1+\xi_1(p))}{1+\xi_2(p)},
    \\
    \frac{1}{M(p)}=\frac{1-e^{-p(L-2)}}{1-e^{-p}}=\frac{1+\xi_3(p)}{1-e^{-p}},
\end{gather*}
\begin{align*}
    \frac{m(2p)}{m(p)}&=\frac{e^{-p}\left(1-(L-2)e^{-2p(L-3)}+(L-3)e^{-2p(L-2)}\right)}{(1+e^{-p})^2\left(1-(L-2)e^{-p(L-3)}+(L-3)e^{-p(L-2)}\right)}
    \\&=\frac{e^{-p}(1+\xi_4(p))}{(1+e^{-p})^2(1+\xi_5(p))},
\end{align*}

where the error functions satisfy:
\begin{align*}
    |\xi_1(p)|,\cdots,|\xi_5(p)|\leq 0.0001,\ \forall t > T_1^g.
\end{align*}

Then the dynamics satisfy:
\begin{align*}
    \frac{\rd p}{\rd t} &=\frac{e^{-p} g (1+\xi_1(p))}{1+\xi_2(p)}\left(g^\star-g\frac{1+e^{-p(L-2)}}{1+e^{-p}}+\frac{ge^{-p}(1+\xi_3(t))}{(1+e^{-p})^2(1+\xi_5(t))}\right),
    \\
    \frac{\rd g}{\rd t}&=\frac{1+\xi_3(p)}{1-e^{-p}}\left(g^\star-g\frac{1+e^{-p(L-2)}}{1+e^{-p}}\right).
\end{align*}

 When $g<\frac{1}{2}\frac{\alpha^\star}{1+\alpha^\star}$, we have
 $$\frac{\rd p}{\rd g}\leq 2\left(e^{-p}-e^{-2p}\right)g.$$
By define $T_{1/2}^g:=\inf\{t>0:g(t)>g^\star/2\}$ and $\tilde{p}:=p(T_{1/2}^g)$, we have
 $$\ln(e^{\tilde{p}}-1)\leq \frac{1}{4}{g^\star}^2-g(0)^2+e^{p(0)}-1+\ln(e^{p(0)}-1)$$
then $\tilde{p}\leq\mathcal{O}(\sqrt{p(0)})$, from which we infer that p barely increases when $t\leq T_{1/2}^g$.

For $0\leq t\leq T_{1/2}^g$,
$$\frac{\rd g}{\rd t}\geq\frac{1}{1-e^{-p(0)}}\left[g^\star-\frac{g}{1+e^{-p(0)}}\right]$$
$$g\geq g^\star(1+e^{-p(0))})+\left[g(0)-g^\star(1+e^{-p(0)})\right]\exp\left(\frac{-t}{1-e^{-2p(0)}}\right)$$
so
$$T_{1/2}^g\leq (1-e^{-2p(0)})\ln\left(\frac{g^\star(1+e^{-p(0)})-g(0)}{g^\star\left((1+e^{-p(0)})-\frac{1}{2}\right)}\right)=\mathcal{O}\left(2p(0)\right)$$

For $T_{1/2}^g\leq t\leq T_1^g$, let $p_1:=p(T_1^g)$,
\begin{align*}
    \frac{\rd p}{\rd g}&\leq 1.01 e^{-p}(1-e^{-p})g\left(1+\frac{\frac{g}{1+e^{-p}}-\frac{g}{(1+e^{-p})^2}}{\frac{\alpha^\star}{1+\alpha^\star}-\frac{g}{1+e^{-p}}}\right)
    \\&\leq \frac{1.01}{4}\frac{\alpha^\star}{1+\alpha^\star} (1+e^{-p_1})
\end{align*}
then 
$$p_1-p(0)\leq \frac{1.01}{4}\left(\frac{\alpha^\star}{1+\alpha^\star}\right)^2 (1+e^{p_1}),$$
$$p_1\leq\frac{1}{2\left(\frac{\alpha^\star}{1+\alpha^\star}\right)^2-1},$$
and we take $\alpha^\star>1$.

Since for $T_{1/2}^g\leq t\leq T_1^g$,
$$\frac{\rd p}{\rd t}\leq 2e^{-p}g^\star\left(g^\star-\frac{1}{8}g^\star\right),$$
$$\frac{\rd p}{\rd t}\geq \frac{1}{2}e^{-p}g^\star\left(g^\star-\frac{1}{1+e^{-p_1}}g^\star\right),$$
we have
$$T_1^g-t_1\leq\mathcal{O}\left((e^{2p_1}-1)\left(\frac{1+\alpha^\star}{\alpha^\star}\right)^2\right).$$
% $$T_1^g-t_1=\mathcal{O}\left((e^{p_1}-1)\left(\frac{\alpha^\star}{1+\alpha^\star}\right)^2\right)=\mathcal{O}\left(\left(\frac{\alpha^\star}{1+\alpha^\star}\right)^2\right).$$

Hence, putting the two part of time together we have
\begin{equation}\label{proof: equ: time t}
\begin{aligned}
T_1^g&\leq\mathcal{O}\left(p(0)+(e^{2p_1}-1)\left(\frac{1+\alpha^\star}{\alpha^\star}\right)^2\right)
\\&=\cO\left(\sigma_{\init}+(e^{2p_1}-1)\left(\frac{1+\alpha^\star}{\alpha^\star}\right)^2\right)
=\cO(1).
\end{aligned}
\end{equation}

% which can implies that 
% \begin{align*}
%     T_{1}^g\leq \cO(1).
% \end{align*}
% \begin{align*}
%     p(T_1^g)\geq ....
% \end{align*}

% \todo[inline]{wmz: complete the proof in Step I}

\underline{\bf Part II. Analysis for the convergence phase $t>T_{1}^g$.}

We will prove that, in this phase, $(p,g)$ keep in a stable region, and the convergence occurs.

Recall the dynamics:
\begin{align*}
    \frac{\rd p}{\rd t} &=\frac{m(p) g}{M(p)^2}\left( g^\star-g\frac{1+e^{-p(L-2)}}{1+e^{-p}}+g\frac{m(2p)}{m(p)}\right),
    \\
    \frac{\rd g}{\rd t}&=\frac{1}{M(p)}\left(g^\star-g\frac{1+e^{-p(L-2)}}{1+e^{-p}}\right).
\end{align*}

Using contradiction, it is easy to verify that for all $t>T_1^g$,
\begin{align*}
    g^\star<g(t)<2g^\star,\quad\frac{\rd p(t)}{\rd t}>0,
\end{align*}

which means $g$ has entered a stable region (although it is possible that $g$ is non-monotonic), while $p$ keeps increase. In fact, if $T^g_{2g^\star}:=\inf\{t>0:g(t)=2g^\star\}$, then $\frac{\rd g}{\rd t}|_{T^g_{2g^\star}}<0$, which leads to a contradiction. If $T^{{\rd p}/{\rd t}}_0:=\inf\{t>0:\frac{\rd p(t)}{\rd t}=0\}$, then 
$$\left.\left(g^\star-g\frac{1+e^{-p(L-2)}}{1+e^{-p}}+g\frac{m(2p)}{m(p)}\right)\right|_{T^{{\rd p}/{\rd t}}_0}=0,\quad\frac{\rd g}{\rd t}<0,$$
$$\frac{\rd}{\rd t}\left(g^\star-g\frac{1+e^{-p(L-2)}}{1+e^{-p}}+g\frac{m(2p)}{m(p)}\right)=-g'\frac{1+e^{-p(L-2)}}{1+e^{-p}}+g'\frac{m(2p)}{m(p)}>0,$$
where the last inequality leads to a contradiction.

Thus, $p(t)>p(T_{1}^g)>p(0)=\sigma_{\init}$ holds in this phase. 
Therefore, the dynamics 
\begin{align*}
    \frac{\rd p}{\rd t} &=\frac{e^{-p} g (1+\xi_1(p))}{1+\xi_2(p)}\left(g^\star-g\frac{1+e^{-p(L-2)}}{1+e^{-p}}+\frac{ge^{-p}(1+\xi_3(t))}{(1+e^{-p})^2(1+\xi_5(t))}\right),
    \\
    \frac{\rd g}{\rd t}&=\frac{1+\xi_3(p)}{1-e^{-p}}\left(g^\star-g\frac{1+e^{-p(L-2)}}{1+e^{-p}}\right),
\end{align*}

also satisfy \begin{align*}
    |\xi_1(p)|,\cdots,|\xi_5(p)|\leq 0.0001,\ \forall t > T_1^g.
\end{align*}

For simplicity, we consider the transform:
\begin{align*}
    u:=e^{-p}.
\end{align*}

Then the dynamics of $u$ and $g$ can be written as:
\begin{align*}
    \frac{\rd u}{\rd t}&=-\frac{ (1+\xi_1(p)) u^2 g}{1+\xi_2(p)}\left(g^\star-g\frac{1+u^{L-2}}{1+u}+\frac{gu(1+\xi_4(p))}{(1+u)^2(1+\xi_5(p))}\right),
    \\
    \frac{\rd g}{\rd t}&=\frac{1+\xi_3(p)}{1-u}\left(g^\star-g\frac{1+u^{L-2}}{1+u}\right).
\end{align*}

Notice that this dynamics are controlled by high-order terms. Consequently, we construct a variable to reflect the dynamics of high-order term:
\begin{align*}
    v:=ug^\star+(g^\star-g).
\end{align*}

Then the dynamics of $u$ and $v$ satisfy:
\begin{align*}
    \frac{\rd u}{\rd t}&=-\frac{ (1+\xi_1(p))u^2 g}{1+\xi_2(p)}\left(\frac{v-u^{L-2}g}{1+u}+\frac{gu(1+\xi_4(p))}{(1+u)^2(1+\xi_5(p))}\right),
    \\\frac{\rd v}{\rd t}&=-\frac{ (1+\xi_1(p))u^2 gg^\star}{1+\xi_2(p)}\left(\frac{v-u^{L-2}g}{1+u}+\frac{gu(1+\xi_4(p))}{(1+u)^2(1+\xi_5(p))}\right)-\frac{1+\xi_3(p)}{1-u^2}\left(v-u^{L-2}g\right).
\end{align*}

Now we consider the Lyapunov function about $u,v$:
\begin{align*}
    G(u,v):=\frac{1}{2}\left(u^2+v^2\right).
\end{align*}

Then it is straightforward:
\begin{align*}
    &\frac{\rd G}{2\rd t}=u\frac{\rd u}{\rd t}+v\frac{\rd v}{\rd t}
    \\=&-\frac{u^3g(1+\xi_1(p))}{1+\xi_2(p)}\left(\frac{v-u^{L-2}g}{1+u}+\frac{gu(1+\xi_4(p))}{(1+u)^2(1+\xi_5(p))}\right)
    \\&-\frac{ (1+\xi_1(p))u^2 v gg^\star}{1+\xi_2(p)}\left(\frac{v-u^{L-2}g}{1+u}+\frac{gu(1+\xi_4(p))}{(1+u)^2(1+\xi_5(p))}\right)
    \\&-\frac{1+\xi_3(p)}{1-u^2}\left(v-u^{L-2}g\right)v.
\end{align*}

By $|\xi_1|,\cdots,|\xi_5|\leq 0.0001$, we have the following estimate for the Lyapunov dynamics:
\begin{align*}
    \frac{\rd G}{2\rd t}\leq& \frac{1.001 g}{1+u}|u^3 v| +\frac{1.0001 g ^2}{1+u} u^{L+1} - \frac{0.999g^2}{(1+u^2)}u^4
    \\&-\frac{0.999 g g^\star}{1+u}u^2 v^2+\frac{1.001 g^2 g^\star}{1+u}|u^{L}v|+\frac{1.001 g^2g^\star}{(1+u^2)}|u^3 v|
    \\&-\frac{0.999}{1-u^2}v^2+\frac{1.001 g }{1-u^2}|u^{L-2}v|
\end{align*}

By $u^{L-5}=e^{-p(L-5)}<0.0001$ and $0<u<e^{-p(T_1^g)}$, we further have:
\begin{align*}
    &\frac{\rd G}{2\rd t}\leq \frac{1.002 g}{1+u}|u^3 v|-\frac{0.99g^2}{(1+u)^2}u^4
    -\frac{0.999 g g^\star}{1+u}u^2 v^2+\frac{1.005 g^2g^\star}{(1+u)^2}|u^3 v|-\frac{0.999}{1-u^2}v^2
    \\\leq&-\frac{0.99 g^2}{(1+u)^2}u^4-\frac{0.99 g{g^\star}}{1+u}u^2v^2-\frac{0.99}{1-u^2}v^2+1.01\left(\frac{g}{1+u}+\frac{g^2{g^\star}}{(1+u)^2}\right)|u^3v|.
\end{align*}
By using the following inequalities:
$$ \frac{g^2{g^\star}}{(1+u)^2}|u^3v|\leq\frac{1}{2}\left(\frac{1.98}{1.01}\frac{g{g^\star}}{1+u}u^2v^2+\frac{1.01}{1.98}\frac{g^3{g^\star}}{(1+u)^3}u^4\right) $$
$$ \frac{g}{1+u}|u^3v|\leq\frac{1}{2}\left(\frac{0.99}{1.01}(1+u)v^2+\frac{1.01}{0.99}\frac{g^2}{(1+u)^3}u^6\right) $$
$$ -\frac{1}{1-u^2}+\frac{1}{2}(1+u)<-\frac{2}{5} $$

we have
\begin{align*}
    \frac{\rd G}{\rd t} \leq -0.99\frac{g^2}{(1+u)^2}u^4+\frac{1.01}{3.96}\frac{g^3{g^\star}}{(1+u)^3}u^4+\frac{1.01}{1.98}\frac{g^2}{(1+u)^3}u^6-\frac{1.98}{5}v^2.
\end{align*}

Since ${g^\star}<g<2{g^\star}$, $u>0$ for $t> T_1^g$, and $\frac{u^2}{1+u}\leq \frac{1}{2}$ for $0\leq u\leq 1$, we have:
\begin{align*}
    &\frac{1}{4}\frac{g^3{g^\star}}{(1+u)^3}+\frac{1}{2}\frac{g^2u^2}{(1+u)^3}
    \leq \frac{g^2}{(1+u)^2}\left(\frac{{{g^\star}}^2}{2(1+u)}+\frac{u^2}{2(1+u)}\right)
    \\\leq& \frac{g^2}{(1+u)^2}\left(\frac{1}{2}+\frac{1}{4}\right)
    =\frac{3}{4}\frac{g^2}{(1+u)^2},
\end{align*}

then
\begin{align*}
    &\frac{\rd G(u,v)}{\rd t}\leq-0.22\frac{g^2}{(1+u)^2}u^4-\frac{2}{5}v^2
    \\\leq& -\frac{0.99}{16}{{g^\star}}^2u^4-\frac{1.98}{5}v^2
    \leq-\frac{{{g^\star}}^2}{65}G(u,v)^2,
\end{align*}

which implies:
\begin{align*}
    G(u(t),v(t))\leq\frac{1}{G(u(t_1),v(t_1))+\frac{{g^\star}^2}{64}(t-t_1)},\quad\forall t> T_1^g.
\end{align*}

Hence, 
\begin{align*}
    u^2(t),\quad v^2(t)=\cO\left(\frac{1}{{g^\star}^2 t}\right)=\cO\left(\frac{1}{t}\right),\quad\forall t> T_1^g=\cO(1)
\end{align*}

which implies:
\begin{equation}\label{proof: equ: convergence, p g}
\begin{aligned}
 e^{-p(t)} &= u(t)=\mathcal{O}\left(\frac{1}{\sqrt{t}}\right),\quad\forall t> T_1^g=\cO(1);
 \\
 g(t)-{g^\star} &= g^\star u(t)-v(t)\leq\mathcal{O}\left(\frac{g^\star}{\sqrt{t}}\right)+\cO\left(\frac{1}{\sqrt{t}}\right)=\cO\left(\frac{1}{\sqrt{t}}\right),\quad\forall t> T_1^g=\cO(1).
\end{aligned}
\end{equation}

{\bf Notably}, these proofs capture the {\bf entire} training dynamics of $p,g$, from $t=0$ to $t=T_1^g$, and finally to $t\to+\infty$, providing a fine-gained analysis for each phase.

\subsubsection{Dynamics of the parameters for induction head}
\label{subsection: dynamics induction head}

Recall the partial loss about the induction head:
{\small
\begin{align*}
    \cL_{\IH}(\theta)={\frac{1}{2} \bbE_{\bX}\Bigg[\left(\frac{1}{1+{\alpha^\star}} (x_s)_{s=2}^{L-1}\sm\Big( \left({w^\star}^2 x_L x_{s-1}\right)_{s=2}^{L-1}\Big)^\top-h\cdot(x_s)_{s=2}^{L-2}\sm\Big(\left(w^2 x_Lx_{s-1}\right)_{s=2}^{L-2}\Big)^\top\right)^2\Bigg]}.
\end{align*}}

{\bf Technical simplification.}
Unlike $\cL_{\FG}(\theta)$, the denominators of the softmax terms $\sm\Big( \left({w^\star}^2 x_L x_{s-1}\right)_{s=2}^{L-1}\Big)$ and $\sm\Big(\left(w^2 x_Lx_{s-1}\right)_{s=2}^{L-2}\Big)$ in $\cL_{\IH}(\theta)$ depend on the input tokens $X$, making it hard to derive a closed-form expression for $\cL_{\IH}(\theta)$. 
In~\cite{bai2023transformers}, the authors consider a simplified transformer model, which replaces $\sm(z_1,\cdots,z_L)$ with $\frac{1}{L}\exp(z_1,\cdots,z_L)$. 
This approximation is nearly tight when $z_1,\cdots,z_L\approx 0$.
Notice that 
1) ${w}^2 x_L x_{s-1} \approx 0$ holds near the small initialization, i.e., for $w\approx\sigma_{\rm init} \ll 1$. In fact, our analysis shows that $w \approx \sigma_{\rm init}$ is maintained over a long period.
2) ${w}^\star = \mathcal{O}(1)$, which implies that ${w}^2 x_L x_{s-1} \approx 0$ for most input sequence.
Thus, we adopt the simplification used in~\cite{bai2023transformers}, resulting in the following approximation of the loss function:
\begin{align*}
    \cL_{\IH}(\theta):=\frac{1}{2}{\bbE}_{\bX}\Bigg[\left(\frac{1}{1+{\alpha^\star}}\frac{1}{L-2}\sum_{s=2}^{L-1} \exp ({w^\star}^2 x_L x_{s-1})x_s-h\frac{1}{L-2}\sum_{s=2}^{L-2}\exp(w^2 x_Lx_{s-1})x_s\right)^2\Bigg].
\end{align*}

Then by a straightforward calculation with Lemma~\ref{lemma: calculate Gaussian}, we can derive its explicit formulation: 
% The proof of the above lemma can be established through easy calculation. Then 
\begin{equation}\label{proof: equ: IH loss}
\begin{aligned}
    \cL_{\IH}(\theta) &=\frac{(1-4{w^{\star}}^4)^{-\frac{1}{2}}}{2(1+\alpha^\star)^2(L-2)} + \frac{1}{2}\frac{h^2}{L-2}(1-4w^4)^{-\frac{1}{2}}-\frac{h(1-(w^2+{w^{\star}}^2)^2)^{-\frac{1}{2}}}{(1+\alpha^\star)(L-2)}.
\end{aligned}
\end{equation}

Furthermore, we can calculate GF dynamics as follows:
\begin{align*}
    \frac{\rd w}{\rd t} &= \frac{h}{(1+\alpha^\star)(L-2)}(1-(w^2+{w^\star}^2)^2)^{-\frac{3}{2}}\cdot(w^2+{w^\star}^2)\cdot 2w
    -\frac{h^2}{L-2}(1-4w^4)^{-\frac{3}{2}}\cdot 4w^3,
    \\
    \frac{\rd h}{\rd t} &= \frac{1}{(1+\alpha^\star)(L-2)}(1-(w^2+{w^\star}^2)^2)^{-\frac{1}{2}}-\frac{h}{L-2}(1-4w^4)^{-\frac{1}{2}}.
\end{align*}

For simplicity, we denote:
\begin{align*}
    w^\star:=w^\star,\quad h^\star:=\frac{1}{1+\alpha^\star}.
\end{align*}

\underline{\bf Part I. The trend and monotonicity of $w,h$.}

For simplicity, we denote the tuning time point of $h$:
\begin{align*}
    T_2^h:=\inf\left\{t>0:\frac{\rd h(t)}{\rd t}=0\right\}.
\end{align*}

In this step, we will prove the following three claims regarding the trend and monotonicity of $w,h$, which are essential for our subsequent analysis:
\begin{itemize}[leftmargin=2em]
    \item {\bf (P1.1)}  $h$ initially increases beyond $h^\star$, and then remains above this value.
    \item {\bf (P1.2)} $w$ keeps increasing but always stays below $w^\star$.
    \item {\bf (P1.3)} $h$ increases before $T_2^h$, but decreases after $T_2^h$.
\end{itemize}

\underline{\em {\bf (P1.1)} $h$ initially increases beyond $h^\star$, and then remains above this value.}

% Intuitively, noticing $|\frac{dw}{dt}|:|\frac{dh}{dt}|\sim wh:1\ll 1$, $w$ is the slow dynamic, while $h$ is the fast dynamics. 

We will prove that initially, $h$ increases beyond $h^\star$, and keeps growing beyond $h^\star$. Define 
\begin{align*}
    T_1^h:=\inf\{t>0: h(t)>h^\star\},
\end{align*}
% \todo[inline]{where is the time estimate for $h$ increasing to $h^\star$, i.e. $T_1^h$?}

we will prove that $h$ remains above $h^\star$ thereafter.

For simplicity, we denote 
\begin{align*}
    \psi(x)=(1-x^2)^{-\frac{1}{2}},\quad \phi(x)=(1-x^2)^{-\frac{3}{2}}\cdot x,
\end{align*} 

then the dynamics holds:
\begin{align*}
    \frac{\rd h}{\rd t}
    &= \frac{h}{L-2}\psi(w^2+{w^\star}^2)\left[\frac{h^\star}{h}-\frac{\psi(2w^2)}{\psi(w^2+{w^\star}^2)}\right],
    \\
    \frac{\rd w}{\rd t}
    &= \frac{2h^2w}{L-2}\cdot\phi(w^2+{w^\star}^2)\cdot\left[\frac{h^\star}{h}-\frac{\phi(2w^2)}{\phi(w^2+{w^\star}^2)}\right].
\end{align*}

Notice that
$\frac{\phi(2w^2)}{\phi(w^2+{w^\star}^2)}<\frac{\psi(2w^2)}{\psi(w^2+{w^\star}^2)},\ w<w^\star$, while $\frac{\phi(2w^2)}{\phi(w^2+{w^\star}^2)}>\frac{\psi(2w^2)}{\psi(w^2+{w^\star}^2)},\ w>w^\star$.
% $$\frac{\phi(2v)}{\phi(v+v^*)}<\frac{\psi(2v)}{\psi(v+v^*)},\ w<w^\star
% $$ 
% $\frac{\phi(2v)}{\phi(v+v^*)}\geq\frac{\psi(2v)}{\psi(v+v^*)}$.

We denote the first hitting time of $h$ decreasing to $h^\star$ as $T^h_{h^\star}$:
\begin{align*}
    T^h_{h^\star}:=\inf\left\{t>T_2^h:h(t)<h^\star\right\}.
\end{align*}

If $w(T^h_{h^\star})\geq w^\star$, then at the first hitting time of $w$ increasing to $w^\star$, $\frac{\rd w}{\rd t}<0$, which leads to a contradiction. If $w(T^h_{h^\star})< w^\star$, then $\frac{\rd h}{\rd t}|_{T^h_{h^\star}}>0$, which also leads to a contradiction. 
Hence, $T^h_{h^\star}=+\infty$,
which means that $h$ always remains above $h^\star$ for $t>T_2^h$.

\underline{\em {\bf (P1.2)} $w$ keeps increasing but always below $w^\star$.}

We first prove that $w$ always remains below $w^\star$. We denote the first hitting time of $w$ increasing to $w^\star$ as ${t'}$, then it is not difficult to see $\frac{\rd w}{\rd t}|_{t'}<0$, which leads to a contradiction.

Next we prove that $w$ keeps increasing throughout. We define the following functions
$$ H := \frac{1}{1+\alpha^\star}\left(1-(w^2+{w^\star}^2)^2\right)^{-\frac{3}{2}}(w^2+{w^\star}^2)-h(1-4w^4)^{-\frac{3}{2}}\cdot 2w^2 $$
$$ Q := \frac{1}{1+\alpha^\star}\left(1-(w^2+{w^\star}^2)^2\right)^{-\frac{1}{2}}-h\left(1-4w^4\right)^{-\frac{1}{2}} $$
If at some $\bar{t}$, $\frac{\rd w}{\rd t}$ reaches its zero point at the first time, then
$$\frac{\rd H}{\rd t}\bigg|_{\bar{t}}=-h'(\bar{t})(1-4{w^\star}^4)^{-\frac{3}{2}}\cdot 2w({\bar{t}})>0,$$
which leads to a contradiction. Hence $\bar{t}$ does not exist and $w$ keeps increasing.

\underline{\em {\bf (P1.3)} After the tuning point $t>T_2^h$, $h$ will be monotonically decreasing.}

The first sign-changing zero point of $\frac{\rd h}{\rd t}$ is $T_2^h$, then $Q(T_2^h)=0$. $H(T_2^h)>0$,
\begin{align*}
    \frac{\rd Q}{\rd t}\bigg|_{T_2^h}&=\frac{1}{1+\alpha^\star}(1-(w(T_2^h)^2+{w^\star}^2)^2)^{-\frac{1}{2}}\cdot 2w(T_2^h)\cdot w'(T_2^h)
    \\&\quad\cdot\left[(1-(w(T_2^h)^2+{w^\star}^2)^2)^{-1}\cdot(w(T_2^h)^2+{w^\star}^2)-(1-4w(T_2^h)^4)^{-1}\cdot 4w(T_2^h)^2\right].
\end{align*}
We can see that $T_2^h$ is a sign-changing zero point only if 
$$\frac{(1-4w(T_2^h)^4)\cdot(w(T_2^h)^2+{w^\star}^2)}{(1-(w(T_2^h)^2+{w^\star}^2)^2)\cdot 4w(T_2^h)^2}<1,$$
i.e. we have:
\begin{equation}\label{proof: equ: bounded, turning point, w}
w(T_2^h)>w^{\circ}:=\sqrt{\frac{3-4{w^\star}^4-\sqrt{(4{w^\star}^4-3)^2-16{w^\star}^4}}{8{w^\star}^2}}\geq\frac{w^\star}{2},
\end{equation}
when $w^\star=\cO(1)$.
% If $w^\star=0.49$, $w^{\circ}\approx 0.3$.

Next we show that $h$ keeps decreasing after $T_2^h$. We denote the first zero point of $\frac{\rd h}{\rd t}$ as $t^\circ$, then $Q(t^\circ)=0$. Since $\frac{\rd w}{\rd t}|_{t^\circ}>0$, we have $\frac{\rd Q}{\rd t}|_{t^\circ}>0$ which leads to a contradiction. Hence $t^\circ$ does not exist and $h$ keeps decreasing after $T_2^h$.

\underline{\bf Part II. Estimation of $T_1^h$, $T_2^h$, and the tight estimate of $w(t)$ before $T_2^h$.}

At the first stage, we prove that $h$ grows first and 
$w$ barely increases. If $w\leq 0.01 w^\star$ and $h\leq\frac{1}{1+\alpha^\star}\frac{(1-{w^\star}^4)^{-\frac{1}{2}}}{(1-0.01^4{w^\star}^4)^{-\frac{1}{2}}}$, 
$$ \frac{\rd h}{\rd t}\geq \frac{-1}{L-2}\left[h(1-0.01^4{w^\star}^4)^{-\frac{1}{2}}-\frac{1}{1+\alpha^\star}(1-{w^\star}^4)^{-\frac{1}{2}}\right], $$
\begin{equation}\label{equ: h approx}
h\geq\frac{1}{1+\alpha^\star}\frac{(1-{w^\star}^4)^{-\frac{1}{2}}}{(1-0.01^4{w^\star}^4)^{-\frac{1}{2}}}-\left[\frac{1}{1+\alpha^\star}\frac{(1-{w^\star}^4)^{-\frac{1}{2}}}{(1-0.01^4{w^\star}^4)^{-\frac{1}{2}}}-h(0)\right]\exp\left(\frac{-t}{(L-2)(1-0.01{w^\star}^4)^{\frac{1}{2}}}\right).
\end{equation}
For $h$ increasing from $h(0)$ to $\frac{1}{1+\alpha^\star}$, it takes
\begin{align}\label{equ: T_1^h}
T_1^h&\leq(1-0.01{w^\star}^4)^{\frac{1}{2}}(L-2)\ln\left(\frac{1}{1-\frac{(1-{w^\star}^4)^{\frac{1}{2}}}{(1-0.01^4{w^\star}^4)^{\frac{1}{2}}}}\right)
\notag \\&\leq 2(L-2)(1-\frac{1}{2}{w^\star}^4)=\cO(L). 
\end{align}

% As for the lower bound of $T_1^h$. We denote the hitting time 
% $
% T_{1/2}^{h}:=\inf\left\{t>0:h(t)>\frac{1}{2}h^\star\right\},
% $
% which satisfy $T_{1/2}^{h}<T_1^h$. In a similar way, we can prove $T_{1}^{h}=\Omega(L)$. which implies
% \begin{align*}
%     T_{1/2}^{h},\quad T_1^h=\Theta(L).
% \end{align*}

For $0\leq t\leq T_1^h$,
$$ \frac{\rd w}{\rd t}\leq\frac{1}{L-2}(1-4{w^\star}^4)^{-\frac{3}{2}}\cdot {w^\star}^2\cdot 4w. $$
Hence, it take $\mathcal{O}(L\log(1/\sigma_{\init}))$ for $w$ to reach $0.01w^\star$, which allows sufficient time for $h$ to reach $\frac{1}{1+\alpha^\star}$ beforehand. 

Therefore, there exists a small constant $\varepsilon(w(0),w^\star)$ only depends on $w(0)$ and $w^\star$ such that $h$ is dominated by $1+\varepsilon(w(0),w^\star)$ times right hand side of (\ref{equ: h approx}), from which we deduce that (\ref{equ: T_1^h}) is a tight estimation of $T_1^h$ instead of an upper bound, i.e. $T_1^h=\Theta(L)$.

We then give a bound for $h(T_2^h)$. By $\frac{\rd h}{\rd t}=0$, 
$$ h(T_2^h)/h^\star\leq\frac{(1-4w^4)^{\frac{1}{2}}}{(1-(w^2+{w^\star}^2)^2)^{\frac{1}{2}}}:=r(w).$$
Moreover, $r(w)$ is an decreasing function of $w$ for $w>w^{\circ}$, and 
$w^{\circ}$ is a function of $w^\star$, we have
$$ h(T_2^h)/h^\star\leq r(w^{\circ}):=R(w^\star),$$

where $w^\circ$ is a function about $w^\star$, defined in Eq.~\eqref{proof: equ: bounded, turning point, w}.
It is clear that 
$$R(w^\star=0)=1,\quad R'(w^\star=0)=0.$$

Then using the continuity of $R'(\cdot)$ (in $[0,0.4]$), there exists $c>0$ such that $|R'(w^\star)|<0.04$ holds for all $0<w^\star<c$, which implies:
\begin{align*}
    R(w^\star)=R(0)+\int_0^{w^\star} R'(v) \rd v 
    <1+0.04 w^\star,\quad 0<w^\star<c.
\end{align*}

i.e., if $w^\star=O(1)$, then $R(w^\star)<1+0.04w^\star$.
% It is clear that $R(w^\star)$ is a convex function and is dominated by $\gamma(w^\star)=1+ 0.04375w^\star\leq 1.0175$ for $0\leq w^\star\leq 0.4$. 
This implies:
\begin{equation}\label{proof: equ: bounded, h}
 h^\star\leq h(t)\leq(1+0.04375w^\star) h^\star,\quad\forall t\geq T_1^h.
\end{equation}

By some computation, we can prove that $w^{\circ}(w^\star)$ is an increasing function of $w^\star$, and is always above $\frac{1}{2}w^\star$. Thus we obtain a lower bound of $w^{\circ}$ for the estimation of lower bound of $T_2^h$: 

For the second stage, $h$ barely changes and $w$ starts to grow exponentially fast, and we use the tight estimation of $T_{1/2}^{w}:=\inf\left\{t>0:w(t)>\frac{1}{2}w^\star\right\}$ to give a lower bound of $T_2^h$. During this stage,
\begin{align*}
    \frac{\rd w}{\rd t}&\leq\frac{2w}{(1+\alpha^\star)^2(L-2)}\left[(1-(w^2+{w^\star}^2)^2)^{-\frac{3}{2}}\cdot(w^2+{w^\star}^2)-(1-4w^4)^{\frac{3}{2}}\right]
    \\&\leq \frac{2w}{(1+\alpha^\star)^2(L-2)}(1-4{w^\star}^4)^{\frac{3}{2}}\cdot 2{w^\star}^2,
\end{align*}
and $w$ has upper bound 
\begin{equation}\label{equ: w upper}
w\leq w(0)\exp\left(\frac{4{w^\star}^2(1-4{w^\star}^4)^{\frac{3}{2}}}{(1+\alpha^\star)(L-2)}t\right). 
\end{equation}
Hence, the lower bound of time for $w$ to reach $\frac{1}{2}w^\star$ is
$$ T_{1/2}^w-T_1^h=\frac{(1+\alpha^\star)^2(L-2)}{4{w^\star}^2(1-4{w^\star}^4)^{\frac{3}{2}}}\ln(\frac{w^\star}{2w(0)}),$$
and lower bound for $T_{1/2}^w$ is 
\begin{align}\label{equ: T_{1/2}^w}
    T_{1/2}^w &\geq (L-2)\left[\frac{(1+\alpha^\star)^2\ln(\frac{w^\star}{2w(0)})}{4{w^\star}^2(1-4{w^\star}^4)^{\frac{3}{2}}}-\ln\left(1-(1-{w^\star}^4)^{\frac{1}{2}}\right)\right]\notag
    \\&\geq \frac{(L-2)(1+\alpha^\star)^2}{16{w^\star}^2}\ln\left(\frac{1}{w(0)}\right)=\Omega\left(\frac{(1+\alpha^\star)^2L}{{w^\star}^2}\log\left(\frac{1}{\sigma_\init}\right)\right).
\end{align}

On the other hand, we estimate the lower bound of $w$. Let
$$ C(x)=(1-x^2)^{-\frac{3}{2}}\cdot x ,$$
then
$$ C'(x) = 3(1-x^2)^{-\frac{5}{2}}x^2+(1-x^2)^{-\frac{3}{2}}>1,\quad 0<x<1 ,$$
$$ C''(x) = 15x^3(1-x^2)^{-\frac{7}{2}}+6x(1-x^2)^{-\frac{5}{2}}+3x(1-x^2)^{-\frac{5}{2}}>0,\quad 0<x<1 .$$
$C(x)$ is a monotonically increasing convex function on $(0,1)$ and $C(x)\geq x$.

Using conclusions above, before $w^2$ increases to $\frac{1}{2\gamma(w^\star)+\beta-1}{w^\star}^2$ for some $\beta>0$, 
\begin{align*}
    &\quad C(w^2+{w^\star}^2)
    \\&\geq C((2\gamma(w^\star)+\beta)w^2)
    \\&\geq C(2\gamma(w^\star)\cdot w^2)+C(\beta w^2)\quad(\text{Lemma~\ref{lemma: Weighted Karamata Inequality}})
    \\&\geq \gamma(w^\star)\cdot C(2w^2)+\beta w^2 \quad(C(ax)\geq aC(x)\text{, for } a>1)
\end{align*}
then we have
\begin{align*}
    \frac{\rd w}{\rd t}&\geq \frac{2w}{(1+\alpha^\star)^2(L-2)}(C(w^2+{w^\star}^2)-\gamma(w^\star)\cdot C(2w^2)) 
    \\&\geq \frac{2w}{(1+\alpha^\star)^2(L-2)}\frac{\beta}{\gamma(w^\star)+\beta}{w^\star}^2
\end{align*}
and 
$$ w\geq w(0)\exp\left(\frac{2\beta}{\gamma(w^\star)+\beta}\frac{1}{(1+\alpha^\star)^2(L-2)}{w^\star}^2t\right). $$

Take $\beta=2$, then 
\begin{equation}\label{proof: equ: before point, w}
w\geq w(0)\exp\left(\frac{{w^\star}^2t}{(1+\alpha^\star)^2(L-2)}\right),\ \forall t\in[0,T_{1/2}^w].
\end{equation}
From the above inequality, (\ref{equ: T_{1/2}^w}) is not only an upper bound, but a tight estimation of $T_{1/2}^w$, i.e.
$$T_{1/2}^w=\Theta\left(\frac{(1+\alpha^\star)^2L}{{w^\star}^2}\log\left(\frac{1}{\sigma_\init}\right)\right).$$

% Therefore, $w(t)$ has the two-sided bound:
% \begin{equation}
% w=\sigma_{\init}\exp\left(\Theta\left(\frac{{w^\star}^2 t}{(1+\alpha^\star)^2L}\right)\right),\quad \forall t\in[0,T_2^h].
% \end{equation}

% where 
% \begin{align*}
%     T_2^h=\Theta\left(\frac{(1+\alpha^\star)^2L}{{w^\star}^2}\log\left(\frac{1}{\sigma_\init}\right)\right).
% \end{align*}

\underline{\bf Part II. Dynamics after the critical point $T_{1/2}^w$.}

For simplicity, we consider:
\begin{align*}
    v:=w^2,
\end{align*}
and denote $v^\star:={w^\star}^2,h^\star:=\frac{1}{1+\alpha^\star}$.
Then we focus on the dynamics of $v$ and $h$.

Additionally, we introduce a few notations used in this part:
\begin{align*}
    \phi(x):=\frac{x}{(1-x^2)^{3/2}},\quad
    \psi(x):=\frac{1}{(1-x^2)^{1/2}}.
\end{align*}

Then the dynamics of $v$ and $g$ are:
\begin{align*}
    \frac{\rd v}{\rd t}& = \frac{4 v h}{L-2}\big(h^\star\phi(v+v^\star)-h\phi(2v)\big),
    \\
    \frac{\rd h}{\rd t}& = \frac{1}{L-2}
    \big(h^\star\psi(v+v^\star)-h\psi(2v)\big).
\end{align*}

\underline{\em {\bf Step II.1.} A coarse estimate of the relationship between $v$ and $h$. }

It is easy to verify the monotonicity that $\frac{\rd v}{\rd t}>0$ and $\frac{\rd h}{\rd t}<0$ for $t>t_2$. Additionally, we have
\begin{align*}
    \frac{\psi(v+v^\star)}{\psi(2v)}< \frac{h}{h^\star} < \frac{\phi(v+v^\star)}{\phi(2v)}.
\end{align*}

Then by Monotone convergence theorem, we obtain:
\begin{align*}
    \lim_{t\to+\infty}v=v^\star,\quad
    \lim_{t\to+\infty}h=h^\star.
\end{align*}

\underline{\em {\bf Step II.2.} Convergence analysis by Lyapunov function.}

This step aims to establish the convergence rate of $v$ and $h$.
% Note that for $v$ and $h$, we are only concerned with the convergence rate after a sufficiently long period, making the analysis simpler compared to the convergence analysis of $p,g$.

% Due to $\lim_{t\to+\infty}v=v^\star,
    % \lim_{t\to+\infty}h=h^\star$, 
% there exists a time $T_2^{v,h}>t_2$ such that for all $T_2^{v,h}$,
% $$0.99v^\star<v(t)< v^\star,h^\star<h(t)<1.01h^\star$$ 
% Our proof only need to focus on the dynamics for $t>T_2^{v,h}$.

In fact, the dynamics of $v,h$ can be approximately characterized by their linearized dynamics.
In contrast, the dynamics of $p,g$ are controlled by high-order terms. Therefore, the proof for $v$ and $h$ is significantly simpler than the corresponding proof for $p$ and $g$.
We only need to consider the simplest Lyapunov function:
\begin{align*}
    G(v,h):=\frac{1}{2}\Big((v-v^\star)^2+(h-h^\star)^2\Big).
\end{align*}

It is easy to verify that
\begin{align*}
    &(L-2)\frac{\rd G(v,h)}{\rd t}=(v-v^\star)\frac{\rd v}{\rd t}+(h-h^\star)\frac{\rd h}{\rd t}
    \\=&
    4 v h (v-v^\star)\big(h^\star\phi(v+v^\star)-h\phi(2v)\big)+
    (h-h^\star)\big(h^\star\psi(v+v^\star)-h\psi(2v)\big)
    \\=&
    4 v h (v-v^\star)\Big(\phi(v+v^\star)(h^\star-h)-h(\phi(v+v^\star)-\phi(2v))\Big)
    \\&+(h-h^\star)\Big((h^\star-h)\psi(v+v^\star)+h(\psi(v+v^\star)-\psi(2v))\Big)
    \\=&
    -4vh^2(v^\star-v)(\phi(v+v^\star)-\phi(2v))-\psi(v+v^\star)(h-h^\star)^2
    \\&+4vh\phi(v+v^\star)(v-v^\star)(h^\star-h)+h(h-h^\star)(\psi(v+v^\star)-\psi(2v)).
\end{align*}

% We can choose $v^\star\leq 0.4$, which satisfies $w^\star=\sqrt{v^\star}=\cO(1)$.

Let $v^\star\leq0.3=\cO(1)$. Recalling~\eqref{proof: equ: bounded, turning point, w} and~\eqref{proof: equ: bounded, h}, as well as the monotonicity about $p$ and $w$, we have:
\begin{align*}
    \frac{v^\star}{4}<v(t)<v^\star; \quad h^\star <h(t) < 1.02 h^\star,\quad \forall t >T_2^h.
\end{align*}

Combining these estimates with the properties of $\phi$ and $\psi$, we have the following straight-forward estimates:
\begin{gather*}
\phi(v+v^\star)-\phi(2v)=\phi'(\xi)(v^\star-v)=\frac{1+2\xi^2}{(1-\xi^2)^{5/2}}(v^\star-v)\geq v^\star-v;
\\
\phi(v+v^\star)\leq\phi(2v^\star)\leq 1;
\\
\psi(v+v^\star)=\frac{1}{(1-(v+v^\star)^2)^{1/2}}\geq 1;
\\
\psi(v+v^\star)-\psi(2v)=\psi'(\xi)(v^\star-v)=\frac{\xi}{(1-\xi^2)^{3/2}}(v^\star-v)\leq 1.3 v^\star (v^\star-v).
\end{gather*}

Thus, we have the following estimate for the Lyapunov function:
\begin{align*}
    &(L-2)\frac{\rd G(v,h)}{\rd t}
    \\\leq& -\frac{4}{1.02} v^\star {h^\star}^2(v-v^\star)^2-(h-h^\star)^2
    \\&+4.08 v^\star h^\star(v-v^\star)(h^\star-h) +1.3\cdot 1.02 v^\star h^\star (v^\star-v)(h-h^\star)
    \\=&
    -\frac{4}{1.02} v^\star {h^\star}^2(v-v^\star)^2-(h-h^\star)^2+5.41v^\star h^\star (v^\star-v)(h-h^\star)
    \\\leq&
    -3.92 v^\star {h^\star}^2(v-v^\star)^2-(h-h^\star)^2+\left(9.6 {v^\star}^2 {h^\star}^2(v-v^\star)^2+\frac{3}{4}(h-h^\star)^2\right)
    \\\leq& -(3.92-9.6\cdot 0.3) v^\star {h^\star}^2(v-v^\star)^2 - 0.25 (h-h^\star)^2
    \leq -\frac{1}{4}v^\star {h^\star}^2 G(v,h).
\end{align*}

Consequently, we have the exponential bound for all $t>T_2^h$:
\begin{align*}
    G(v(t),h(t))\leq G\left(v(T_2^{h}),h(T_2^{h})\right)\exp\left(-\frac{v^\star {h^\star}^2}{4(L-2)} (t-T_2^{h})\right),\quad \forall t>T_2^{h},
\end{align*}
% where $G(v(T_{\rm T_2}),h(T_{\rm T_2}))=\cO\left(\right)$

This can imply:
\begin{equation}\label{proof: equ: convergence, w h}
\begin{aligned}
    (h(t)-h^\star)^2&=(h(T_2^h)-h^\star)^2\exp\left(-\Omega\left(\frac{{w^\star}^2(t-T_2^{h})}{L(1+\alpha^\star)^2}\right)\right)
    \\&=\cO\left({h^\star}^2\exp\left(-\Omega\left(\frac{{w^\star}^2(t-T_2^{h})}{L(1+\alpha^\star)^2}\right)\right)\right),\quad \forall t>T_2^{h};
    \\(w(t)-w^\star)^2&=(w(T_2^h)-w^\star)^2\exp\left(-\Omega\left(\frac{{w^\star}^2(t-T_2^{h})}{L(1+\alpha^\star)^2}\right)\right)
    \\&=\cO\left({w^\star}^2\exp\left(-\Omega\left(\frac{{w^\star}^2(t-T_2^{h})}{L(1+\alpha^\star)^2}\right)\right)\right),\quad \forall t>T_2^{h}.
\end{aligned}
\end{equation}

{\bf Notably}, these proofs capture the {\bf entire} training dynamics of $w,h$, from $t=0$ to $t=T_1^h$, to $t=T_{1/2}^w\leq T_2^h$, and finally to $t\to+\infty$, providing a fine-gained analysis for each phase.

\subsection{Proof of Theorem~\ref{thm: optimization}}

This theorem is a direct corollary of our analysis of the entire training dynamics in Appendix~\ref{subsection: dynamics 4-gram} and~\ref{subsection: dynamics induction head}, leveraging the relationship between the parameters and the loss.

\underline{\em Proof of Phase I (partial learning).}

By combining~\eqref{proof: equ: FG loss} and~\eqref{proof: equ: convergence, p g}, it follows that: $\cL_{\FG}(\theta(0))=\Theta(1)$. Moreover,
\begin{align*}
    \cL_{\FG}(\theta(t))=\cO\left(\frac{1}{t}\right),\quad t>T_1^g=O(1).
\end{align*}

Thus, there exists a sufficiently large $T_{\rm I}=\Theta(1)$, such that:
\begin{align*}
    \cL_{\FG}(\theta(T_{\rm I}))\leq 0.01\cL_{\FG}(\theta(0)).
\end{align*}

Recalling our proof in Appendix~\ref{subsection: dynamics induction head}, for $t< T_{1/2}^h=\cO(L)$, it holds that
$h(t)<\sigma_{\init}+\cO({t}/{((1+\alpha^\star)L)}),w(t)<\sigma_{\init}+o({t}/{((1+\alpha^\star)L)})$. Additionally, since $T_{\rm I}=\Theta(1)\ll \Theta(L)$, it follows that 
$$w(T_{\rm I})=\cO(\sigma_{\rm init}+1/L)<2\sigma_{\init}\ll w^\star,\quad h(T_{\rm I})=\cO(\sigma_{\rm init}+1/L)<2\sigma_{\init}\ll h^\star.$$

Substituting these estimates into~\eqref{proof: equ: IH loss}, we obtain by Lipschitz continuity of $\cL_{\IH}$:
\begin{align*}
    |\cL_{\IH}(\theta(T_{\rm I}))-\cL_{\IH}(\theta(0))|
    &\leq 2\sigma_{\init}\left(\left|\frac{\partial \cL_{\IH}}{\partial w}\right|+\left|\frac{\partial \cL_{\IH}}{\partial h}\right|\right)
    \\&\leq 2\sigma_{\init}\left(\cO\left(\frac{1}{(1+\alpha^\star)L}\right)+o\left(\frac{1}{(1+\alpha^\star)L}\right)\right)
    \\&\leq 0.01 \cL_{\IH}(\theta(0)).
\end{align*}
Thus,
\begin{align*}
    \cL_{\IH}(\theta(T_{\rm I}))\geq 0.99 \cL_{\IH}(\theta(0)).
\end{align*}

\underline{\em Proof of Phase II (plateau) + Phase III (emergence).}

First,~\eqref{equ: w upper} and ~\eqref{proof: equ: before point, w} ensures that $w$ grows exponentially before $t<T_{1/2}^w$: 
$$\sigma_{\init}\exp\left(\frac{{w^\star}^2}{(1+\alpha^\star)^2(L-2)}t\right) \leq w\leq  \sigma_{\init}\exp\left(\frac{4{w^\star}^2(1-4{w^\star}^4)^{\frac{3}{2}}}{(1+\alpha^\star)(L-2)}t\right).$$
Thus, we have: 
\begin{align*}
    w(t)=\sigma_{\init}\exp\left(\Theta\left    (\frac{{w^\star}^2 t}{(1+\alpha^\star)^2L}\right)\right),\quad t<\Theta\left(\frac{(1+\alpha^\star)^2L}{{w^\star}^2}\log\left(\frac{1}{\sigma_\init}\right)\right).
\end{align*}

Now we define the observation time $T_o:=T_1^h=\Theta(L)$. Notably,
\begin{align*}
    h(T_o)=h^\star,\quad
    w(T_o)<0.01w^\star.
\end{align*}

The exponential growth of $w$ further implies:
\begin{align*}
    T_{0.01}^w:=\left\{t>0:w(t)>0.01w^\star\right\}=\Theta\left(\frac{(1+\alpha^\star)^2L}{{w^\star}^2}\log\left(\frac{1}{\sigma_\init}\right)\right).
\end{align*}

Regarding the dynamics of $h$, by~\eqref{proof: equ: bounded, h}, we have $|h(t)-h(T_o)|<0.02|h(T_o)|,\ \forall t\geq T_o$.

Now we incorporate these facts ( $0<w(T_o)< 0.01 w^\star$, $0<w(T_{0.01}^w)\leq 0.01 w^\star$, $|h(T_{0.01}^w)-h(T_o)|<0.02|h(T_o)|$, $h(T_o)=h^\star$) into the loss~\eqref{proof: equ: IH loss}.
By the Lipschitz continuity of $\cL_{\IH}$, it is straightforward that 
$$\cL_{\IH}(\theta(T_{0.01}^w))\geq 0.99\cL(\theta(T_o)).$$

Thus, we have established the lower bound for $T_{\rm II}$:
\begin{align*}
    T_{\rm II}&:=\inf\big\{t>T_o:\cL_{\IH}(\theta(t))\leq 0.99 \cdot\cL_{\IH}(\theta(T_o))\big\}
    \\&\geq T_{0.01}^w=\Omega\left(\frac{(1+\alpha^\star)^2L}{{w^\star}^2}\log\left(\frac{1}{\sigma_\init}\right)\right).
\end{align*}

Combining the loss~\eqref{proof: equ: IH loss} and our parameter estimates~\eqref{proof: equ: convergence, w h}, we obtain:
\begin{align*}
  \cL_{\IH}(\theta(t))=\cO\left(\exp\left(-\Omega\left(\frac{{w^\star}^2 t}{L(1+\alpha^\star)^2}\right)\right)\right),\ t>T_2^h=\Theta\left(\frac{(1+\alpha^\star)^2L}{{w^\star}^2}\log\left(\frac{1}{\sigma_\init}\right)\right).
\end{align*}

This implies the upper bound for $T_{\rm III}$:
\begin{align*}
    &T_{\rm III}:=\inf\left\{t>T_o:\cL_{\IH}(\theta(t))\leq 0.01\cdot\cL_{\IH}(\theta(T_o))\right\}
    \\&=T_{1/2}^w+\cO\left((\alpha^\star+1)^2L\log(1/\sigma_{\rm init})/{w^\star}^2\right)=\cO\left((\alpha^\star+1)^2L\log(1/\sigma_{\rm init})/{w^\star}^2\right).
\end{align*}

Combining the fact $T_{\rm II}< T_{\rm III}$, the lower bound for $T_{\rm II}$, and the uppper bound for $T_{\rm III}$, we obtain the two-sided bounds for both $T_{\rm II}$ and $T_{\rm III}$:
\begin{align*}
    T_{\rm II},\ T_{\rm III}=\Theta\left((\alpha^\star+1)^2L\log(1/\sigma_{\rm init})/{w^\star}^2\right).
\end{align*}

\underline{\em Proof of Phase IV (convergence).}

By combining the loss~\eqref{proof: equ: FG loss},~\eqref{proof: equ: IH loss}, and our parameter estimates~\eqref{proof: equ: convergence, p g},~\eqref{proof: equ: convergence, w h}, it follows that:
\begin{align*}
  \cL_{\FG}(\theta(t))=\cO\left(\frac{1}{t}\right),\quad
    \cL_{\IH}(\theta(t))=\cO\left(\exp\left(-\Omega\left(\frac{{w^\star}^2 t}{L(1+\alpha^\star)^2}\right)\right)\right),\quad t>T_{\rm III}.
\end{align*}

% \begin{align*}
%     T_{\rm II}&:=\inf\big\{t>T_o:\cL_{\IH}(\theta(t))\leq 0.99 \cdot\cL_{\IH}(\theta(T_o))\big\}=\Theta\left((\alpha^\star+1)^2L\log(1/\sigma_{\rm init})/{w^\star}^2\right);
%     \\
%     T_{\rm III}&:=\inf\big\{t>T_o:\cL_{\IH}(\theta(t))\leq 0.1\cdot\cL_{\IH}(\theta(T_o))\big\}=\Theta\left((\alpha^\star+1)^2L\log(1/\sigma_{\rm init})/{w^\star}^2\right).
%     \end{align*} 

% \section{Key Lemmas}

\vspace{1.cm}

\section{Useful Inequalities}\label{appendix: lemma}

% In this section, we list the lemmas used in our work, their proof can be found in the original article. 

% \subsection{Approximation Theorems for $\ell_1$ Space}
% The following lemma provides the approximation rate for Delta function with exponentially decayed series.

\begin{lemma}[Corollary A.7 in \citet{edelman2022inductive}]\label{lemma: softmaxlipschitz}
    For any $\theta, \theta'\in \mathbb{R}^d$, we have 
    $$\Vert\sm(\theta)-\sm(\theta')\Vert_1\leq 2\Vert\theta-\theta'\Vert_{\infty}$$
\end{lemma}

\begin{lemma}\label{lemma E.1}
For any $T\in\bbN_+$, $q,m\in\bbN_+$, there exist a $\phi_{m}^{\exp}(t)=\sum\limits_{k=1}^m \alpha_k e^{-\beta_k t}$ such that
\begin{align*}
    \norm{\bbI(\cdot=T)-\phi_{m}^{\rm exp}(\cdot)}_{\ell_1(\bbN)}
    \leq C\frac{A^q (q^2)^q e^{0.01(q+1)T}}{m^{q}},
\end{align*}
where $\beta_k>0$ holds for any $k\in[m]$, and $A,C>0$ are absolute constants.
\end{lemma}

\begin{proof}[Proof of Lemma~\ref{lemma E.1}]
This lemma is a corollary of Lemma F.1 in ~\citet{wang2024understanding}. By Lemma F.1 in ~\citet{wang2024understanding} and its proof: for any $T\in\bbN_+$, $q,m\in\bbN_+$, there exists a $C(q)>0$ and a $\phi_{m}^{\exp}(t)=\sum\limits_{k=1}^m \alpha_k e^{-\beta_k t}$ such that
\begin{align*}
    \norm{\bbI(\cdot=T)-\phi_{m}^{\rm exp}(\cdot)}_{\ell_1(\bbN)}
    \leq\frac{C(q)e^{0.01(q+1)T}}{m^{q}},
\end{align*}
where $\beta_k>0$ holds for any $k\in[m]$. Moreover, 
\begin{align*}
    C(q)=\frac{M(q)}{(1-1/e)^q},\quad 
    M(q)=\max_{0\leq k\leq q}\sup_{x\in[-1,1]}\left|\Psi^{(k)}(x)\right|,
\end{align*}
where $\Psi(x)=\begin{cases}
    \exp\left(-\frac{1}{1-x^2}\right),\ x\in(-1,1) \\
    0,\ \text{otherwise}
\end{cases}$ is the standard bump function on $[-1,1]$. By a straight-forward estimate, there exist absolute constants $C_1,C_2>0$ such that
\begin{align*}
    \sup_{x\in[-1,1]}\left|\Psi^{(k)}(x)\right|\leq C_1 (C_2)^k (k!)^2,
    \quad\forall k\in\bbN^+.
\end{align*}

Thus, $C(q)\leq C_1 \left(\frac{C_2}{1-1/e}\right)^q (q!)^2$. By using Stirling's formula
, there exists an absolute $C_3>0$ such that:
\begin{align*}
    C(q)\leq C_1 \left(C_3 q^2\right)^q.
\end{align*}
\end{proof}

\begin{lemma}\label{lemma: calculate Gaussian}
    $\mathop{\bbE}\limits_{X,Y,Z}\exp(aXY)\bZ^2 = (1-a^2)^{-1/2}$, $a<1$.
\end{lemma}
\begin{proof}[Proof of Lemma~\ref{lemma: calculate Gaussian}]
    \begin{align*}
        &\quad \int\exp(aXY)Z^2\left(\frac{1}{2\pi}\right)^{-3/2}\exp(-\frac{1}{2}X^2-\frac{1}{2}Y^2-\frac{1}{2}Z^2)\ dX dY dZ
        \\ & = \int\frac{1}{2\pi}\exp(-\frac{1}{2}(X-aY)^2-\frac{1}{2}Y^2+\frac{1}{2}a^2Y^2)\ d(X-aY)dY
        \\ & = \int\frac{1}{\sqrt{2\pi}}\exp(-\frac{1}{2}W^2)\ dW \ (W = (1-a^2)^{1/2}Y )
        \\ & = (1-a^2)^{-1/2}
    \end{align*}
\end{proof}

\begin{lemma}\label{lemma: calculate softmax} Let $M(p):=\frac{1-e^{-p(L-2)}}{1-e^{-p}}$, then it holds that
$$\norm{\sm\big((-p(L-1-s))_{s=1}^{L-1}\big)}_2^2 = \frac{M(2p)}{M(p)^2}.$$
\end{lemma}

\begin{definition}[weakly majorizes]
    A vector \( \mathbf{x} \in \mathbb{R}^n \) is said to \textit{weakly majorize} another vector \( \mathbf{y} \in \mathbb{R}^n \), denoted by \( \mathbf{x} \prec_w \mathbf{y} \), if the following conditions hold:
\begin{enumerate}
    \item \( \sum_{i=1}^{k} x_{[i]} \leq \sum_{i=1}^{k} y_{[i]} \) for all \( k = 1, 2, \dots, n-1 \),
    \item \( \sum_{i=1}^{n} x_{[i]} = \sum_{i=1}^{n} y_{[i]} \),
\end{enumerate}
where \( x_{[i]} \) and \( y_{[i]} \) are the components of \( \mathbf{x} \) and \( \mathbf{y} \), respectively, arranged in decreasing order.

\end{definition}

\begin{lemma}[Weighted Karamata Inequality]
\label{lemma: Weighted Karamata Inequality}
    Let \( f \colon \mathbb{R} \to \mathbb{R} \) be a convex function, and let \( \mathbf{x} = (x_1, x_2, \dots, x_n) \) and \( \mathbf{y} = (y_1, y_2, \dots, y_n) \) be two vectors in \( \mathbb{R}^n \). If \( \mathbf{x} \) \textit{weakly majorizes} \( \mathbf{y} \) (i.e., \( \mathbf{x} \prec_w \mathbf{y} \)), and \( w_1, w_2, \dots, w_n \) are non-negative weights such that \[\sum_{i=1}^{n} w_i = 1,\]then the following inequality holds:\[\sum_{i=1}^{n} w_i f(x_i) \leq \sum_{i=1}^{n} w_i f(y_i).\]
\end{lemma}

\end{document}